\documentclass{article}

\usepackage{microtype}
\usepackage{graphicx}
\usepackage{booktabs} 

\usepackage[utf8]{inputenc} 
\usepackage[T1]{fontenc}    
\usepackage{hyperref}       
\usepackage{nicefrac}       
\usepackage{microtype}      
\usepackage[table]{xcolor}
\usepackage{framed}
\colorlet{shadecolor}{pink}
\usepackage{authblk}
\usepackage{csquotes}

\usepackage{graphicx}
\usepackage{soul}
\usepackage{subcaption}
\usepackage{booktabs} 
\usepackage{tablefootnote}
\usepackage{dsfont}

\usepackage{multirow}
\usepackage{makecell}

\usepackage{amsmath,amsthm,amssymb,amsfonts}
\usepackage{algorithmic}
\usepackage{enumerate}
\usepackage{comment}
\usepackage{bm}
\usepackage{bbm}
\usepackage{pifont}
\usepackage{makecell}

\usepackage{hyperref}

\usepackage[accepted]{icml2023}

\usepackage{amsmath}
\usepackage{amssymb}
\usepackage{mathtools}
\usepackage{amsthm}

\usepackage[capitalize,noabbrev]{cleveref}

\theoremstyle{plain}
\newtheorem{theorem}{Theorem}[section]

\newtheorem{lemma}[theorem]{Lemma}
\newtheorem{corollary}[theorem]{Corollary}
\theoremstyle{definition}
\newtheorem{definition}[theorem]{Definition}
\newtheorem{assumption}[theorem]{Assumption}
\theoremstyle{remark}
\newtheorem{remark}[theorem]{Remark}

\usepackage[textsize=tiny]{todonotes}

\icmltitlerunning{Beyond Uniform Lipschitz Condition in Differentially Private Optimization}

\begin{document}

\twocolumn[
\icmltitle{Beyond Uniform Lipschitz Condition in Differentially Private Optimization}



\icmlsetsymbol{equal}{*}

\begin{icmlauthorlist}
\icmlauthor{Rudrajit Das}{equal,ut}
\icmlauthor{Satyen Kale}{goo}
\icmlauthor{Zheng Xu}{goo}
\icmlauthor{Tong Zhang}{goo,hkust}
\icmlauthor{Sujay Sanghavi}{ut}
\end{icmlauthorlist}

\icmlaffiliation{ut}{UT Austin}
\icmlaffiliation{goo}{Google Research}
\icmlaffiliation{hkust}{HKUST}

\icmlcorrespondingauthor{Rudrajit Das}{rdas@utexas.edu}

\icmlkeywords{Machine Learning, ICML}

\vskip 0.3in
]



\printAffiliationsAndNotice{\icmlEqualContribution} 

\begin{abstract}
  Most prior results on differentially private stochastic gradient descent (DP-SGD) are derived under the simplistic assumption of uniform Lipschitzness, i.e., the per-sample gradients are uniformly bounded. We generalize uniform Lipschitzness by assuming that the per-sample gradients have sample-dependent upper bounds, i.e., per-sample Lipschitz constants, which themselves may be unbounded. We provide principled guidance on choosing the clip norm in DP-SGD for convex over-parameterized settings satisfying our general version of Lipschitzness when the per-sample Lipschitz constants are bounded; specifically, we recommend tuning the clip norm only till values up to the minimum per-sample Lipschitz constant. This finds application in the private training of a softmax layer on top of a deep network pre-trained on public data. We verify the efficacy of our recommendation via experiments on 8 datasets. Furthermore, we provide new convergence results for DP-SGD on convex and nonconvex functions when the Lipschitz constants are unbounded but have bounded moments, i.e., they are heavy-tailed. 
\end{abstract}

\section{Introduction}
\label{intro}
With the ever-increasing amount of data being used, there is a growing need for the development of privacy-preserving training schemes for machine learning (ML) models. Differential privacy (DP) \citep{dwork2006calibrating} is a popular privacy-preserving framework that is being incorporated in the training of ML models. We formally define DP in \Cref{def-dp}, but at a high level, DP can be guaranteed by adding Gaussian noise, where the noise scale is determined by the \enquote{sensitivity} to an individual's data. There has been copious research on private optimization for private training; in this paper, we focus on DP-SGD \citep{abadi2016deep} which is the default algorithm for private optimization in practice.

We briefly introduce the problem setting and DP-SGD to facilitate further discussion (see \Cref{sec:prelim} for more details). We consider empirical risk minimization (ERM) of $f(\bm{w}) = \frac{1}{n}\sum_{i=1}^n f_i(\bm{w})$, where each $f_i: \mathbb{R}^d \xrightarrow{} \mathbb{R}$. In every iteration of DP-SGD (stated in \Cref{alg:1}), the optimizer receives a noise-perturbed average of the \textit{clipped} per-sample gradients for performing the update; noise is added to guarantee differential privacy. Specifically, at iteration $t$, the optimizer receives $\bm{g}_t = \frac{1}{b}\sum_{i \in \mathcal{S}_t} \text{clip}(\nabla f_i(\bm{w}_t), \tau) + \bm{\zeta}_t$, where $\mathcal{S}_t$ is a random batch of samples formed by picking each sample in $\{1, \ldots, n\}$ with probability $({b}/{n})$, $\text{clip}(\bm{z},c) := \bm{z} \min(1, {c}/{\|\bm{z}\|})$ for a vector $\bm{z}$, $\tau$ is the clipping threshold or clip norm, and $\bm{\zeta}_t$ is an isotropic Gaussian random vector whose variance is proportional to $\tau^2$ and also depends on the amount of privacy required.

Clipping is employed in DP-SGD to bound the maximum sensitivity of the average gradient to each sample’s individual gradient, which is required to set the noise variance. However, clipping can also make $\bm{g}_t$ a {\em biased} estimator of $\nabla f(\bm{w}_t)$, and the amount of bias depends on the clip norm $\tau$ -- the higher we set $\tau$, the lower is the bias, and vice-versa. As the noise variance is proportional to $\tau^2$, there is an inherent tension between the bias and variance of $\bm{g}_t$ due to the clip norm $\tau$. This raises a natural question - \textit{how do we set \enquote{good} clip norms to balance the bias-variance tradeoff?}

In order to circumvent the analysis of the clipping bias, most prior convergence results for private optimization \citep{bassily2014private,bassily2019private,wang2018differentially,wang2019efficient} assume that the loss function is \textit{uniformly Lipschitz} for all samples and model parameters, i.e., the per-sample gradients (w.r.t. the model parameters) have a sample-independent upper bound known as the Lipschitz constant. Under this assumption, the clip norm is set equal to the Lipschitz constant resulting in zero bias as no clipping happens. {But in practice, such a choice of clip norm results in very poor performance; see \Cref{fig:results} and its caption.} 
Additionally, uniform Lipschitzness does not even hold for simple problems like linear regression with Gaussian data, precluding the existence of a trivial clip norm for analysis. 

{In this work, we generalize uniform Lipschitzness by instead assuming that the per-sample gradients have sample-dependent upper bounds, i.e., per-sample Lipschitz constants, which themselves may be unbounded. {Under this generalized assumption, we provide a theoretically-motivated clip norm selection strategy for convex settings.} 
Our method finds direct application in the private training of the softmax layer of pre-trained deep networks which is a popular and effective scheme for private training \citep{de2022unlocking, mehta2022large}. In practice, the clip norm is tuned over multiple runs which is not only computationally inefficient but also increases the privacy cost \cite{papernot2021hyperparameter}. So it is desirable to have principled methods for setting the clip norm to alleviate these two issues.} {Additionally, we provide novel convergence results for DP-SGD when the per-sample Lipschitz constants are heavy-tailed.}

Before we state our contributions in detail, we need to briefly introduce the metric quantifying convergence, which we call the \enquote{optimization risk}. Let $\bm{w}_\text{priv}$ be the output of DP-SGD (Alg. \ref{alg:1}). If $f$ is convex, the optimization risk is the expected suboptimality gap, i.e., $\mathbb{E}[f(\bm{w}_\text{priv})] - \min_{\bm{w}} f(\bm{w})$. If $f$ is nonconvex, the optimization risk is the expected gradient-norm squared, i.e. $\mathbb{E}[\|\nabla f(\bm{w}_\text{priv})\|^2]$. When DP-SGD is $(\varepsilon,\delta)$-DP (defined in \Cref{def-dp}), our convergence results are expressed in terms of the following key quantity:
\begin{equation}
    \label{eq:def-vaphi}
    \varphi := {\sqrt{\nu d \log({1}/{\delta})}}/{n\varepsilon},
\end{equation}
where $d$ is the dimension of the model parameters, $n$ is the number of samples, and $\nu$ is an absolute constant. We assume that $n$ is large enough so that $\varphi < 1$, and the number of iterations of DP-SGD is sufficiently large. 
We now list our \textbf{main contributions}, and also summarize the main theoretical results in Tables \ref{tab-2} and \ref{tab-1}.

\textbf{(a)} Throughout this work, we generalize uniform Lipschitzness by assuming that the per-sample gradients have {sample-dependent} upper bounds, i.e., \textit{per-sample Lipschitz constants}, which themselves may be unbounded; we call this \textit{generalized Lipschitzness} (\Cref{asmp-lip-gen}). 
\vspace{-0.1 cm}
\begin{itemize}
    \item In \Cref{sec:1}, we provide a principled clip norm {tuning} strategy for DP-SGD under generalized Lipschitzness. Specifically, we focus on convex settings with near-interpolation like conditions (Asmp. \ref{overparam}), i.e., $\frac{1}{n}\sum_{i=1}^n (f_i(\bm{w}^{*}) - \min_{\bm{w}} f_i(\bm{w})) \approx 0$ where $\bm{w}^{*} = \text{arg min}_{\bm{w}} f(\bm{w})$; over-parameterization is an example of this. For such cases, we recommend \textit{tuning the clip norm only till values up to the minimum per-sample Lipschitz constant} (\Cref{reco}), say $G_{\text{min}}$, based on \Cref{thm-small-clip} where we show that the optimization risk attains the best bound when the clip norm is $\leq G_{\text{min}}$. This is in contrast to prior {theoretical} works which set the clip norm equal to the \textit{maximum} per-sample Lipschitz constant, say $G_{\text{max}}$, for ease of analysis. 
    \vspace{-0.1 cm}
    \item {Our recommendation for convex settings is of relevance to the private training of the softmax (i.e., last) layer of deep networks pre-trained on public data. In \Cref{expts}, we corroborate our recommendation with experiments on four\footnote{We show results on two more datasets in \Cref{extra-expts}.} vision datasets, viz., Caltech-256 \cite{griffin2007caltech}, Food-101 \cite{bossard14}, CIFAR-100 and CIFAR-10, and two language datasets, viz., TweetEval Emoji \cite{barbieri2018semeval} and Emotion \cite{saravia-etal-2018-carer}. As an e.g., for Caltech-256 and Food-101 with $\varepsilon = 2$, the test accuracy obtained by setting the clip norm $\tau = G_{\text{min}}$ is better than that of $\tau = G_{\text{max}}$ by more than 23\% and 11\%, respectively.}
\end{itemize}
\vspace{-0.2 cm}

\textbf{(b)} {In \Cref{non-lip}, we provide optimization risk bounds for DP-SGD under {generalized Lipschitzness}, without assuming interpolation, when the per-sample Lipschitz constants have bounded $k^{\text{th}}$ moment, i.e., they are \textit{heavy-tailed} (Asmp. \ref{dec23-asmp1}).} For private \textbf{unconstrained} \textbf{convex} and smooth \textbf{nonconvex} optimization under 
this assumption, we derive risk bounds of $\mathcal{O}({\varphi^{1 - \frac{2}{k+1}}})$ and $\mathcal{O}({\varphi^{1 - \frac{1}{2k-1}}})$, respectively\footnote{The seemingly better result for the nonconvex setting is because of the difference in the risk metrics between the convex and nonconvex cases.}; see Theorems \ref{thm1-jan19} and \ref{thm1-dec15}. Under an additional mild assumption, we improve the risk bound in the convex case to $\mathcal{O}({\varphi^{1 - \frac{1}{k}}})$ and also derive a matching lower bound, thereby establishing the \textit{optimality of DP-SGD} in this case; see \Cref{asmp1-may4} and Theorems \ref{thm1-may4} and \ref{lower_bound}. These are the first results for private \textbf{unconstrained} optimization under the heavy-tailed assumption or anything similar. Our results also imply the \textit{optimality of DP-SGD} in the \textbf{unconstrained} convex case under \textit{uniform} Lipschitzness; see Cor. \ref{cor-unif}. To our knowledge, this is the first matching lower bound for the \textbf{unconstrained} convex case under \textit{uniform} Lipschitzness.
\vspace{-0.2 cm}
\begin{table}[!htb]
\caption{\textbf{Summary of our clip norm selection result} (Thm. \ref{thm-small-clip}) for the convex case under generalized Lipschitzness (Asmp. \ref{asmp-lip}) and interpolation (Asmp. \ref{overparam}). $G_1$ and $G_n$ are the \textit{minimum} and \textit{maximum} per-sample Lipschitz constants as per Asmp. \ref{asmp-lip}. In the table, $\alpha^{*} = \alpha(G_1) \geq 1$, where $\alpha(G_1)$ is defined in \Cref{may2-asmp-1}, and $B = \mathcal{O}(\|\bm{w}_0 - \bm{w}^{*}\| G_n \varphi)$, where $\bm{w}_0$ is the initial point, $\bm{w}^{*}$ is a minimizer of $f$ and $\varphi = \mathcal{O}(\nicefrac{\sqrt{d  \log({1}/{\delta})}}{n\varepsilon})$.
}
\label{tab-2}
\vspace{-0.3 cm}
\begin{center}
\resizebox{\columnwidth}{!}{
\begin{tabular}{|l|c|}
\toprule
Clip Norm $\tau$ & \makecell{Risk Upper Bound when \\ Perfect Interpolation holds \\
($\Delta(\bm{w}^{*}) = 0$ in Asmp. \ref{overparam})} \\
\midrule
$\in (0, G_1]$ (\textbf{this work})  & ${B}/{{\alpha}^{*}}$ (${\alpha}^{*} \geq 1$)\\
\midrule
$\in (G_1, G_n)$ (\textbf{this work})  & $\geq {B}/{\alpha}^{*}$ but $\leq B$ \\
\midrule
$G_n$ (default choice of prior theory)  & $B$ \\
\bottomrule
\end{tabular}
}
\end{center}
\vskip -0.1in
\end{table}

\begin{table*}[!htb]
\renewcommand\thempfootnote{\arabic{mpfootnote}}
\begin{minipage}{\textwidth} 
\caption{\textbf{Summary of optimization risk (OR) bounds.} 
OR is defined in \Cref{def-cor} and $\varphi = \mathcal{O}(\nicefrac{\sqrt{d  \log({1}/{\delta})}}{n\varepsilon}) < 1$.
In \Cref{dec23-asmp1}, we assume that the per-sample gradients have {sample-dependent} upper bounds with bounded $k^{\text{th}}$ moment ($k > 1$).
In \Cref{asmp1-may4}, we assume a mild lower bound on function suboptimality of points far away from the optimum.
}
\label{tab-1}
\vspace{-0.15cm}
\begin{center}
\resizebox{\columnwidth}{!}{
\begin{tabular}{c|c|c|c}
\toprule
Reference & Assumption(s) \& Setting & Risk Upper Bound & Matching Lower Bound? \\ 
\midrule
\citet{bassily2014private} & {Uniform} Lipschitz \& Convex Constrained  & $\mathcal{O}(\varphi)$ & Yes \\
\midrule
\textbf{This work} & {Uniform} Lipschitz \& Convex {Unconstrained} ($\mathcal{W} = \mathbb{R}^d$) & $\mathcal{O}(\varphi)$ (Cor. \ref{cor-unif}) & Yes, {for $\delta < e^{-\varepsilon^2}$} (Cor. \ref{cor-unif}) \\
\midrule
\makecell{\citet{kamath2021improved}, \\ \citet{lowy2022private}\footnote{These two works consider the stochastic optimization setting with $(0,\mathcal{O}(\varepsilon^2))$-zCDP and derive bounds for the generalization error. In \Cref{non-lip-cvx}, we show that the same bounds hold for the training error (i.e., OR) in empirical risk minimization with $(\varepsilon,\delta)$-DP.}} & \Cref{dec23-asmp1} \& Convex Constrained  & $\mathcal{O}(\varphi^{1 - \frac{1}{k}})$  & Yes \\
\midrule
\textbf{This work} & \Cref{dec23-asmp1} \& Convex Unconstrained  & $\mathcal{O}(\varphi^{1 - \frac{2}{k+1}})$ (Thm. \ref{thm1-jan19}) & ? \\
\midrule
\textbf{This work} & Assumptions \ref{dec23-asmp1}, \ref{asmp1-may4} \& Convex Unconstrained & $\mathcal{O}(\varphi^{1 - \frac{1}{k}})$ (Thm. \ref{thm1-may4}) &
Yes, {for $\delta < e^{-\varepsilon^2}$} (Thm. \ref{lower_bound}) \\
\midrule
\makecell{\citet{arora2022faster}, \\ \citet{tran2022momentum}\footnote{The $\mathcal{O}(\varphi^{4/3})$ bound is attained by algorithms very different from DP-SGD. For DP-SGD like algorithms, \citet{wang2018differentially} obtain the best known bound of $\mathcal{O}(\varphi)$.}} & {Uniform} Lipschitz \& Smooth Nonconvex Unconstrained & $\mathcal{O}(\varphi^{4/3})$ & ?\\
\midrule
\textbf{This work} & \Cref{dec23-asmp1} \& Smooth Nonconvex Unconstrained  & $\mathcal{O}(\varphi^{1 - \frac{1}{2k-1}})$ (Thm. \ref{thm1-dec15}) & ? \\
\bottomrule
\end{tabular}
}
\end{center}
\end{minipage}
\end{table*}

\section{Related Work}
\label{rel-wrk}
\textbf{DP-(S)GD with Clipping:}
\citet{abadi2016deep} introduce the famous DP-SGD algorithm with clipping for differentially private training in practice.  
There are some papers analyzing the effect of clipping in DP-(S)GD in different settings such as \citet{chen2020understanding,bu2021convergence,song2021evading}. {However, these works do not provide any practical insights into how to set the clip norm for DP-SGD, which is a key focus of our work. Related to our focus, there are some variants of DP-SGD such as \citet{andrew2019differentially,du2021dynamic,wu2021adaptive} that adaptively change the clip norm and/or noise variance to improve convergence. However, practitioners typically use a \textit{constant} clip norm which does not change with the iteration number, and so we only focus on how to set a \textit{constant} clip norm used in \textit{vanilla} DP-SGD \cite{abadi2016deep}.
}

\textbf{Differentially Private Optimization under Uniform Lipschitzness:} There are several papers on private empirical risk minimization (ERM) \cite{chaudhuri2008privacy, chaudhuri2011differentially,kifer2012private,song2013stochastic,duchi2013local,bassily2014private,talwar2014private,talwar2015nearly,wu2017bolt,iyengar2019towards} and private stochastic optimization \cite{bassily2014private,bassily2019private,feldman2020private,asi2021private-2,asi2021private,kulkarni2021private,bassily2021differentially} for \textit{convex Lipschitz} objectives within a \textit{bounded set}. The optimal risk bound for private constrained convex ERM over a bounded set 
in the Lipschitz case 
is $\mathcal{O}(\varphi)$ \cite{bassily2014private}. \citet{zhang2017efficient, wang2018differentially,wang2019efficient,arora2022faster,tran2022momentum} derive convergence results for private unconstrained \textit{nonconvex} ERM with {Lipschitz and smooth} objectives.
\citet{wang2018differentially, wang2019efficient} obtain a risk bound of $\mathcal{O}(\varphi)$ for DP-SGD like algorithms. \citet{arora2022faster,tran2022momentum} obtain an improved bound of $\mathcal{O}(\varphi^{4/3})$ but with more advanced algorithms. 
{These papers simply clip the gradients up to the Lipschitz constant, and do not explore the effect of clipping to smaller values. But as shown in Fig. \ref{fig:results}, such a choice of clip norm performs poorly in practice and besides, uniform Lipschitzness may not always hold.
}

\textbf{Bounded Gradient Moments:} Our assumption of per-sample Lipschitz constants having bounded $k^\text{th}$ moment, i.e. \Cref{dec23-asmp1}, generalizes the \enquote{heavy-tailed} assumption of \citet{wang2020differentially,hu2021high} for private stochastic convex optimization (SCO) with bounded second moment. The bounded $k^\text{th}$ moment assumption has been also analyzed in \citet{kamath2021improved,lowy2022private} for private SCO. 
However, these papers focus only on \textit{constrained convex} optimization. In practice, however, unconstrained minimization is usually performed while training ML models. In this paper, we focus on the \textit{more practical and harder} (w.r.t. analysis) case of private \textit{unconstrained} convex as well as nonconvex optimization (which has not been considered before) under this assumption. We discuss the assumptions of these papers in more detail after \Cref{dec23-asmp1}. Note that none of these papers have any result like our theoretically-motivated clip norm selection method \textit{for use in practice} with \textit{our level of experimental verification}.

\section{Preliminaries} \label{sec:prelim}
\textbf{Notation:} Vectors and matrices are in bold face. 
For any $n \in \mathbb{N}$, the set $\{1,\ldots,n\}$ is denoted by $[n]$, and the uniform distribution over $\{0,\ldots,n\}$ is denoted by $\text{unif}[0,n]$.
$\|.\|$ denotes the $\ell_2$ norm throughout this work. For a function $h$ and any point $\bm{x} \in \mathcal{X}$, the \enquote{suboptimality gap} (at $\bm{x}$) over $\mathcal{X}$ means $h(\bm{x}) - \min_{\bm{x}' \in \mathcal{X}} h(\bm{x}')$.
The function $\text{clip}(.,.): \mathbb{R}^d \times \mathbb{R}^{+} \xrightarrow{}\mathbb{R}^d$ is defined as:
\begin{equation}
    \label{eq:jan8-1}
    \text{clip}(\bm{z},c) := \bm{z} \min(1, {c}/{\|\bm{z}\|}).
\end{equation}
\begin{definition}[\textbf{Lipschitz}]
$h:\mathcal{T} \xrightarrow{} \mathbb{R}$ is to said to be $G$-Lipschitz if $\sup_{\bm{t} \in \mathcal{T}}\|\nabla h(\bm{t})\| \leq G$.
\end{definition}
\begin{definition}[\textbf{Smooth}]
$h:\mathcal{T} \xrightarrow{} \mathbb{R}$ is to said to be $L$-smooth if for all $\bm{t}, \bm{t}' \in \mathcal{T}$, $\|\nabla h(\bm{t}) - \nabla h(\bm{t}')\| \leq L\|\bm{t} - \bm{t}'\|$.
\end{definition}

\begin{definition}
[\textbf{Differential Privacy} \cite{dwork2014algorithmic}]
\label{def-dp}
Suppose each sample $\in \mathcal{S}$. Given a query function $h:\mathcal{S}^n \xrightarrow{} \mathcal{X}$, a randomized mechanism $\mathcal{M}:\mathcal{X} \xrightarrow{} \mathcal{Y}$ is said to be $(\varepsilon,\delta)$-DP if for any two datasets $\mathcal{D}, \mathcal{D}' \in \mathcal{S}^n$ differing in exactly one sample and for any measurable $\mathcal{R} \in \mathcal{Y}$, $\mathbb{P}(\mathcal{M}(h({D})) \in \mathcal{R}) \leq e^{\varepsilon} \mathbb{P}(\mathcal{M}(h({D}')) \in \mathcal{R}) + \delta$.
\end{definition}
The customary way to guarantee DP is to add zero-mean Gaussian noise to the output of $h(.)$ above; this is known as the Gaussian mechanism \cite{dwork2014algorithmic}.

\textbf{Problem Setting and DP-SGD:} 
Suppose we are given a dataset of $n$ i.i.d. samples (features and corresponding labels) $\mathcal{Z} := \{(\bm{x}_i, y_i)\}_{i=1}^n$ drawn from some distribution $\mathcal{D}$. We wish to train a model, parameterized by $\bm{w} \in \mathcal{W} \subseteq \mathbb{R}^d$, on the data via DP-SGD such that the whole training process is $(\varepsilon,\delta)$-DP. We use a loss function $\ell(\bm{w},.)$ (for e.g., the squared loss or cross-entropy loss with some regularization possibly) to learn the model. Let $f_i(\bm{w}) := \ell(\bm{w}, \bm{x}_i, y_i)$; then, we are trying to privately minimize
\begin{equation}
    \label{eq:nov17-100}
    f(\bm{w}) = \frac{1}{n}\sum_{i=1}^n f_i(\bm{w}).
\end{equation}
DP-SGD is summarized in \Cref{alg:1}. Gradient clipping is employed to bound the sensitivity of the average gradient to each sample's individual gradient. Gaussian noise is added to guarantee DP.
In \citet{abadi2016deep}, the last iterate $\bm{w}_T$ is returned; in contrast, we return a random iterate.
We now specify the value of $\sigma_n^2$ required to make Alg. \ref{alg:1} $(\varepsilon, \delta)$-DP  using the moments accountant method of \citet{abadi2016deep}; we provide a short proof in \Cref{pf-dp}.  
\begin{theorem}[\textbf{Moments Accountant~\cite{abadi2016deep}}]
\label{thm-dp}
For $\varepsilon < \mathcal{O}\big(\frac{b^2}{n^2}T\big)$, \Cref{alg:1} is $(\varepsilon, \delta)$-DP for $\sigma_n^2 = \frac{\nu T \log(\frac{1}{\delta}) \tau^2}{n^2 \varepsilon^2}$, where $\nu$ is an absolute constant.
\end{theorem}
\begin{algorithm}[!tbh]
	\caption{\texttt{DP-SGD} \cite{abadi2016deep}}
	\label{alg:1}
	\begin{algorithmic}[1]
		\STATE {\bfseries Input:}
		Domain of parameters $\mathcal{W}$, initial point $\bm{w}_0 \in \mathcal{W}$, number of iterations $T$, learning rates  $\{\eta_{t}\}_{t=0}^{T-1}$, sample selection probability $(b/n)$, clip norm $\tau$ and noise variance $\sigma_n^2$.
		\vspace{0.05 cm}
		\FOR{$t = 0,\dots, T-1$}
		\vspace{0.05 cm}
		\STATE Form a random mini-batch $\mathcal{S}_t$ by picking each sample independently of the others with probability $b/n$.
		\vspace{0.05 cm}
		\STATE  
		Get the noisy mini-batch stochastic gradient
		$\bm{g}_t = \frac{1}{b}\sum_{i \in \mathcal{S}_t} \text{clip}(\nabla f_i(\bm{w}_t), \tau) + \bm{\zeta}_t$, where $\bm{\zeta}_t \sim \mathcal{N}(\vec{0}_d, \sigma_n^2 \text{I}_d)$ and $\text{clip}()$ is defined in (\ref{eq:jan8-1}).
		\vspace{0.08 cm}
		\STATE Let $\bm{z}_{t+1} \xleftarrow{} \bm{w}_t -  \eta_t \bm{g}_t$. 
		Update $\bm{w}_{t+1} \xleftarrow{} \Pi_{\mathcal{W}}(\bm{z}_{t+1})$, where $\Pi_{\mathcal{W}}(\bm{z})$ is the projection of $\bm{z}$ onto $\mathcal{W}$. (Note that $\Pi_{\mathbb{R}^d}(\bm{z}) = \bm{z}$.)
		\vspace{0.08 cm}
		\ENDFOR
		\vspace{0.05 cm}
		\STATE Return $\bm{w}_\text{priv} = \bm{w}_{\widehat{t}}$, where $\widehat{t} \sim \text{unif}[0,T-1]$.
	\end{algorithmic}
\end{algorithm}

Finally, we define our convergence metric for DP-SGD which we call the \textit{optimization risk}.
\begin{definition}[\textbf{Optimization Risk}]
\label{def-cor}
Recall $\bm{w}_\text{priv}$ is the output of DP-SGD (Alg. \ref{alg:1}) after $T$ iterations.
\\
\textbf{(i)} Suppose $f$ is convex. We define the \textit{convex optimization risk} as $\textup{OR}(T) := \big(\mathbb{E}[f(\bm{w}_\text{priv})] - f(\bm{w}^{*})\big)$, where $\bm{w}^{*} \in \text{argmin}_{\bm{w} \in \mathcal{W}} f(\bm{w})$. 

\textbf{(ii)} Suppose $f$ is nonconvex \& $\mathcal{W} = \mathbb{R}^d$. We define the \textit{nonconvex optimization risk} as $\textup{OR}(T) :=  \mathbb{E}[\|\nabla f(\bm{w}_\text{priv})\|^2]$.

Note that the expectations above are w.r.t. the randomness of \Cref{alg:1} (in particular, conditioned on the dataset $\mathcal{Z}$).
\end{definition}
\vspace{-0.3 cm}
Also recall the key quantity $\varphi = \frac{\sqrt{\nu d  \log({1}/{\delta})}}{n\varepsilon} < 1$ defined in \cref{eq:def-vaphi}. Our bounds on the optimization risk will be in terms of $\varphi$.
{For brevity, we only present abridged versions of our results in the main paper and provide the full versions and proofs in the Appendix.}

\section{Generalized Lipschitzness}
\label{dp-sgd-setting}
Here we introduce our \textit{proposed} generalization to the commonly used uniform Lipschitzness assumption. 
\begin{assumption}[\textbf{Generalized Lipschitzness}]
\label{asmp-lip-gen}
For any $\bm{w} \in \mathcal{W}$ and $(\bm{x},y) \sim \mathcal{D}$, the following holds for some sample-dependent function $G(\bm{x},y)$:
\begin{equation*}
    \|\nabla_{\bm{w}} \ell(\bm{w}, \bm{x}, y)\| \leq G(\bm{x}, y),
\end{equation*}
where $\ell$ is our loss function (mentioned before \cref{eq:nov17-100}). We call $G(\bm{x}, y)$ the \enquote{per-sample Lipschitz constant}.\footnote{This assumption extends to non-differentiable functions by replacing the gradient with sub-gradient.}
\end{assumption}
\vspace{-0.2 cm}
Note that we are \textit{not} imposing the condition that $G(\bm{x}, y)$ be itself bounded for all $(\bm{x}, y)$. In fact, if we do impose that, then we recover uniform Lipschitzness. We now provide an important example where generalized Lipschitzness holds.

\textbf{Logistic regression:} 
Consider doing logistic regression for multi-class classification with the cross-entropy loss, where $m$ is the number of classes. Suppose $\bm{x} \sim \mathcal{F}$ (with a `1' appended to account for the bias term) is the feature and $y \in [m]$ is the corresponding class number.
In \Cref{log-reg-asmp}, we show that $\|\nabla \ell_{\bm{w}}(\bm{w}, \bm{x}, y)\| \leq \sqrt{2} \|\bm{x}\|$. Thus, \Cref{asmp-lip-gen} holds with $G(\bm{x},y) = \sqrt{2}\|\bm{x}\|$ \textit{for any} $\mathcal{W}$.

\section{Clip Norm Selection for DP-SGD under Generalized Lipschitzness}
\label{sec:1}
{Here, we will provide a principled strategy to select the constant clip norm $\tau$ (which does not change across iterations) in DP-SGD (\Cref{alg:1}) under generalized Lipschitzness when the maximum per-sample Lipschitz constant is bounded, i.e., uniform Lipschitzness holds.} To that end, we assume the following.
\begin{assumption}
\label{asmp-lip}
\Cref{asmp-lip-gen} holds for the dataset $\mathcal{Z} = \{(\bm{x}_i, y_i)\}_{i=1}^n$ that we receive. Let $G_i = G(\bm{x}_i, y_i)$ for $i \in [n]$. Also, without loss of generality, the sample indices are arranged so that $0 < G_1 < \ldots < G_n < \infty$\footnote{We assume strict inequalities here because the probability measure of equality holding is zero. Also, we assume $G_1 > 0$, as otherwise $f_1$ is a constant function which is trivially minimized everywhere, and we can minimize $\frac{1}{n-1}\sum_{i=2}^n f_i(\bm{w})$ instead.}.
\end{assumption}
Thus, $\{G_i\}_{i=1}^n$ are the per-sample Lipschitz constants for the dataset $\mathcal{Z}$. For e.g., logistic regression with cross-entropy loss satisfies \Cref{asmp-lip} with $G_i = \sqrt{2}\|\bm{x}_i\|$ {(see discussion on logistic regression after \Cref{asmp-lip-gen})}.

Under \Cref{asmp-lip}, if we follow the approach of prior theoretical works such as \citet{bassily2014private}, then we would choose $G_n$ as the clip norm $\tau$ -- this is associated with zero bias (as no actual clipping occurs) but high noise variance, yielding a risk bound of $\mathcal{O}(G_n \varphi)$ in the convex case. While the dependence on $\varphi$ is tight \cite{bassily2014private}, it is not clear if $\tau = G_n$ leads to the best \textbf{constant factors} in the risk bound -- which is what makes a difference in practice. {Using Thm. \ref{thm-small-clip} of this work, \textit{we show that the best constant factors are obtained by choosing $\tau \leq G_1$ in convex problems with near-interpolation like conditions} (i.e., the training data can be almost perfectly fitted) while retaining the $\mathcal{O}(G_n \varphi)$ dependence. This is consistent with experiments in \Cref{expts}, where clip norms $\leq G_1$ perform the best. Intuitively, this happens because the high noise variance associated with large clip norms is more detrimental to convergence 
than the bias associated with small clip norms. 
Let us now talk about this result in more detail.}

\subsection{Convex Case}
\label{clip-norm-choice}
The results here are for a general convex constraint set $\mathcal{W}$ which can be $\mathbb{R}^d$ too. 
Before we can state our result, we need to introduce some definitions first. Let $\Psi := \text{arg min}_{\bm{w} \in \mathcal{W}} f(\bm{w})$ and $f_i^{*} = \min_{\bm{w} \in \mathcal{W}} f_i(\bm{w})$. 
\begin{definition}
\label{het-def}
For any $\bm{w}^{*} \in \Psi$, define $\Delta(\bm{w}^{*}) := \frac{1}{n}\sum_{i \in [n]} (f_i(\bm{w}^{*}) - f_i^{*})$.
\end{definition}

\begin{definition}
\label{may2-asmp-1}
Suppose Assumption \ref{asmp-lip} holds. For clip norm $\tau \in (0, G_n]$, define:
\begin{equation*}
    \alpha(\tau) := 
    \inf_{\bm{w} \in \mathcal{W} \setminus \Psi} \frac{\frac{1}{n}\sum_{i \in [n]}\min\big(\frac{1}{\tau},\frac{1}{G_i}\big) \big({f_i(\bm{w}) - f_i^{*}}\big)}{\frac{1}{n}\sum_{i \in [n]}\frac{f_i(\bm{w}) - f_i^{*}}{G_n}}.
\end{equation*}
Note that: 
\\
\textup{{(i)}} $\alpha(\tau) \geq 1$ for all $\tau \in (0, G_n]$ and $\alpha(G_n) = 1$.
\\
\textup{{(ii)}} $\alpha(\tau)$ is a non-increasing function of $\tau$.
\\
\textup{{(iii)}} $\alpha(\tau) = \alpha(G_1)$ for all $\tau \in (0,G_1]$.
\\
\textup{{(iv)}} $\alpha(G_1)$ is strictly greater than 1 \textit{unless} there exists a $\widetilde{\bm{w}}^{\ast}$ such that $\widetilde{\bm{w}}^{\ast}$ is a minimizer of $\{f_i\}_{i=1}^{n-1}$ but not of $f_n$.
\\
\textup{{(v)}} $\frac{G_n}{\tau \alpha(\tau)} \geq 1$ for all $\tau \in (0, G_n]$.
\\
\textup{{(vi)}} $\frac{G_n}{\tau \alpha(\tau)}$ is a non-increasing function of $\tau$.
\end{definition}

Let us see why $\alpha(\tau) \geq 1$ in \Cref{may2-asmp-1}. By definition, $f_i(\bm{w}) - f_i^{*} \geq 0$ for all $\bm{w} \in \mathcal{W}$. Now since $G_1 < \ldots < G_n$ (as per \Cref{asmp-lip}) and $\tau \leq G_n$, we have that: 
\begin{equation}
    \label{eq:may4-1}
    \min\Big(\frac{1}{\tau},\frac{1}{G_i}\Big) \big({f_i(\bm{w}) - f_i^{*}}\big) \geq \frac{f_i(\bm{w}) - f_i^{*}}{G_n} \text{ } \forall  i \in [n].
\end{equation}
Thus, $\alpha(\tau) \geq 1$ for all $\tau \leq G_n$. 
(ii) and (iii) are easy to verify using properties of $\min()$. Let us now discuss why (iv) must be true. For $\tau = G_1$, the only way equality will hold in \cref{eq:may4-1} for some $\bm{w} \notin \Psi$ is if $f_i(\bm{w}) = f_i^{*}$ $\forall$ $i \in [n-1]$ but $f_n(\bm{w}) > f_n^{*}$; (iv) follows from this. (v) and (vi) can be checked by just writing out the expression for $\frac{G_n}{\tau \alpha(\tau)}$ and then using properties of $\min()$.
We are now ready to present our result for the convex case. 

\begin{theorem}[\textbf{Convex Case}]
\label{thm-small-clip}
Suppose each $f_i$ is convex, $\mathcal{W}$ is a convex set (which can be $\mathbb{R}^d$) and Assumption \ref{asmp-lip} holds. Fix some $C > 0$. In Alg. \ref{alg:1}, set $T = \frac{1}{3 \varphi^2}$ and $\eta_t = \frac{3 C \varphi}{2 \tau}$ $\forall$ $t$ for clip norm $\tau$. 
Then, DP-SGD has the following optimization risk bound as a function of $\tau \in (0, G_n]$:
\begin{multline}
    \label{eq:may6-1}
    \textup{OR}(T) \leq \underbrace{\frac{1}{\alpha(\tau)}\Bigg(\bigg(\frac{\|\bm{w}_0 - \bm{w}^{*}\|^2}{C} + 2 C\bigg) G_n \varphi\Bigg)}_\textup{(A)} \\ + \underbrace{\Bigg(\frac{G_n}{\tau \alpha(\tau)} - 1\Bigg) {\Delta(\bm{w}^{*})}}_\textup{(B)}, 
\end{multline}
where $\Delta(\bm{w}^{*}) \geq 0$ and $\alpha(\tau) \geq 1$ are as defined in Definitions \ref{het-def} and \ref{may2-asmp-1}, respectively.
\end{theorem}
In \cref{eq:may6-1}, term (A) is a non-decreasing function of $\tau$; this follows from (ii) in \Cref{may2-asmp-1}. But term (B) $\geq 0$ is a non-increasing function of $\tau$; this follows from (v) and (vi) in \Cref{may2-asmp-1}. Thus, increasing $\tau$ increases (A) but reduces (B), and vice-versa. So in general, there is a \textbf{tradeoff} between the values of (A) and (B) while setting $\tau$. To provide a recommendation for the choice of $\tau$, we shall focus on a commonly occurring scenario.

\textbf{Interpolation:} Modern \textbf{over-parameterized} ML models are able to perfectly fit all the training data \cite{zhang2021understanding,ma2018power}. In such cases, it is reasonable to make the following \textit{relaxed interpolation} assumption (based on Assumption 1 of \citet{ma2018power}).
\begin{assumption}[\textbf{Relaxed Interpolation}]
\label{overparam}
There exists some $\bm{w}^{*} \in \textup{argmin}_{\bm{w} \in \mathcal{W}} f(\bm{w})$ such that $\Delta(\bm{w}^{*}) = \frac{1}{n}\sum_{i \in [n]} (f_i(\bm{w}^{*}) - f_i^{*}) \approx 0$.
\end{assumption}
\vspace{-0.1 cm}
Assumption 1 of \citet{ma2018power} would imply the existence of some $\bm{w}^{*} \in \textup{argmin}_{\bm{w} \in \mathcal{W}} f(\bm{w})$ such that $\Delta(\bm{w}^{*}) = 0$ (as $\bm{w}^{*} \in \textup{argmin}_{\bm{w} \in \mathcal{W}} f_i(\bm{w})$ $\forall$ $i \in [n]$, per their assumption). One can expect \Cref{overparam} to hold for \textit{separable} classification problems even without over-parameterization.

\textit{Under \Cref{overparam}}, the dominant term in the risk bound of \Cref{thm-small-clip} (Eq. (\ref{eq:may6-1})) would be (A). As discussed previously, (A) is a non-decreasing function of $\tau$. Thus, the lowest risk bound in \cref{eq:may6-1} is obtained for $\tau \in (0,G_1]$. Also recall that $\alpha(\tau) = \alpha(G_1)$ $\forall \text{ } \tau \in (0,G_1]$ (see (iii) in \Cref{may2-asmp-1}).  Since $\alpha(G_n) = 1$, there is an $\alpha(G_1)$-fold improvement in the risk bound with $\tau \leq G_1$ compared to the naive choice of $\tau = G_n$ when $\Delta(\bm{w}^{*}) = 0$.

\begin{remark}[\textbf{Recommendation}]
\label{reco}
For settings where interpolation holds (such as over-parameterization), we recommend tuning the clip norm only till values up to the minimum per-sample Lipschitz constant.
\end{remark}
{We assume that we have some prior estimate of the minimum Lipschitz constant, just like prior works on uniform Lipschitzness assume that the maximum Lipschitz constant is known. For e.g., it can be estimated privately by applying the {Report Noisy Max} method \cite{dwork2014algorithmic} on the negative Lipschitz constants.}

\subsection{Empirical Results}
\label{expts}
We consider the base problem of private multinomial logistic regression (a convex problem satisfying Asmp. \ref{asmp-lip}) to corroborate our theory and recommendation. But we focus on the more powerful application of privately training (only) the last (softmax) layer of deep networks pre-trained on public data which is equivalent to performing logistic regression with features being the previous layer's outputs. Training a softmax layer over {pre-trained networks} is a popular and effective scheme for private training \cite{de2022unlocking, mehta2022large}. This is because fine-tuning all the layers may not significantly improve performance but the extra parameters increase the privacy cost. In addition, training only the last layer is computationally much cheaper than fine-tuning all the layers with DP-SGD (due to per-sample clipping).

{Our experiments here are conducted on four vision datasets available in Torchvision, viz., Caltech-256 with 257 classes, Food-101 with 101 classes, CIFAR-100 with 100 classes and CIFAR-10 with 10 classes, and two language datasets available in Hugging Face, viz., TweetEval Emoji with 20 classes and Emotion with 6 classes. For Caltech-256 and Food-101 (resp., CIFAR-100 and CIFAR-10), we use 512-dimensional features obtained from the pre-softmax layer of a pre-trained ResNet-34 (resp., ResNet-18) model on ImageNet for logistic regression, which is equivalent to training the last (i.e., softmax) layer of a pre-trained ResNet-34 (resp., ResNet-18) model. For the language datasets, we use 768-dimensional features obtained from a pre-trained DistilBERT model \cite{sanh2019distilbert} for logistic regression, which is the same as training a linear layer on top of a pre-trained DistilBERT model.}
As mentioned after \Cref{asmp-lip}, the per-sample Lipschitz constant is equal to $\sqrt{2}$ times the norm of the sample's feature vector (with a `1' appended to incorporate the bias term). 
We consider three privacy levels - $(2,10^{-5})$-DP, $(4,10^{-5})$-DP and $(6,10^{-5})$-DP, with {batch size = 500}. We test several values of the clip norm $\tau$, viz., the $0^{\text{th}}$, $10^{\text{th}}$, $20^{\text{th}}$, $40^{\text{th}}$, $80^{\text{th}}$ and $100^{\text{th}}$ percentile of the per-sample Lipschitz constants (as well as some values smaller than the $0^{\text{th}}$ percentile). Note that $G_1$ and $G_n$ correspond to the $0^{\text{th}}$ and $100^{\text{th}}$ percentiles, respectively. For each value of $\tau$, we tune over several values of the constant learning rate $\eta$, viz., $\{0.0001, 0.0003, 0.0006, 0.001, 0.003, 0.006, 0.01, 0.03, 
\\
0.06, 0.1, 0.3, 0.6, 1, 3, 6, 10\}$. PyTorch's Opacus library \cite{opacus} is used for private training. 

In \Cref{fig:results}, we plot the {best test accuracy} obtained for different values of $\tau$ (by tuning over $\eta$) averaged over the last 5 epochs and across 3 different runs. The figure caption discusses the results. {The exact accuracy values are listed in \Cref{expts-table}.
Also, in \Cref{extra-expts}, we show empirical 
results on two more datasets, viz., EMNIST and Fashion-MNIST.
}

\begin{figure*}[!htb]
\centering 
\subfloat[Caltech-256]{
    \label{caltech256}
	\includegraphics[width=0.35\textwidth]{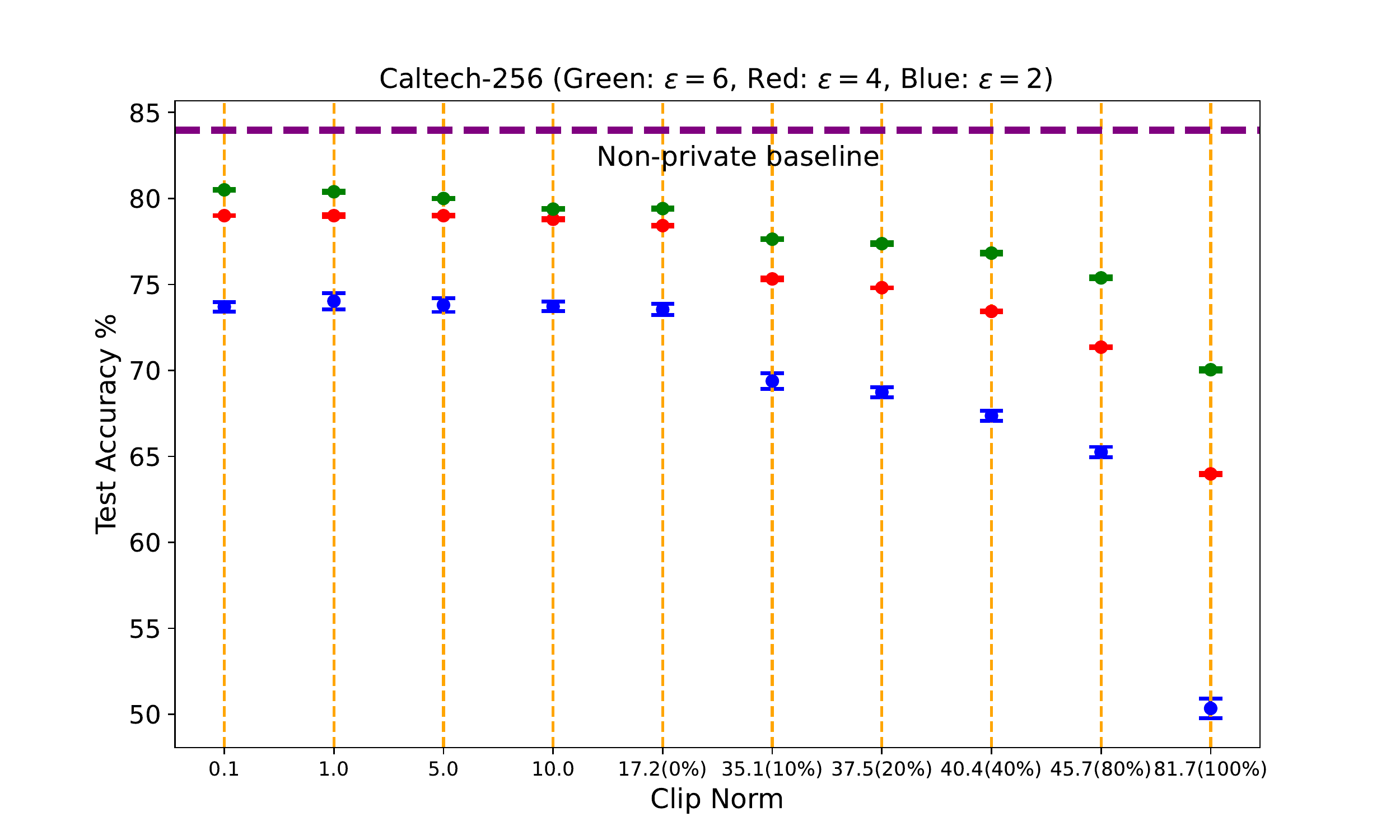}
	}
\kern-2em
\subfloat[Food-101]{
    \label{food101}
	\includegraphics[width=0.35\textwidth]{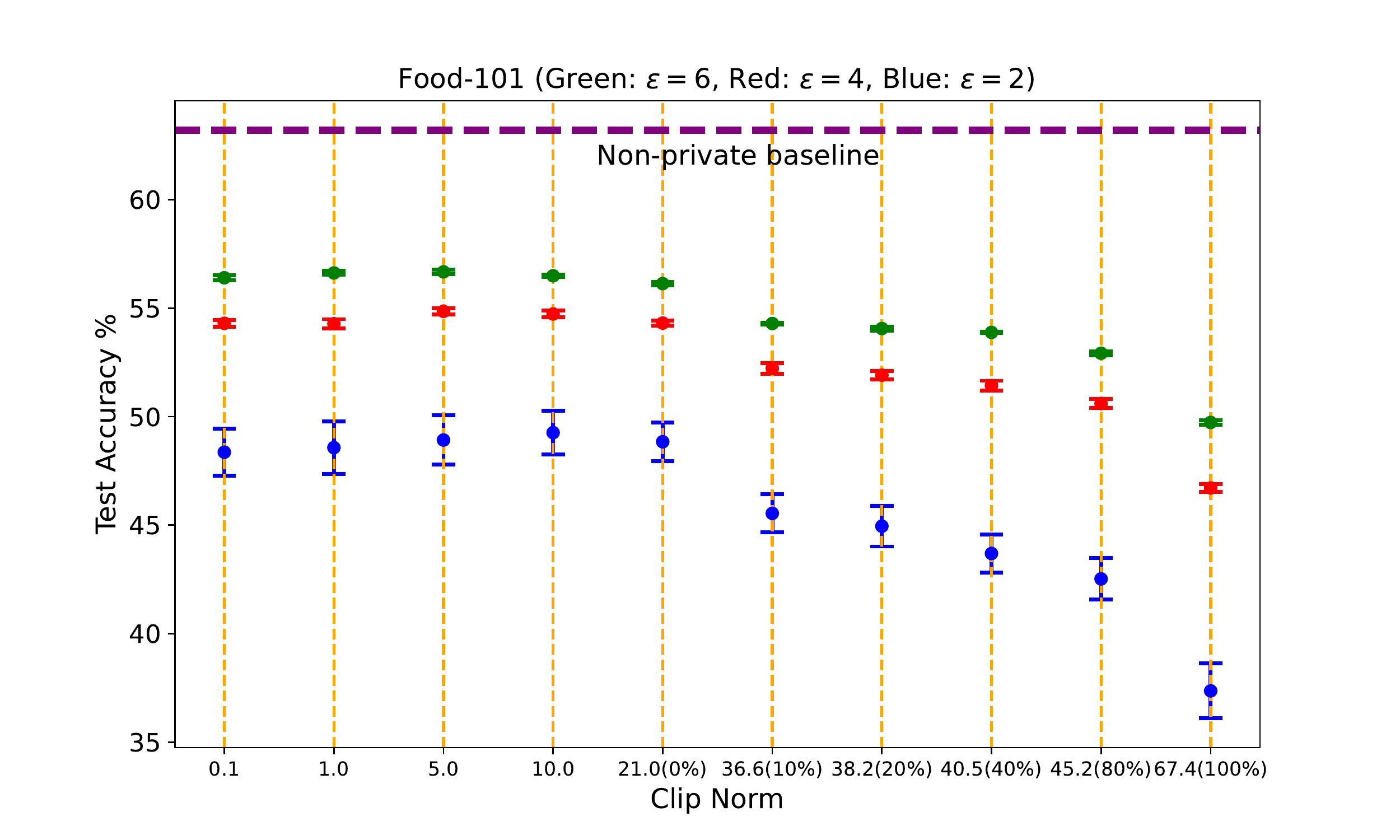}
	}
\kern-2em
\subfloat[CIFAR-100]{
    \label{cifar100}
	\includegraphics[width=0.35\textwidth]{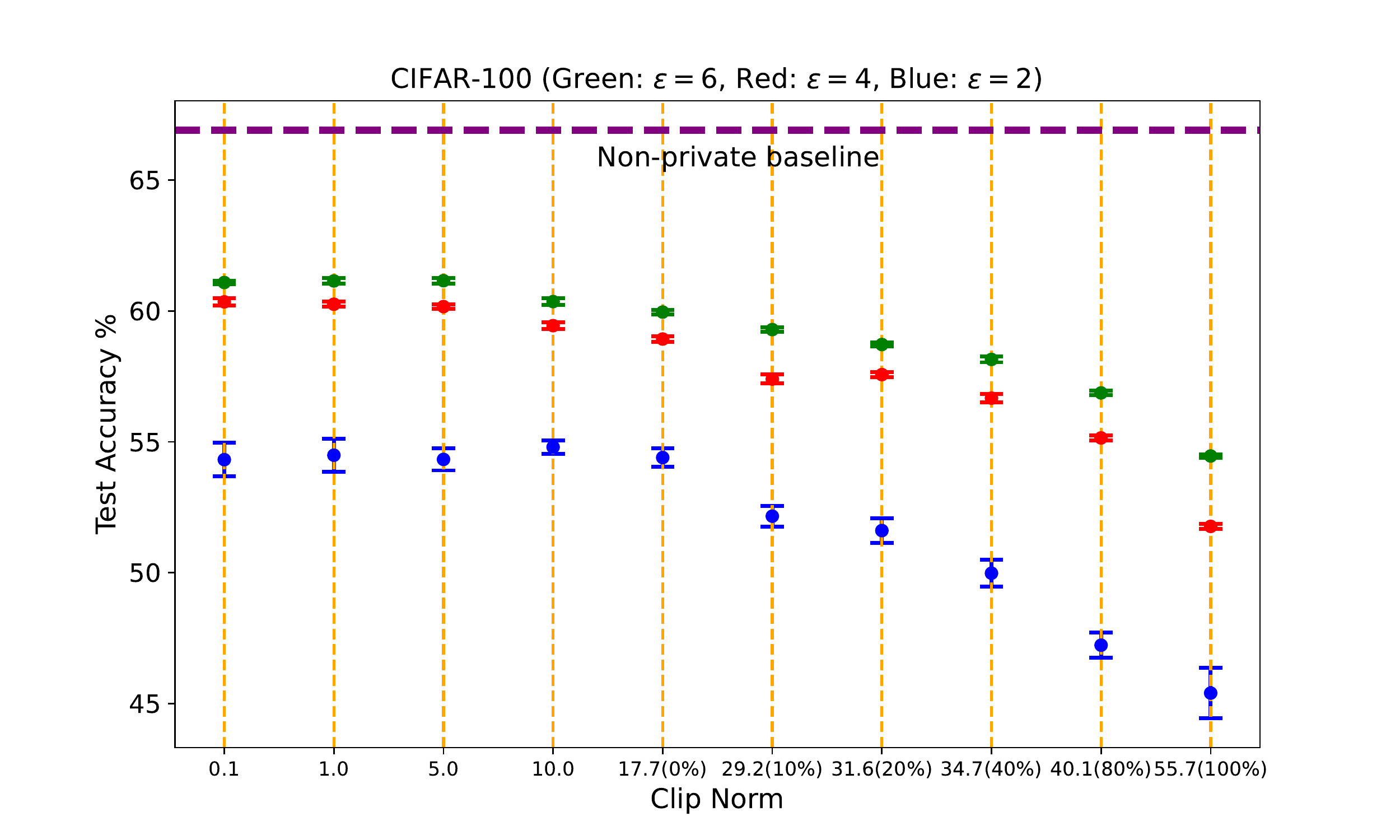}
	} 
\\
\subfloat[CIFAR-10]{
    \label{cifar10}
	\includegraphics[width=0.35\textwidth]{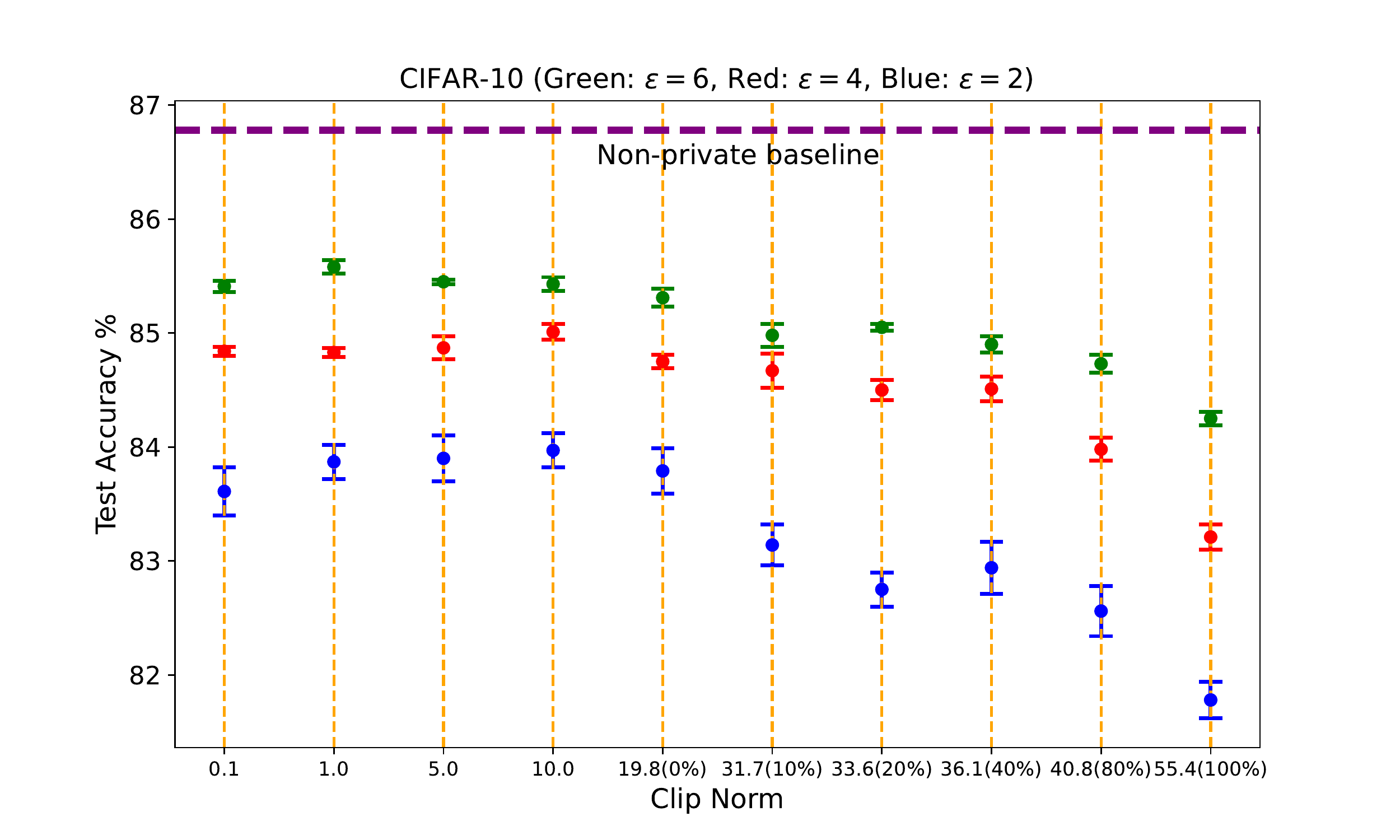}
	} 
\kern-2em
\subfloat[TweetEval Emoji]{
    \label{emoji}
	\includegraphics[width=0.35\textwidth]{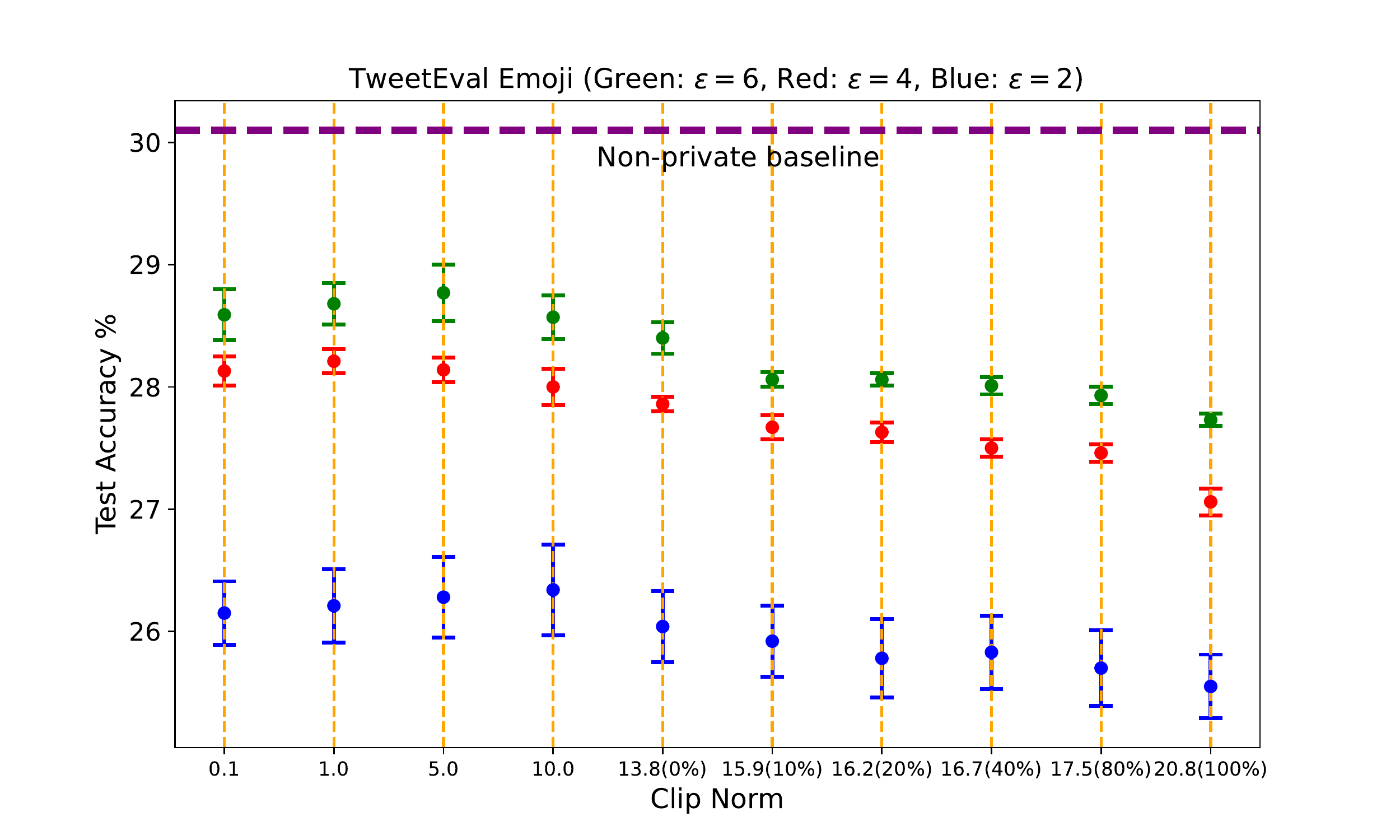}
	} 
\kern-2em
\subfloat[Emotion]{
    \label{emotion}
	\includegraphics[width=0.35\textwidth]{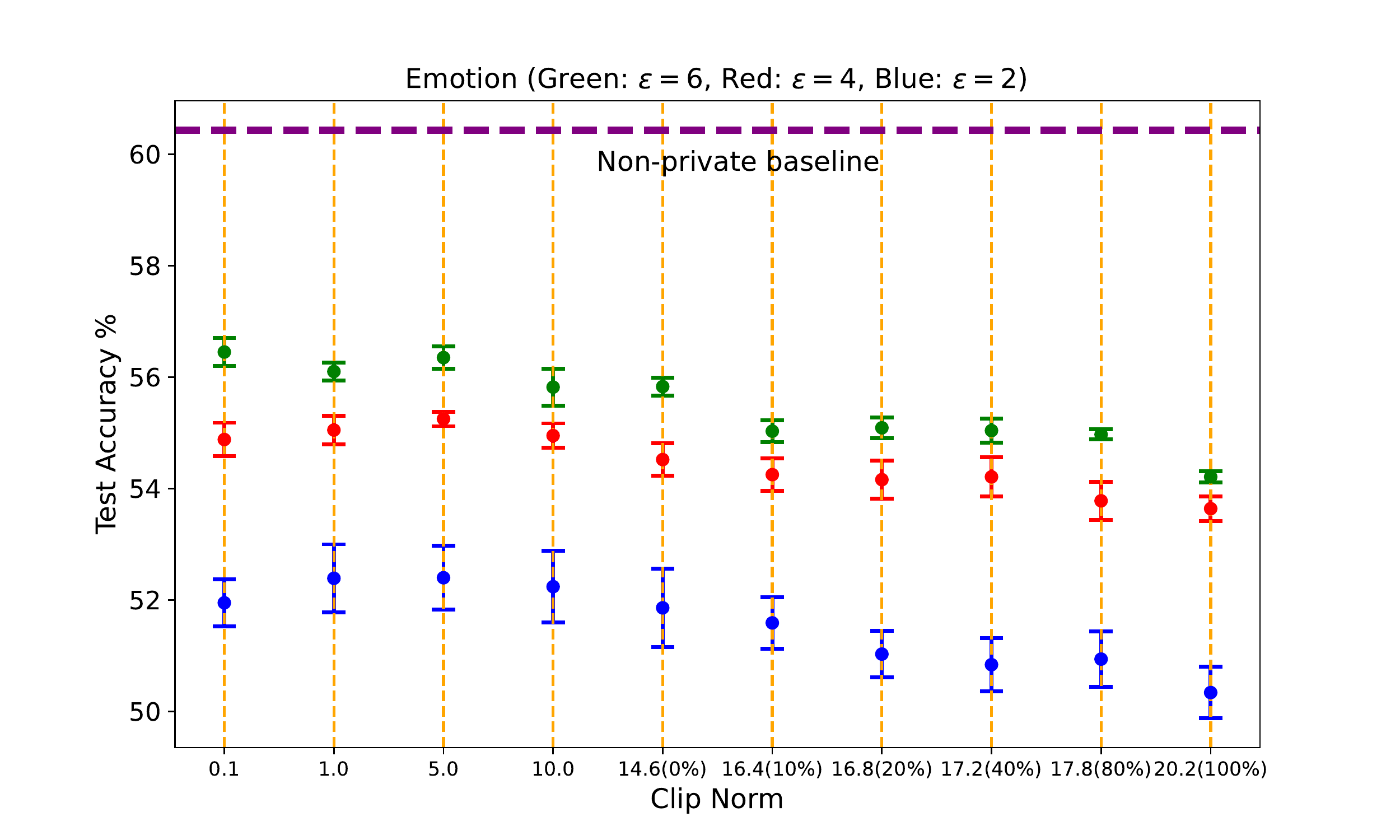}
	} 
\caption{Average test accuracy (depicted by the blobs) $\pm$ 1 standard deviation (depicted by the bars around the blobs) in the last 5 epochs for different values of clip norm $\tau$. \enquote{\%} in the x-axis stands for percentile. Observe that \textit{clip norms $\leq G_1$ ($0^{\text{th}}$ percentile or minimum of per-sample Lipschitz constants) generally perform better than clip norms $> G_1$}. Specifically, the performance with $\tau = G_1$ is significantly better than that with $\tau = G_n$ ($100^{\text{th}}$ percentile or maximum). 
Concretely, for Caltech-256 and Food-101 (datasets with the largest no. of classes), in the case of $\varepsilon = 2$, the mean accuracy with $\tau = G_1$ is better than that with $\tau = G_n$ by $>$ 23\% and  11\%, respectively; the corresponding improvement in the case of $\varepsilon = 4$ is $>$ 15\% and 7.5\%, respectively. These observations are consistent with our theory and recommendation in \Cref{clip-norm-choice}.}
\label{fig:results}
\end{figure*}

{Now that we have corroborated our insight from \Cref{clip-norm-choice} that clip norms $\leq$ $G_1$ yield the best performance, {we shall show some results when $G_1$ is first estimated \textit{privately} using the Report Noisy Max method (applied on the negative Lipschitz constants) and then used in DP-SGD as the clip norm. We abbreviate the Report Noisy Max method as RNMM subsequently.} We show results for the first three datasets in \Cref{fig:results}, viz., Caltech-256, Food-101 and CIFAR-100. We set $\varepsilon = 0.3$ for RNMM and $\varepsilon=\{1.7, 3.7, 5.7\}$ for DP-SGD so that the \textit{total} privacy budget is $\varepsilon=\{2, 4, 6\}$ for which we showed results earlier \textit{without} RNMM (where we did \textit{not} estimate $G_1$ privately).\footnote{We did not optimize the splitting of privacy budget between RNMM and DP-SGD; this is something that can be investigated in future work.} \Cref{tab-new=rnmm} compares the results with and without RNMM. The table caption discusses the results. 
}

{
\begin{table*}[!htb]
\centering
\caption{{Average test accuracy $\pm$ 1 standard deviation in the last 5 epochs (i) \textit{with} and (ii) \textit{without} RNMM as described in the last paragraph of \Cref{expts}; the latter is the best accuracy (w.r.t. the clip norm) among the results in \Cref{fig:results} (which are without RNMM). Note that the difference in accuracy with and without RNMM \textit{reduces as $\varepsilon$ increases}. Specifically, for $\varepsilon = \{4, 6\}$, RNMM does not harm performance too much.}}
\vspace{0.2 cm}
\resizebox{0.95\textwidth}{!}{
\begin{tabular}{|c|c|c|c|c|c|c|}
\hline
\textbf{Dataset} & \makecell{$\varepsilon = 2$ \\ \textbf{w/} RNMM} & \makecell{$\varepsilon = 2$ \\ \textbf{w/o} RNMM} & \makecell{$\varepsilon = 4$ \\ \textbf{w/} RNMM} & \makecell{$\varepsilon = 4$ \\ \textbf{w/o} RNMM} & \makecell{$\varepsilon = 6$ \\ \textbf{w/} RNMM} & \makecell{$\varepsilon = 6$ \\ \textbf{w/o} RNMM}
\\
\hline
Caltech-256 & $70.03\pm0.66$ & $74.03\pm0.47$ & $78.27\pm0.05$ & $79.00\pm0.04$ & $79.53\pm0.07$ & $80.50\pm0.04$
\\
\hline
Food-101 & $45.84\pm2.00$ & $49.26\pm1.00$ & $54.69\pm0.14$ & $54.85\pm0.14$ & $56.30\pm0.12$ & $56.67\pm0.10$
\\
\hline
CIFAR-100 & $52.89\pm1.02$ & $55.10\pm0.26$ & $60.19\pm0.09$ & $60.35\pm0.13$ & $60.89\pm0.04$ & $61.16\pm0.11$
\\
\hline
\end{tabular}
}
\label{tab-new=rnmm}
\end{table*}
}

\section{Convergence of DP-SGD under Heavy-Tailed Lipschitz Constants}
\label{non-lip}
{When uniform Lipschitzness holds, the optimization risk is $\mathcal{O}(\varphi)$ in the convex constrained case \cite{bassily2014private} as well as the smooth nonconvex unconstrained case \cite{wang2018differentially}. In this section, we shall quantify the optimization risk under generalized Lipschitzness (\Cref{asmp-lip-gen}) with the relaxation of the per-sample Lipschitz constants being \textit{heavy-tailed}, instead of being bounded; we \textbf{do not assume interpolation} here. More specifically, we assume that the per-sample Lipschitz constants have bounded $k^\text{th}$ uncentered moment, for some $k>1$, w.r.t. the distribution $\mathcal{D}$. This is formally stated next.
\begin{assumption}[\textbf{Bounded $k^{\text{th}}$ Moment}]
\label{dec23-asmp1}
Suppose \Cref{asmp-lip-gen} holds. For some $k > 1$ and $G > 0$, we have: 
\begin{equation*}
    \Big(\mathbb{E}_{(\bm{x},y) \sim \mathcal{D}}\Big[\big(G(\bm{x}, y)\big)^k\Big]\Big)^{1/k} \leq G.
\end{equation*}
\end{assumption}
Let us revisit the logistic regression example that we discussed after \Cref{asmp-lip-gen} to see why \Cref{dec23-asmp1} is weaker than uniform Lipschitzness. We saw that \Cref{asmp-lip-gen} holds with $G(\bm{x},y) = \sqrt{2}\|\bm{x}\|$ here. Now, if the feature distribution $\mathcal{F}$ is such that its support includes \textit{unbounded} vectors but $\mathbb{E}_{\mathcal{F}}[\|\bm{x}\|^p] \leq G_{\mathcal{F}}^p < \infty$ for some $p > 1$, then \Cref{dec23-asmp1} holds here for $k=p$ and $G = \sqrt{2} G_{\mathcal{F}}$ but uniform Lipschitzness does not hold. However, if the support of $\mathcal{F}$ includes only \textit{bounded} vectors, then both \Cref{dec23-asmp1} and uniform Lipschitzness hold. Also note that uniform Lipschitzness is a special case of \Cref{dec23-asmp1} with $k = \infty$ and a finite $G$.}

{\Cref{dec23-asmp1}\footnote{{This is of significance even in non-private optimization; for e.g., \citet{zhang2020adaptive} show that the gradients of Transformer models such as BERT are heavy-tailed.}} 
is similar to the bounded moment assumption made in \citet{wang2020differentially, kamath2021improved,lowy2022private} for private \textit{stochastic convex optimization} (SCO). But \citet{wang2020differentially, kamath2021improved} assume coordinate-wise bounded moments which we do not, \citet{wang2020differentially} only consider $k=2$ and \citet{kamath2021improved} assume bounded \textit{centered} (around the mean) moment. All these papers provide results for the convex case \textit{within a bounded convex set} (i.e., $\mathcal{W}$ is bounded); we shall focus on the \textbf{unconstrained} (i.e., $\mathcal{W} = \mathbb{R}^d$) \textbf{convex} case here, which is \textit{more practical and harder} (w.r.t. analysis). For completeness, we include a result for ERM in the \textit{constrained} convex case in \Cref{non-lip-cvx} which matches the bound of \citet{kamath2021improved} for SCO. Moreover, we also present a result for the \textbf{unconstrained smooth nonconvex} case; these works do not provide any results for the smooth nonconvex case.}

Before we present our results for the \textit{unconstrained convex} case, let us first discuss the main technical difficulty compared to the \textit{constrained convex} case. The overall optimization bias in the convex case depends on the bias in mean gradient estimation induced due to clipping as well as on the distance of the current point ($\bm{w}_t$) from the optimal point ($\bm{w}^{*}$), i.e., $\|\bm{w}_t - \bm{w}^{*}\|$ at iteration $t$. 
In the constrained convex case, $\|\bm{w}_t - \bm{w}^{*}\|$ can be easily bounded by the diameter of the constraint set. However, in the unconstrained case, bounding $\|\bm{w}_t - \bm{w}^{*}\|$ needs extra work; surprisingly, this has not been analyzed in prior work on DP-SGD. 
Now, we present our result under \Cref{dec23-asmp1} in the \textit{unconstrained} \textit{convex} case. 
\begin{theorem}[\textbf{Unconstrained Convex Case}]
\label{thm1-jan19}
Suppose \Cref{dec23-asmp1} holds, $f$ is convex and $\mathcal{W} = \mathbb{R}^d$. Fix some $\gamma \in (0,1)$ and $C > 0$. In \Cref{alg:1}, set $T = {\varphi^{-2}}$, $\tau = \mathcal{O}\big({G} \gamma^{-\frac{1}{k}} \varphi^{-\frac{2}{k+1}}\big)$ and $\eta_t = \mathcal{O}\big(C \gamma^{\frac{1}{k}} G^{-1} \varphi^{1+\frac{2}{k+1}}\big)$ for all $t < T$. Then with a probability of at least $(1 - \gamma)$ which is w.r.t. the random dataset $\mathcal{Z}$ that we obtain, DP-SGD (\Cref{alg:1}) has the following guarantee:
\begin{equation*}
    \textup{OR}(T) \leq \mathcal{O}\Bigg(\frac{G}{\gamma^{1/k}}\Bigg(\frac{\|\bm{w}_0 - \bm{w}^{*}\|^2}{C} + C \Bigg)\Bigg){\varphi^{(1 - \frac{2}{k+1})}}.
\end{equation*}
\end{theorem}

\begin{remark}[\textbf{Comparison with prior results}]
As per Thm. \ref{thm1-jan19}, the optimization risk is $\mathcal{O}(\varphi^{1 - {\frac{2}{k+1}}})$ in the
the bounded $k^{\text{th}}$ moment \textbf{unconstrained} convex case. In comparison, the risk is $\mathcal{O}(\varphi^{1 - {\frac{1}{k}}})$ in the bounded $k^{\text{th}}$ moment \textbf{constrained} convex case as per Thm. \ref{thm1-dec13} in \Cref{non-lip-cvx}, and $\mathcal{O}(\varphi)$ in the uniform Lipschitz convex case \cite{bassily2014private}.
\end{remark}
{The difference in the risk bound between the constrained and unconstrained cases arises because $\|\bm{w}_t - \bm{w}^{*}\| = \mathcal{O}(1)$ (w.r.t. $\varphi$) in the constrained case but our bound for $\|\bm{w}_t - \bm{w}^{*}\|$ in the unconstrained case depends on $\varphi$. If one is able to improve this bound for the unconstrained case, then the risk for the unconstrained case will also improve; but doing so in the absence of any other assumption seems difficult.
}

Under a mild additional assumption, we are able to improve the risk bound to $\mathcal{O}(\varphi^{1 - \frac{1}{k}})$ in the unconstrained case, thereby matching the result in the constrained case. We present this extra assumption first, followed by the result.
\begin{assumption}
\label{asmp1-may4}
For any $\bm{w}$ s.t. $\|\bm{w} - \bm{w}^{*}\| > D = \mathcal{O}(1)$, where $\bm{w}^{*} \in \text{arg min}_{\bm{w} \in \mathbb{R}^d} f(\bm{w})$:
\begin{equation}
    \label{eq:asmp}
    f(\bm{w}) - f(\bm{w}^{*}) > \Big({16 \varphi^{1 - \frac{1}{k}} {G}}\Big) \|\bm{w} - \bm{w}^{*}\|.
\end{equation}
\end{assumption}
\vspace{-0.3 cm}
For large $n$ (our regime), $\varphi$ is small; thus, ${16 \varphi^{(1 - {1}/{k})} {G}}$ is also small in comparison to the average Lipschitz constant $=\mathcal{O}(G)$ and so, assuming \cref{eq:asmp} \textit{for points far away from the optimum} is reasonable. Besides, an assumption similar to \Cref{asmp1-may4}, known as \textit{sharpness} (or the \L{}ojasiewicz inequality), has been used in prior optimization literature \cite{polyak1979sharp,burke1993weak,bolte2017error,roulet2017sharpness,davis2018subgradient}. {We provide an example for \Cref{asmp1-may4} in Appendix \ref{new-asmp-eg}.}

\begin{theorem}[\textbf{Unconstrained Convex Case Under \Cref{asmp1-may4}}]
\label{thm1-may4}
Suppose Assumptions \ref{dec23-asmp1} and \ref{asmp1-may4} hold, $f$ is convex and $\mathcal{W} = \mathbb{R}^d$. Fix some $C > 0$. In \Cref{alg:1}, set $T = \varphi^{-2}$, $\tau = \mathcal{O}\big(G \varphi^{-\frac{1}{k}}\big)$
and $\eta_t = \mathcal{O}\big(C G^{-1} \varphi^{1+\frac{1}{k}}\big)$ for all $t < T$. Then with a probability of at least $3/4$ which is w.r.t. the random dataset $\mathcal{Z}$ that we obtain,
DP-SGD (Alg. \ref{alg:1}) has the following \textbf{improved} guarantee:
\begin{equation*}
    \textup{OR}(T) \leq \mathcal{O}\Bigg({G}\Bigg(\frac{\|\bm{w}_0 - \bm{w}^{*}\|^2}{C} + C + {D} \Bigg)\Bigg){\varphi^{(1 - \frac{1}{k})}}.
\end{equation*}
\end{theorem}
Thus, the risk bound under \Cref{asmp1-may4} improves to $\mathcal{O}(\varphi^{1 - \frac{1}{k}})$, which matches the bound in the constrained case. One caveat of the above result is that unlike \Cref{thm1-jan19}, the confidence estimate of the high-probability result cannot be controlled by us.
Next, we provide a matching lower bound showing the tightness of \Cref{thm1-may4} w.r.t. $\varphi$.

\begin{theorem}[\textbf{Lower Bound for Unconstrained Convex Case Under \Cref{asmp1-may4}}]
\label{lower_bound}
{Suppose $\delta < \exp(-\varepsilon^2)$.} There exists a convex loss function $\ell$, such that for every $(\varepsilon,\delta)$-DP algorithm $\mathcal{A}$ which tries to solve for $\bm{w}^{*} = \text{arg min}_{\bm{w} \in \mathbb{R}^d} f(\bm{w})$ where $f$ is the average loss for a dataset $\mathcal{Z}$ of $n$ samples drawn from the data distribution $\mathcal{D}$, there exists a choice of $\mathcal{D}$ such that:
\vspace{-0.15 cm}
\begin{itemize}
    \item $f$ satisfies Assumptions \ref{dec23-asmp1} and \ref{asmp1-may4} (the latter up to constant terms and w.h.p. w.r.t. $\mathcal{Z}$).
    \vspace{-0.15 cm}
    \item $\mathbb{E}_{\mathcal{Z} \sim \mathcal{D}^n, \mathcal{A}}\Big[f(\bm{w}_{\mathcal{Z}}^{(\mathcal{A})}) - f(\bm{w}^{*})\Big] \geq \Omega\big(\varphi^{1 - \frac{1}{k}}\big)$, where $\bm{w}_{\mathcal{Z}}^{(\mathcal{A})}$ is the output of algorithm $\mathcal{A}$ on the dataset $\mathcal{Z}$.
\end{itemize}
\end{theorem}
{
Thus, \textbf{DP-SGD is optimal} in the unconstrained convex case under Assumptions \ref{dec23-asmp1} \& \ref{asmp1-may4} {(for $\delta < \exp(-\varepsilon^2)$)}. We also have the following corollary under uniform Lipschitzness.

\begin{corollary}[\textbf{Uniform Lipschitz}]
\label{cor-unif}
For $k \to \infty$, i.e., under uniform Lipschitzness, the risk in the unconstrained convex case regardless of \Cref{asmp1-may4} is $\mathcal{O}(\varphi)$ as per Thm. \ref{thm1-jan19} and \ref{thm1-may4}. Further, Thm. \ref{lower_bound} is a lower bound even w/o \Cref{asmp1-may4}, and it yields $\Omega(\varphi)$ for $k \to \infty$. Thus, \textbf{DP-SGD is optimal} in the unconstrained convex case under uniform Lipschitzness {(for $\delta < \exp(-\varepsilon^2)$)}.
\end{corollary}
}

\textbf{Regarding the proof of \Cref{lower_bound}:} Even though we follow the proof outline of Theorem 6.4 of \citet{kamath2021improved}, our proof is more involved. First, since we are in the \textit{unconstrained} setting, we had to use a \textit{non-obvious} function (to obtain the lower bound on). Second, since we are in the ERM setting, we have to lower bound the expected training error which is harder than lower bounding the expected generalization error in the SCO setting of \citet{kamath2021improved}. {Finally, since we are dealing with $(\varepsilon,\delta)$-DP (whereas \citet{kamath2021improved} deal with $(0,\rho)$-zCDP), we had to 
re-derive an important tool in their analysis. That is also the reason we had to impose $\delta < \exp(-\varepsilon^2)$; removing this constraint is left for future work. Note that this is the first lower bound for $(\varepsilon,\delta)$-DP in the unconstrained case. See \Cref{app-l} for details.}

Finally, we present our result under \Cref{dec23-asmp1} in the \textit{unconstrained nonconvex} case. 
\begin{theorem}[\textbf{Unconstrained Nonconvex Case}]
\label{thm1-dec15}
Suppose \Cref{dec23-asmp1} holds, $f$ is $L$-smooth and $\mathcal{W} = \mathbb{R}^d$. Fix some $\gamma \in (0,1)$. 
In \Cref{alg:1}, set $T = \varphi^{-2}$, $\tau = \mathcal{O}\big(\gamma^{-\frac{2}{2k-1}} \varphi^{-\frac{1}{2k-1}}\big)$ and $\eta_t = \mathcal{O}\big(\gamma^{\frac{2}{2k-1}} \varphi^{1+\frac{1}{2k-1}}\big)$ for all $t < T$.
Then with a probability of at least $(1 - \gamma)$ which is w.r.t. the random dataset $\mathcal{Z}$ that we obtain, DP-SGD (Alg. \ref{alg:1}) has the following guarantee:
\vspace{-0.1 cm}
\begin{equation*}
    \textup{OR}(T) \leq \mathcal{O}\Big({\gamma^{\frac{-2}{2k-1}}}\varphi^{(1-\frac{1}{2k-1})}\Big).
\end{equation*}
\end{theorem}
\begin{remark}[\textbf{Comparison with Prior Results}]
\label{rem-2}
As per Thm. \ref{thm1-dec15}, the risk is $\mathcal{O}(\varphi^{1 - {\frac{1}{2k-1}}})$ in the bounded $k^{\text{th}}$ moment {nonconvex} case. In comparison, \citet{wang2018differentially,wang2019efficient} get a bound of $\mathcal{O}(\varphi)$ in the uniform \textbf{Lipschitz} nonconvex case (equivalent to $k = \infty$) with DP-SGD like algorithms. 
\end{remark}

{
\section{Limitations}
We now discuss some limitations of our work which render interesting directions of future work. We do not have a lower bound for the nonconvex case in \Cref{non-lip}. Our clip norm selection strategy in \Cref{sec:1} is only for the convex case. Our iteration complexity ($T$) is worse than some related works because we consider the vanilla DP-SGD algorithm of \citet{abadi2016deep} \textit{without any acceleration}, whereas other works consider accelerated variants of DP-SGD. In our case, we had to set $T = \mathcal{O}(\nicefrac{1}{\varphi^2})$ in order to control the clipping bias. But please note that the goal of this work is to attain smaller optimization risk by properly choosing the clip norm, and \textit{not} accelerate convergence. 

\section{Conclusion}
\label{sec:conc}
In this paper, we generalized uniform Lipschitzness by assuming that the per-sample gradients have {sample-dependent} upper bounds called per-sample Lipschitz constants. Under our generalized Lipschitzness assumption, we provided a theoretically-motivated clip norm {tuning} recommendation for DP-SGD in the convex case. We showed the effectiveness of our recommendation via extensive experiments. Additionally, we derived novel convergence results for unconstrained DP-SGD when the per-sample Lipschitz constants are heavy-tailed, i.e., they have bounded moments. 
}

\bibliography{refs}
\bibliographystyle{icml2023}

\onecolumn
\newpage
\appendix

\begin{center}
    \textbf{\Large Appendix}\vspace{5mm}
\end{center}

\section*{Contents}
\vspace{0.1 cm}
\begin{itemize}
    \item \Cref{extra-expts}: Empirical Results on EMNIST and Fashion-MNIST
    \item \Cref{expts-table}: Tables of Results for Section~\ref{expts} and Appendix~\ref{extra-expts}
    \item \Cref{sec:noncvx-exp}: Some Empirical Results in the Non-Convex Case
    \item \Cref{log-reg-asmp}: Logistic Regression Satisfies Assumption~\ref{asmp-lip-gen}
    \item \Cref{new-asmp-eg}: Example for Assumption~\ref{asmp1-may4}
    \item \Cref{app-f}: Some Useful Lemmas
    \item \Cref{pf-dp}: Proof of Theorem~\ref{thm-dp}
    \item \Cref{small-clip-sec}: Full Version and Proof of Theorem~\ref{thm-small-clip}
    \item \Cref{non-lip-cvx}: Result for the Constrained Convex Case under Assumption~\ref{dec23-asmp1}
    \item \Cref{full-1}: Full Version and Proof of Theorem~\ref{thm1-jan19}
    \item \Cref{sec-improve-may4}: Full Version and Proof of Theorem~\ref{thm1-may4}
    \item \Cref{app-l}: Full Version and Proof of Theorem~\ref{lower_bound}
    \item \Cref{lb_cons_pf}: Proof of Theorem~\ref{lower_bound_constrained}
    \item \Cref{full-2}: Full Version and Proof of Theorem~\ref{thm1-dec15}
\end{itemize}

\clearpage
\section{Empirical Results on EMNIST and Fashion-MNIST}
\label{extra-expts}
Here, we show results for private multinomial logistic
regression on two more datasets, viz., (balanced) EMNIST with 47 classes and Fashion-MNIST with 10 classes. For both these datasets, we just use the flattened images as features for logistic
regression. Other experimental details are the same as in \Cref{expts}. 

In \Cref{fig:supp-extra}, we plot the {best test accuracy} obtained for different values of the clip norm $\tau$ (by tuning over the learning rate $\eta$) averaged over the last 5 epochs and across 3 independent runs (just like \Cref{fig:results}). The figure caption discusses the results. {The exact accuracy values for these two datasets are tabulated in Table \ref{tab3-nlp} in \Cref{expts-table}.} 

\begin{figure*}[!htb]
\centering 
\subfloat[EMNIST]{
	\includegraphics[width=0.75\textwidth]{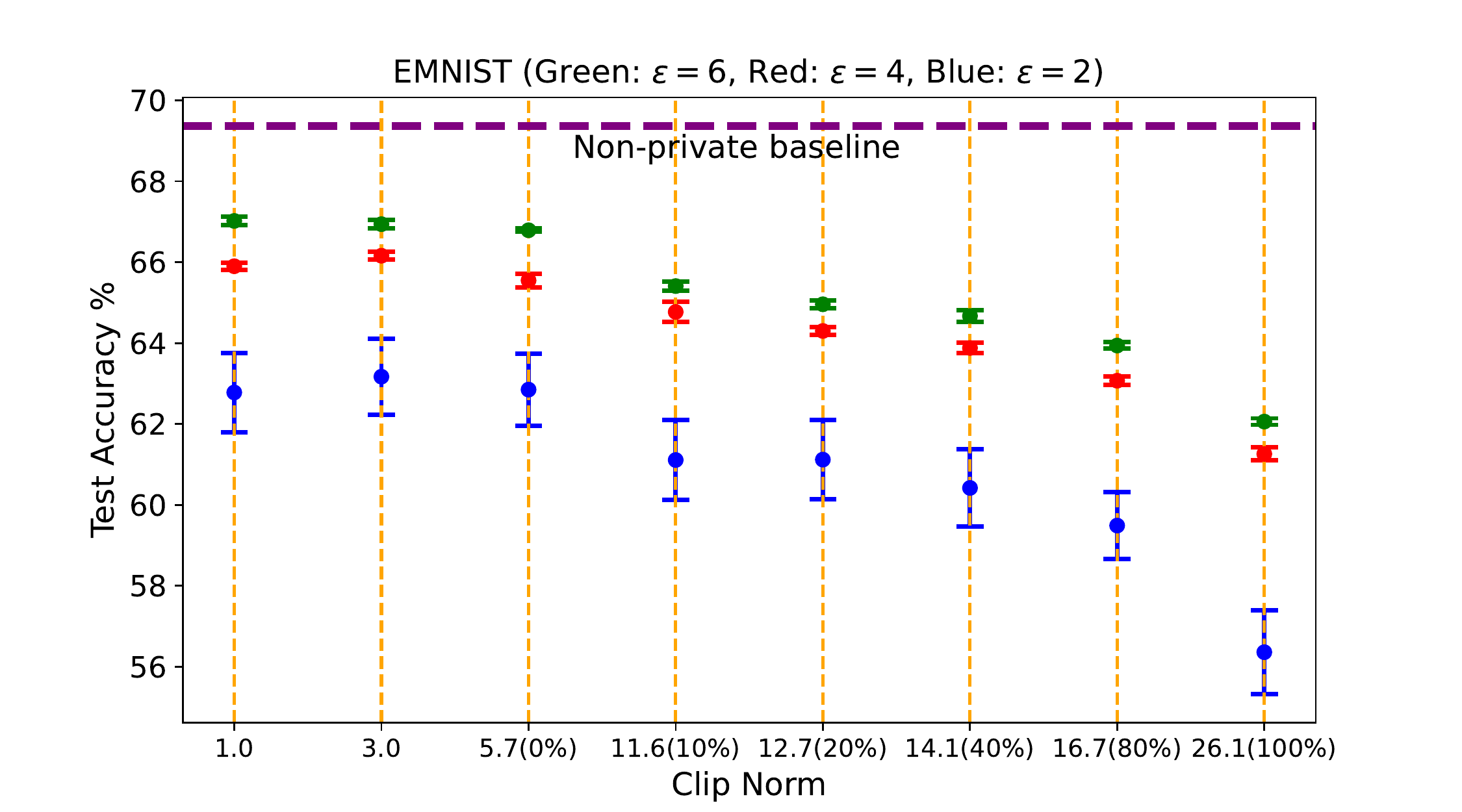}
}
\\
\subfloat[Fashion-MNIST]{
	\includegraphics[width=0.75\textwidth]{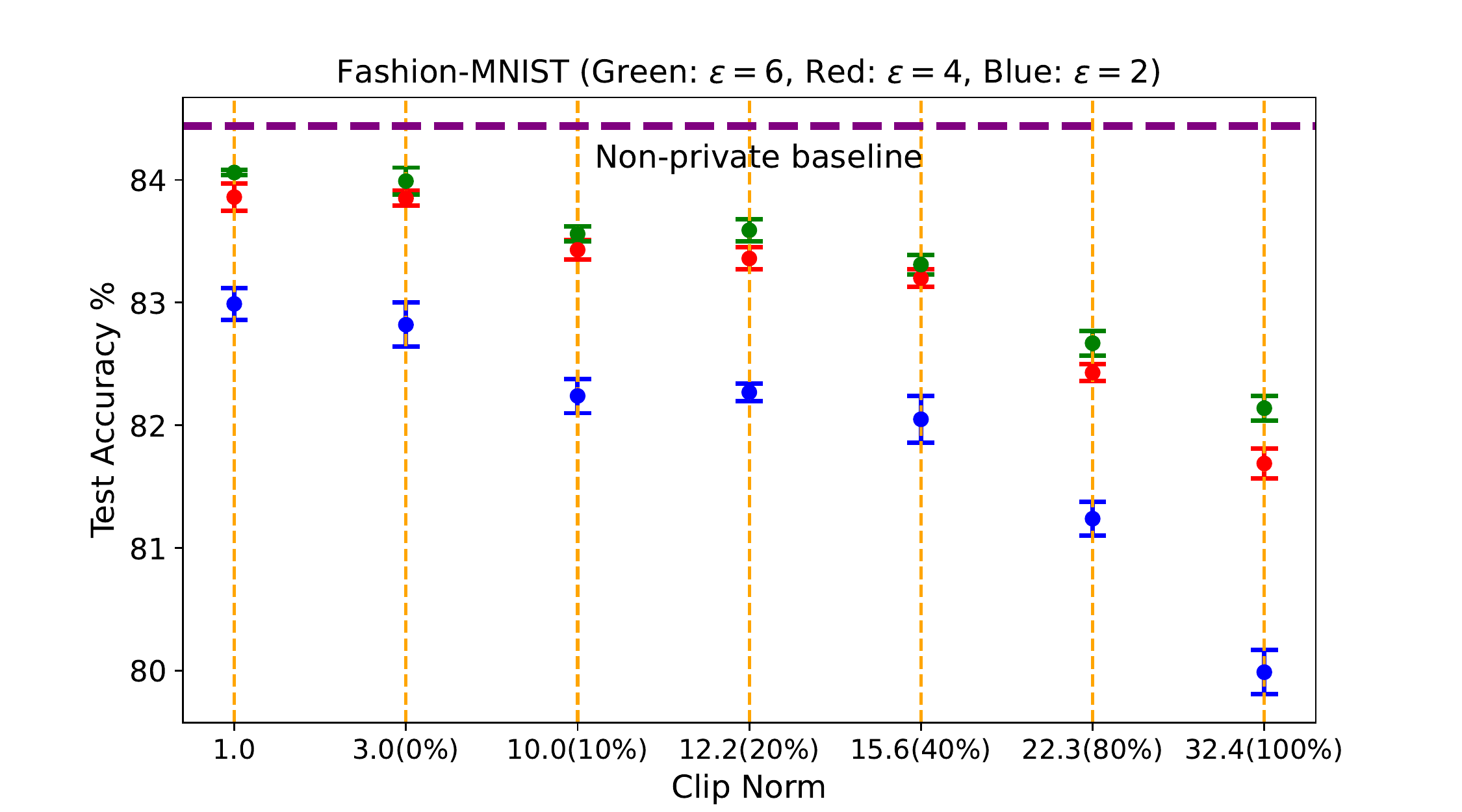}
	} 
\caption{Average test accuracy (depicted by the blobs) $\pm$ 1 standard deviation (depicted by the bars around the blobs) in the last 5 epochs for different values of clip norm $\tau$. \enquote{\%} in the x-axis stands for percentile. Just as in \Cref{fig:results}, \textit{clip norms $\leq G_1$ ($0^{\text{th}}$ percentile of per-sample Lipschitz constants) perform better than clip norms $> G_1$, and the performance with $\tau = G_1$ is much better than that with $\tau = G_n$ ($100^{\text{th}}$ percentile)}. These observations validate our theory and recommendation in \Cref{clip-norm-choice}. }
\label{fig:supp-extra}
\end{figure*}
\clearpage
\section{Tables of Results for Section~\ref{expts} and Appendix~\ref{extra-expts}}
\label{expts-table}
\begin{table}[!htb]
\caption{\textbf{Caltech-256, Food-101 and CIFAR-100:} Average test accuracy $\pm$ 1 standard deviation in the last 5 epochs for different values of clip norm $\tau$ in the experiments of \Cref{expts}. Note that \enquote{pctl.} stands for percentile. \enquote{Non-private baseline} is the accuracy of vanilla non-private SGD in the same setting.}
\vspace{0.3 cm}
\begin{subtable}[c]{\textwidth}
\centering
\begin{tabular}{|l|c|c|c|}
\hline
\textbf{Caltech-256} & \textbf{(a)} $(2,10^{-5})$-DP & \textbf{(b)} $(4,10^{-5})$-DP & \textbf{(c)} $(6,10^{-5})$-DP
\\
\hline
$\tau = 0.1$ & 73.70 $\pm$ 0.28 \% & 79.00 $\pm$ 0.03 \% & 80.50 $\pm$ 0.04 \%
\\
\hline
$\tau = 1.0$ & 74.03 $\pm$ 0.47 \% & 79.00 $\pm$ 0.08 \% & 80.39 $\pm$ 0.07 \%
\\
\hline
$\tau = 5.0$ & 73.80 $\pm$ 0.40 \% & 79.00 $\pm$ 0.04 \% & 80.00 $\pm$ 0.04 \%
\\
\hline
$\tau = 10.0$ & 73.72 $\pm$ 0.28 \% & 78.79 $\pm$ 0.08 \% & 79.38 $\pm$ 0.05 \%
\\
\hline
$\tau = (0^{\text{th}} \text{ pctl.})$ & 73.54 $\pm$ 0.32 \% & 78.42 $\pm$ 0.05 \% & 79.41 $\pm$ 0.06 \%
\\
\hline
$\tau = (10^{\text{th}} \text{ pctl.})$ & 69.38 $\pm$ 0.45 \% & 75.32 $\pm$ 0.08 \% & 77.63 $\pm$ 0.05 \%
\\
\hline
$\tau = (20^{\text{th}} \text{ pctl.})$ & 68.73 $\pm$ 0.30 \% & 74.81 $\pm$ 0.02 \% & 77.37 $\pm$ 0.08 \%
\\
\hline
$\tau = (40^{\text{th}} \text{ pctl.})$ & 67.36 $\pm$ 0.29 \% & 73.43 $\pm$ 0.06 \% & 76.82 $\pm$ 0.09 \%
\\
\hline
$\tau = (80^{\text{th}} \text{ pctl.})$ & 65.25 $\pm$ 0.30 \% & 71.35 $\pm$ 0.05 \% & 75.38 $\pm$ 0.08 \%
\\
\hline
$\tau = (100^{\text{th}} \text{ pctl.})$ & 50.34 $\pm$ 0.57 \% & 63.98 $\pm$ 0.07 \% & 70.04 $\pm$ 0.10 \%
\\
\hline
Non-private baseline & \multicolumn{3}{|c|}{\texttt{83.97 $\pm$ 0.02} \%}
\\
\hline
\end{tabular}
\end{subtable}
\\
\\
\\
\begin{subtable}[c]{\textwidth}
\centering
\begin{tabular}{|l|c|c|c|}
\hline
\textbf{Food-101} & \textbf{(a)} $(2,10^{-5})$-DP & \textbf{(b)} $(4,10^{-5})$-DP & \textbf{(c)} $(6,10^{-5})$-DP
\\
\hline
$\tau = 0.1$ & 48.36 $\pm$ 1.08 \% & 54.30 $\pm$ 0.16 \% & 56.40 $\pm$ 0.12 \%
\\
\hline
$\tau = 1.0$ & 48.57 $\pm$ 1.22 \% & 54.28 $\pm$ 0.21 \% & 56.62 $\pm$ 0.09 \%
\\
\hline
$\tau = 5.0$ & 48.92 $\pm$ 1.14 \% & 54.85 $\pm$ 0.14 \% & 56.67 $\pm$ 0.10 \%
\\
\hline
$\tau = 10.0$ & 49.26 $\pm$ 1.00 \% & 54.73 $\pm$ 0.15 \% & 56.49 $\pm$ 0.05 \%
\\
\hline
$\tau = (0^{\text{th}} \text{ pctl.})$ & 48.84 $\pm$ 0.89 \% & 54.31 $\pm$ 0.11 \% & 56.13 $\pm$ 0.08 \%
\\
\hline
$\tau = (10^{\text{th}} \text{ pctl.})$ & 45.54 $\pm$ 0.88 \% & 52.22 $\pm$ 0.24 \% & 54.29 $\pm$ 0.04 \%
\\
\hline
$\tau = (20^{\text{th}} \text{ pctl.})$ & 44.95 $\pm$ 0.94 \% & 51.91 $\pm$ 0.20 \% & 54.06 $\pm$ 0.09 \%
\\
\hline
$\tau = (40^{\text{th}} \text{ pctl.})$ & 43.69 $\pm$ 0.87 \% & 51.43 $\pm$ 0.22 \% & 53.88 $\pm$ 0.02 \%
\\
\hline
$\tau = (80^{\text{th}} \text{ pctl.})$ & 42.52 $\pm$ 0.96 \% & 50.60 $\pm$ 0.20 \% & 52.92 $\pm$ 0.09 \%
\\
\hline
$\tau = (100^{\text{th}} \text{ pctl.})$ & 37.36 $\pm$ 1.26 \% & 46.71 $\pm$ 0.18 \% & 49.73 $\pm$ 0.11 \%
\\
\hline
Non-private baseline & \multicolumn{3}{|c|}{\texttt{63.20 $\pm$ 0.06} \%}
\\
\hline
\end{tabular}
\end{subtable}
\\
\\
\\
\begin{subtable}[c]{\textwidth}
\centering
\begin{tabular}{|l|c|c|c|}
\hline
\textbf{CIFAR-100} & \textbf{(a)} $(2,10^{-5})$-DP & \textbf{(b)} $(4,10^{-5})$-DP & \textbf{(c)} $(6,10^{-5})$-DP
\\
\hline
$\tau = 0.1$ & 54.32 $\pm$ 0.64 \% & 60.35 $\pm$ 0.13 \% & 61.09 $\pm$ 0.07 \%
\\
\hline
$\tau = 1.0$ & 54.49 $\pm$ 0.63 \% & 60.26 $\pm$ 0.10 \% & 61.15 $\pm$ 0.11 \%
\\
\hline
$\tau = 5.0$ & 54.33 $\pm$ 0.42 \% & 60.17 $\pm$ 0.08 \% & 61.16 $\pm$ 0.11 \%
\\
\hline
$\tau = 10.0$ & 55.10 $\pm$ 0.26 \% & 59.44 $\pm$ 0.12 \% & 60.36 $\pm$ 0.12 \%
\\
\hline
$\tau = 17.7(0^{\text{th}} \text{ pctl.})$ & 54.40 $\pm$ 0.36 \% & 58.93 $\pm$ 0.10 \% &  59.96 $\pm$ 0.09 \%
\\
\hline
$\tau = 29.2(10^{\text{th}} \text{ pctl.})$ & 52.16 $\pm$ 0.40 \% & 57.40 $\pm$ 0.17 \% & 59.89 $\pm$ 0.09 \%
\\
\hline
$\tau = 31.6(20^{\text{th}} \text{ pctl.})$ & 51.61 $\pm$ 0.47 \% & 57.57 $\pm$ 0.10 \% & 58.72 $\pm$ 0.08 \%
\\
\hline
$\tau = 34.7(40^{\text{th}} \text{ pctl.})$ & 49.98 $\pm$ 0.52 \% & 56.67 $\pm$ 0.16 \% & 58.15 $\pm$ 0.11 \%
\\
\hline
$\tau = 40.1(80^{\text{th}} \text{ pctl.})$ & 47.23 $\pm$ 0.48 \% & 55.15 $\pm$ 0.09 \% & 56.87 $\pm$ 0.08 \%
\\
\hline
$\tau = 55.7(100^{\text{th}} \text{ pctl.})$ & 45.40 $\pm$ 0.96 \% & 51.77 $\pm$ 0.09 \% & 54.46 $\pm$ 0.07 \%
\\
\hline
Non-private baseline & \multicolumn{3}{|c|}{\texttt{66.91 $\pm$ 0.05} \%}
\\
\hline
\end{tabular}
\end{subtable}
\label{tab3}
\end{table}

\begin{table}[!htb]
\caption{\textbf{CIFAR-10, TweetEval Emoji and Emotion:} Average test accuracy $\pm$ 1 standard deviation in the last 5 epochs for different values of clip norm $\tau$ in the experiments of \Cref{expts}. Note that \enquote{pctl.} stands for percentile. \enquote{Non-private baseline} is the accuracy of vanilla non-private SGD in the same setting.}
\vspace{0.3 cm}
\begin{subtable}[c]{\textwidth}
\centering
\begin{tabular}{|l|c|c|c|}
\hline
\textbf{CIFAR-10} & \textbf{(a)} $(2,10^{-5})$-DP & \textbf{(b)} $(4,10^{-5})$-DP & \textbf{(c)} $(6,10^{-5})$-DP 
\\
\hline
$\tau = 0.1$ & 83.61 $\pm$ 0.21 \% & 84.84 $\pm$ 0.04 \% & 85.41 $\pm$ 0.05 \%
\\
\hline
$\tau = 1.0$ & 83.87 $\pm$ 0.15 \% & 84.83 $\pm$ 0.04 \% & 85.58 $\pm$ 0.06 \%
\\
\hline
$\tau = 5.0$ & 83.90 $\pm$ 0.20 \% & 84.87 $\pm$ 0.10 \% & 85.45 $\pm$ 0.02 \%
\\
\hline
$\tau = 10.0$ & 83.97 $\pm$ 0.15 \% & 85.15 $\pm$ 0.07 \% & 85.43 $\pm$ 0.06 \%
\\
\hline
$\tau = 19.8 (0^{\text{th}} \text{ pctl.})$ & 83.79 $\pm$ 0.20 \% & 84.75 $\pm$ 0.06 \% & 85.31 $\pm$ 0.08 \%
\\
\hline
$\tau = 31.7 (10^{\text{th}} \text{ pctl.})$ & 83.14 $\pm$ 0.18 \% & 84.67 $\pm$ 0.15 \% & 84.98 $\pm$ 0.10 \%
\\
\hline
$\tau = 33.6 (20^{\text{th}} \text{ pctl.})$ & 82.75 $\pm$ 0.15 \% & 84.50 $\pm$ 0.09 \% & 85.05 $\pm$ 0.03 \%
\\
\hline
$\tau = 36.1 (40^{\text{th}} \text{ pctl.})$ & 82.94 $\pm$ 0.23 \% & 84.51 $\pm$ 0.11 \% & 84.90 $\pm$ 0.07 \%
\\
\hline
$\tau = 40.8 (80^{\text{th}} \text{ pctl.})$ & 82.56 $\pm$ 0.22 \% & 83.98 $\pm$ 0.10 \% & 84.73 $\pm$ 0.08 \%
\\
\hline
$\tau = 55.4 (100^{\text{th}} \text{ pctl.})$ & 81.78 $\pm$ 0.16 \% & 83.21 $\pm$ 0.11 \% & 84.25 $\pm$ 0.06 \%
\\
\hline
Non-private baseline & \multicolumn{3}{|c|}{\texttt{86.78 $\pm$ 0.05} \%}
\\
\hline
\end{tabular}
\end{subtable}
\\
\\
\\
\begin{subtable}[c]{\textwidth}
\centering
\begin{tabular}{|l|c|c|c|}
\hline
\textbf{TweetEval Emoji} & \textbf{(a)} $(2,10^{-5})$-DP & \textbf{(b)} $(4,10^{-5})$-DP & \textbf{(c)} $(6,10^{-5})$-DP
\\
\hline
$\tau = 0.1$ & 26.15 $\pm$ 0.26 \% & 28.13 $\pm$ 0.12 \% & 28.59 $\pm$ 0.21 \%
\\
\hline
$\tau = 1.0$ & 26.21 $\pm$ 0.30 \% & 28.21 $\pm$ 0.10 \% & 28.68 $\pm$ 0.17 \%
\\
\hline
$\tau = 5.0$ & 26.28 $\pm$ 0.33 \% &  28.14 $\pm$ 0.10 \% & 28.77 $\pm$ 0.23 \%
\\
\hline
$\tau = 10.0$ & 26.34 $\pm$ 0.37 \% & 28.00 $\pm$ 0.15 \% & 28.57 $\pm$ 0.18 \%
\\
\hline
$\tau = 13.8(0^{\text{th}} \text{ pctl.})$ & 26.04 $\pm$ 0.29 \% & 27.86 $\pm$ 0.06 \% & 28.40 $\pm$ 0.13 \%
\\
\hline
$\tau = 15.9(10^{\text{th}} \text{ pctl.})$ & 25.92 $\pm$ 0.29 \% & 27.67 $\pm$ 0.10 \% & 28.06 $\pm$ 0.06 \%
\\
\hline
$\tau = 16.2(20^{\text{th}} \text{ pctl.})$ & 25.78 $\pm$ 0.32 \% & 27.63 $\pm$ 0.08 \% & 28.06 $\pm$ 0.05 \%
\\
\hline
$\tau = 16.7(40^{\text{th}} \text{ pctl.})$ & 25.83 $\pm$ 0.30 \% & 27.50 $\pm$ 0.07 \% & 28.01 $\pm$ 0.07 \%
\\
\hline
$\tau = 17.5(80^{\text{th}} \text{ pctl.})$ & 25.70 $\pm$ 0.31 \% & 27.46 $\pm$ 0.07 \% & 27.93 $\pm$ 0.07 \%
\\
\hline
$\tau = 20.8(100^{\text{th}} \text{ pctl.})$ & 25.55 $\pm$ 0.26 \% & 27.06 $\pm$ 0.11 \% & 27.73 $\pm$ 0.05 \%
\\
\hline
Non-private baseline & \multicolumn{3}{|c|}{\texttt{30.10 $\pm$ 0.03} \%}
\\
\hline
\end{tabular}
\end{subtable}
\\
\\
\\
\begin{subtable}[c]{\textwidth}
\centering
\begin{tabular}{|l|c|c|c|}
\hline
\textbf{Emotion} & \textbf{(a)} $(2,10^{-5})$-DP & \textbf{(b)} $(4,10^{-5})$-DP & \textbf{(c)} $(6,10^{-5})$-DP
\\
\hline
$\tau = 0.1$ & 51.95 $\pm$ 0.42 \% & 54.88 $\pm$ 0.30 \% & 56.45 $\pm$ 0.25 \%
\\
\hline
$\tau = 1.0$ & 52.39 $\pm$ 0.61 \% & 55.05 $\pm$ 0.26 \% & 56.10 $\pm$ 0.16 \%
\\
\hline
$\tau = 5.0$ & 52.40 $\pm$ 0.57 \% & 55.25 $\pm$ 0.13 \% & 56.35 $\pm$ 0.20 \%
\\
\hline
$\tau = 10.0$ & 52.24 $\pm$ 0.64 \% & 54.95 $\pm$ 0.22 \% & 55.82 $\pm$ 0.33 \%
\\
\hline
$\tau = 14.6 (0^{\text{th}} \text{ pctl.})$ & 51.86 $\pm$ 0.70 \% & 54.52 $\pm$ 0.29 \% & 55.83 $\pm$ 0.16 \%
\\
\hline
$\tau = 16.4 (10^{\text{th}} \text{ pctl.})$ & 51.59 $\pm$ 0.46 \% & 54.25 $\pm$ 0.29 \% & 55.03 $\pm$ 0.20 \%
\\
\hline
$\tau = 16.8 (20^{\text{th}} \text{ pctl.})$ & 51.03 $\pm$ 0.42 \% & 54.16 $\pm$ 0.34 \% & 55.09 $\pm$ 0.19 \%
\\
\hline
$\tau = 17.2 (40^{\text{th}} \text{ pctl.})$ & 50.84 $\pm$ 0.48 \% & 54.21 $\pm$ 0.35 \% & 55.04 $\pm$ 0.22 \%
\\
\hline
$\tau = 17.8 (80^{\text{th}} \text{ pctl.})$ & 50.94 $\pm$ 0.50 \% & 53.78 $\pm$ 0.34 \% & 54.97 $\pm$ 0.09 \%
\\
\hline
$\tau = 20.2 (100^{\text{th}} \text{ pctl.})$ & 50.34 $\pm$ 0.46 \% & 53.64 $\pm$ 0.22 \% & 54.21 $\pm$ 0.10 \%
\\
\hline
Non-private baseline & \multicolumn{3}{|c|}{\texttt{60.43 $\pm$ 0.22} \%}
\\
\hline
\end{tabular}
\end{subtable}
\label{tab3-nlp-0}
\end{table}

\begin{table}[!htb]
\caption{\textbf{EMNIST and Fashion-MNIST:} Average test accuracy $\pm$ 1 standard deviation in the last 5 epochs for different values of clip norm $\tau$ in the experiments of \Cref{extra-expts}. Note that \enquote{pctl.} stands for percentile. \enquote{Non-private baseline} is the accuracy of vanilla non-private SGD in the same setting.}
\vspace{0.3 cm}
\begin{subtable}[c]{\textwidth}
\centering
\begin{tabular}{|l|c|c|c|}
\hline
\textbf{EMNIST} & \textbf{(a)} $(2,10^{-5})$-DP & \textbf{(b)} $(4,10^{-5})$-DP & \textbf{(c)} $(6,10^{-5})$-DP
\\
\hline
$\tau = 1.0$ & 62.78 $\pm$ 0.98 \% & 65.90 $\pm$ 0.09 \% & 67.02 $\pm$ 0.11 \%
\\
\hline
$\tau = 3.0$ & 63.17 $\pm$ 0.94 \% & 66.16 $\pm$ 0.10 \% & 66.94 $\pm$ 0.10 \%
\\
\hline
$\tau = 5.7 (0^{\text{th}} \text{ pctl.})$ & 62.85 $\pm$ 0.89 \% & 65.55 $\pm$ 0.17 \% & 66.79 $\pm$ 0.04 \%
\\
\hline
$\tau = 11.6 (10^{\text{th}} \text{ pctl.})$ & 61.11 $\pm$ 0.99 \% & 64.77 $\pm$ 0.25 \% & 65.41 $\pm$ 0.11 \%
\\
\hline
$\tau = 12.7 (20^{\text{th}} \text{ pctl.})$ & 61.12 $\pm$ 0.98 \% & 64.30 $\pm$ 0.09 \% & 64.96 $\pm$ 0.10 \%
\\
\hline
$\tau = 14.1 (40^{\text{th}} \text{ pctl.})$ & 60.42 $\pm$ 0.96 \% & 63.88 $\pm$ 0.13 \% & 64.67 $\pm$ 0.14 \%
\\
\hline
$\tau = 16.7 (80^{\text{th}} \text{ pctl.})$ & 59.49 $\pm$ 0.83 \% & 63.07 $\pm$ 0.10 \% & 63.94 $\pm$ 0.08 \%
\\
\hline
$\tau = 26.1 (100^{\text{th}} \text{ pctl.})$ & 56.36 $\pm$ 1.04 \% & 61.26 $\pm$ 0.16 \% & 62.06 $\pm$ 0.08 \%
\\
\hline
Non-private baseline & \multicolumn{3}{|c|}{\texttt{69.37 $\pm$ 0.04} \%}
\\
\hline
\end{tabular}
\end{subtable}
\\
\\
\\
\begin{subtable}[c]{\textwidth}
\centering
\begin{tabular}{|l|c|c|c|}
\hline
\textbf{Fashion-MNIST} & \textbf{(a)} $(2,10^{-5})$-DP & \textbf{(b)} $(4,10^{-5})$-DP & \textbf{(c)} $(6,10^{-5})$-DP
\\
\hline
$\tau = 1.0$ & 82.99 $\pm$ 0.13 \% & 83.86 $\pm$ 0.11 \% & 84.06 $\pm$ 0.02 \%
\\
\hline
$\tau = 3.0 (0^{\text{th}} \text{ pctl.})$ & 82.82 $\pm$ 0.18 \% & 83.85 $\pm$ 0.06 \% & 83.99 $\pm$ 0.11 \%
\\
\hline
$\tau = 10.0 (10^{\text{th}} \text{ pctl.})$ & 82.24 $\pm$ 0.14 \% & 83.43 $\pm$ 0.08 \% & 83.56 $\pm$ 0.06 \%
\\
\hline
$\tau = 12.2 (20^{\text{th}} \text{ pctl.})$ & 82.27 $\pm$ 0.07 \% & 83.36 $\pm$ 0.09 \% & 83.59 $\pm$ 0.09 \%
\\
\hline
$\tau = 15.6 (40^{\text{th}} \text{ pctl.})$ & 82.05 $\pm$ 0.19 \% & 83.20 $\pm$ 0.07 \% & 83.31 $\pm$ 0.08 \%
\\
\hline
$\tau = 22.3 (80^{\text{th}} \text{ pctl.})$ & 81.24 $\pm$ 0.14 \% & 82.43 $\pm$ 0.07 \% & 82.67 $\pm$ 0.10 \%
\\
\hline
$\tau = 32.4 (100^{\text{th}} \text{ pctl.})$ & 79.99 $\pm$ 0.18 \% & 81.69 $\pm$ 0.12 \% & 82.14 $\pm$ 0.10 \%
\\
\hline
Non-private baseline & \multicolumn{3}{|c|}{\texttt{84.44 $\pm$ 0.05} \%}
\\
\hline
\end{tabular}
\end{subtable}
\label{tab3-nlp}
\end{table}

\clearpage

\section{Some Empirical Results in the Non-Convex Case}
\label{sec:noncvx-exp}
Here we show some empirical results on a nonconvex neural network (NN) problem. Specifically, we consider a two-layer feedforward NN having one hidden layer with tanh activation. We use tanh instead of ReLU activation due to two reasons: (i) \citet{papernot2020tempered} show that tanh performs better than ReLU in private training of NNs (which we also observed), and (ii) we expect Lipschitz constants to be smaller with tanh than ReLU. 
We run our experiments on CIFAR-100 and EMNIST; we use pre-trained features for the former, while for the latter, we use the raw images. Specifically, for CIFAR-100, we use 512-dimensional features obtained
from the pre-softmax layer of a pre-trained ResNet-18 model on ImageNet (same as in \Cref{expts}). For EMNIST, we use the flattened images as features (just as mentioned in \Cref{extra-expts}).
For both datasets, we set the dimension of the hidden layer of the NN to be 256. 
Computing the per-sample Lipschitz constants is much harder here so we just test several values of the clip norm $\tau$, viz., $\{1,3,6,12,18,24,30,36\}$, and show the performance trend as a function of $\tau$. All other experimental details are the same as in \Cref{expts}. 

In \Cref{fig:supp-noncvx}, we plot the {best test accuracy} obtained for different values of $\tau$ (by tuning over $\eta$) averaged over the last 5 epochs and across 3 independent runs. The figure caption discusses the results. The exact values are tabulated in \Cref{tab-noncvx}.

So empirically, smaller clip norms perform better in two-layer nonconvex NNs, similar to the convex case.

A single NVIDIA TITAN Xp GPU was used for all the experiments in this paper.

\begin{figure*}[!htb]
\centering 
\subfloat[EMNIST]{
	\includegraphics[width=0.75\textwidth]{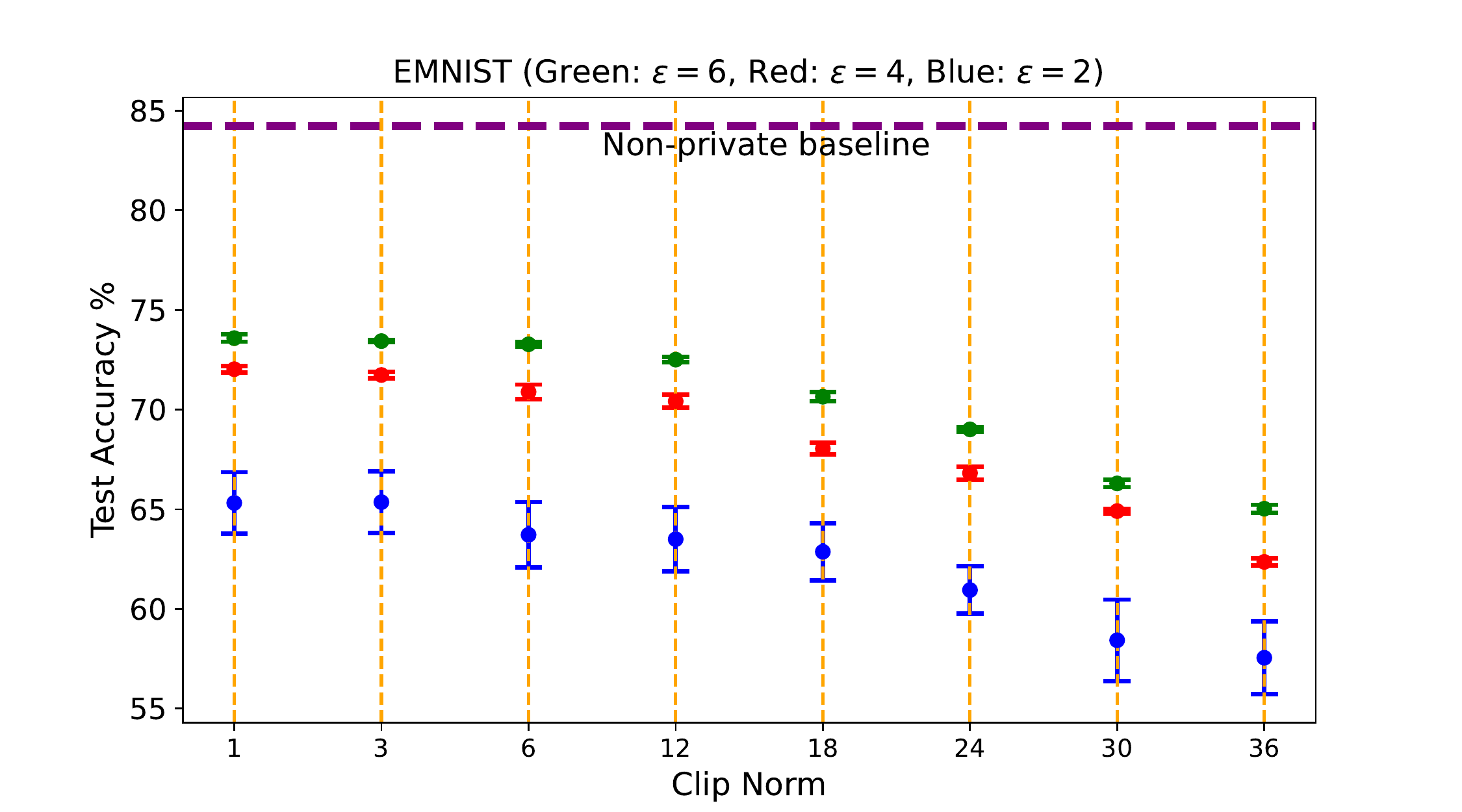}
}
\\
\subfloat[CIFAR-100]{
	\includegraphics[width=0.75\textwidth]{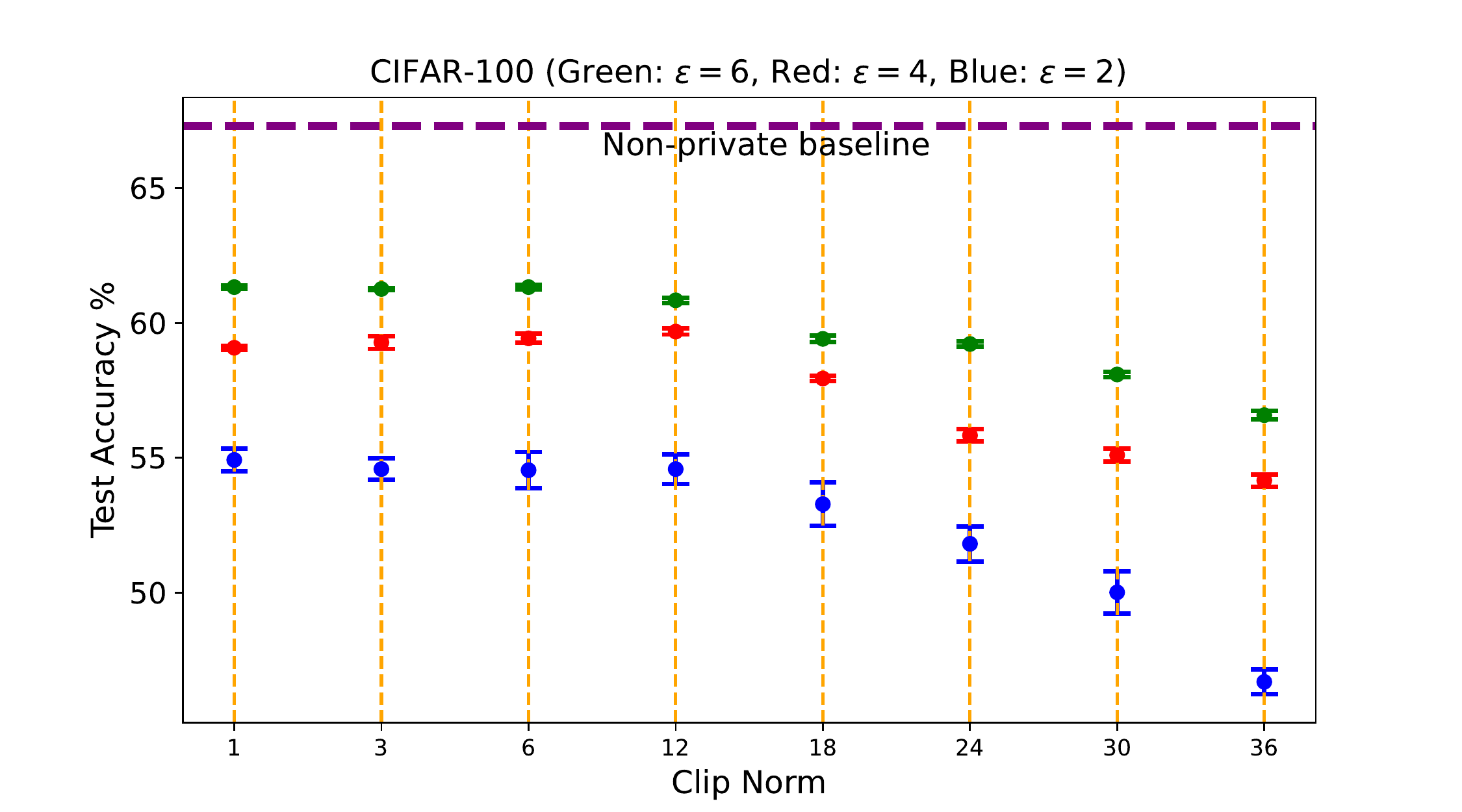}
	} 
\caption{{\textbf{EMNIST and CIFAR-100 with two-layer NN:} Average test accuracy (depicted by the blobs) $\pm$ 1 standard deviation (depicted by the bars above and below the blobs) in the last 5 epochs for different values of clip norm $\tau$.
The general trend above is that the accuracy drops as the clip norm increases; this is similar to what we saw in the experiments on convex problems, and consistent with the main message of \Cref{thm-small-clip} (even though it is for the convex case), viz., smaller clip norms should perform better as they attain a lower risk bound.
}}
\label{fig:supp-noncvx}
\end{figure*}

\begin{table}[!htb]
\caption{\textbf{EMNIST and CIFAR-100 with two-layer NN:} Average test accuracy $\pm$ 1 standard deviation in the last 5 epochs for different values of clip norm $\tau$ in the experiments above. \enquote{Non-private baseline} is the accuracy of vanilla non-private SGD in the same setting.}
\vspace{0.3 cm}
\begin{subtable}[c]{\textwidth}
\centering
\begin{tabular}{|l|c|c|c|}
\hline
\textbf{EMNIST} & \textbf{(a)} $(2,10^{-5})$-DP & \textbf{(b)} $(4,10^{-5})$-DP & \textbf{(c)} $(6,10^{-5})$-DP
\\
\hline
$\tau = 1$ & 65.32 $\pm$ 1.54 \% & 72.03 $\pm$ 0.17 \% & 73.59 $\pm$ 0.18 \%
\\
\hline
$\tau = 3$ & 65.36 $\pm$ 1.54 \% & 71.74 $\pm$ 0.16 \% & 73.44 $\pm$ 0.06 \%
\\
\hline
$\tau = 6$ & 63.72 $\pm$ 1.64 \% & 70.89 $\pm$ 0.37 \% & 73.28 $\pm$ 0.12 \%
\\
\hline
$\tau = 12$ & 63.50 $\pm$ 1.60 \% & 70.42 $\pm$ 0.33 \% & 72.51 $\pm$ 0.14 \%
\\
\hline
$\tau = 18$ & 62.87 $\pm$ 1.44 \% & 68.04 $\pm$ 0.29 \% & 70.65 $\pm$ 0.23 \%
\\
\hline
$\tau = 24$ & 60.95 $\pm$ 1.19 \% & 66.82 $\pm$ 0.33 \% & 69.01 $\pm$ 0.11 \%
\\
\hline
$\tau = 30$ & 58.43 $\pm$ 2.04 \% & 64.91 $\pm$ 0.12 \% & 66.30 $\pm$ 0.19 \%
\\
\hline
$\tau = 36$ & 57.55 $\pm$ 1.83 \% & 62.36 $\pm$ 0.19 \% & 65.03 $\pm$ 0.20 \%
\\
\hline
Non-private baseline & \multicolumn{3}{|c|}{\texttt{84.24 $\pm$ 0.05} \%}
\\
\hline
\end{tabular}
\end{subtable}
\\
\\
\\
\begin{subtable}[c]{\textwidth}
\centering
\begin{tabular}{|l|c|c|c|}
\hline
\textbf{CIFAR-100} & \textbf{(a)} $(2,10^{-5})$-DP & \textbf{(b)} $(4,10^{-5})$-DP & \textbf{(c)} $(6,10^{-5})$-DP
\\
\hline
$\tau = 1$ & 54.92 $\pm$ 0.42 \% & 59.08 $\pm$ 0.08 \% & 61.33 $\pm$ 0.06 \%
\\
\hline
$\tau = 3$ & 54.58 $\pm$ 0.40 \% & 59.28 $\pm$ 0.24 \% & 61.26 $\pm$ 0.04 \% 
\\
\hline
$\tau = 6$ & 54.54 $\pm$ 0.67 \% & 59.43 $\pm$ 0.17 \% & 61.33 $\pm$ 0.08 \%
\\
\hline
$\tau = 12$ & 54.58 $\pm$ 0.55 \% & 59.48 $\pm$ 0.11 \% & 60.84 $\pm$ 0.10 \%
\\
\hline
$\tau = 18$ & 53.28 $\pm$ 0.81 \% & 57.94 $\pm$ 0.09 \% & 59.41 $\pm$ 0.12 \%
\\
\hline
$\tau = 24$ & 51.81 $\pm$ 0.65 \% & 55.83 $\pm$ 0.23 \% & 59.22 $\pm$ 0.09 \%
\\
\hline
$\tau = 30$ & 50.01 $\pm$ 0.78 \% & 55.10 $\pm$ 0.24 \% & 58.09 $\pm$ 0.10 \%
\\
\hline
$\tau = 36$ & 46.69 $\pm$ 0.46 \% & 54.16 $\pm$ 0.23 \% & 56.58 $\pm$ 0.16 \%
\\
\hline
Non-private baseline & \multicolumn{3}{|c|}{\texttt{67.31 $\pm$ 0.04} \%}
\\
\hline
\end{tabular}
\end{subtable}
\label{tab-noncvx}
\end{table}

\clearpage

\section{Logistic Regression Satisfies Assumption~\ref{asmp-lip-gen}}
\label{log-reg-asmp}
Consider doing logistic regression for multi-class classification with the cross-entropy loss, where $m$ is the number of classes. Suppose $\bm{x} \sim \mathcal{F}$ (with a `1' appended to account for the bias term) is the feature vector and $y \in [m]$ is the corresponding class number.
Let the model parameter $\bm{w}$ be split as $\bm{w} = [\bm{w}_1,\ldots,\bm{w}_m]$, where each $\{\bm{w}_j\}_{j=1}^m \in \mathbb{R}^d$, $d$ being the dimension of $\bm{x}$; so, $\bm{w}_j$ denotes the parameter vector corresponding to class $j$. Then, our predicted probability of $\bm{x}$ belonging to class $j$ with the softmax predictor is:
\begin{equation*}
    p_j = \frac{\exp(\bm{w}_j^T \bm{x})}{\sum_{k=1}^m \exp(\bm{w}_k^T \bm{x})}.
\end{equation*}
We use the standard cross-entropy loss for logistic regression which gives us:
\begin{equation}
    \ell(\bm{w}, \bm{x}, y) = -\log({p}_{y}).
\end{equation}
Now, with some differentiation, it can be checked that:
\begin{equation}
    \|\nabla \ell_{\bm{w}}(\bm{w}, \bm{x}, y)\| = \Bigg(\sqrt{\sum_{j \neq y} p_j^2 + (1 - p_y)^2}\Bigg) \|\bm{x}\| \leq \sqrt{2} \|\bm{x}\|.
\end{equation}
Thus, logistic regression satisfies \Cref{asmp-lip-gen} with $G(\bm{x} , y) = \sqrt{2} \|\bm{x}\|$ in any parameter domain $\mathcal{W}$.

\section{Example for Assumption~\ref{asmp1-may4}}
\label{new-asmp-eg}
Suppose $f_i(\bm{w}) = \|\bm{w} - \bm{w}_i^{\ast}\|$. Then, $f(\bm{w}) = \frac{1}{n}\sum_{i=1}^n \|\bm{w} - \bm{w}_i^{\ast}\|$ and $\bm{w}^{\ast}=\text{arg min}_{\bm{w} \in \mathbb{R}^d} f(\bm{w})$. 

Let $\overline{\bm{w}}^{\ast}=\frac{1}{n} \sum_{i=1}^n \bm{w}_i^{\ast}$. For $\|\bm{w} - \bm{w}^{\ast}\| \geq D := \max(2\|\overline{\bm{w}}^{\ast} - \bm{w}^{\ast}\|, \frac{4}{n}\sum_{i=1}^n \|\overline{\bm{w}}^{\ast} - \bm{w}_i^{\ast}\|)$, it can be shown that $f(\bm{w}) - f(\bm{w}^{\ast}) \geq \frac{1}{4}\|\bm{w} - \bm{w}^{\ast}\|$. 
Noting that $G = \mathcal{O}(1)$ here (as the sub-gradient of $f_i$ is bounded by 1 in magnitude), \Cref{asmp1-may4} easily holds here for small $\varphi$.
\begin{proof}
Using the triangle inequality, we have for any $\bm{w}$ satisfying $\|\bm{w} - \bm{w}^{\ast}\| \geq D$:
\begin{flalign}
   f(\bm{w}) & = \frac{1}{n}\sum_{i=1}^n \|\bm{w} - \bm{w}_i^{\ast}\| 
   \\
   & \geq \Big\|\bm{w} -  \frac{1}{n}\sum_{i=1}^n \bm{w}_i^{\ast}\Big\| 
   \\
   & = \|\bm{w} -  \overline{\bm{w}}^{\ast}\| 
   \\
   & = \|\bm{w} - \bm{w}^{\ast} - (\overline{\bm{w}}^{\ast} - \bm{w}^{\ast})\| 
   \\
   & \geq \Big|\|\bm{w} - \bm{w}^{\ast}\| - \|\overline{\bm{w}}^{\ast} - \bm{w}^{\ast}\|\Big|
   \\
   & = \Bigg|\frac{\|\bm{w} - \bm{w}^{\ast}\|}{2} + \frac{\|\bm{w} - \bm{w}^{\ast}\|}{2} - \|\overline{\bm{w}}^{\ast} - \bm{w}^{\ast}\|\Bigg|
   \\
   \label{eq:15-asmp-eg}
   & \geq \frac{\|\bm{w} - \bm{w}^{\ast}\|}{2}.
\end{flalign}
\Cref{eq:15-asmp-eg} follows because $\frac{\|\bm{w} - \bm{w}^{\ast}\|}{2} \geq \frac{D}{2} \geq \|\overline{\bm{w}}^{\ast} - \bm{w}^{\ast}\|$ (from the definition of $D$). Next, since $\bm{w}^{\ast}$ is the minimizer of $f$ and using the definition of $D$, we have:
\begin{equation}
\label{eq:16-asmp-eg}
f(\bm{w}^{\ast}) \leq f(\overline{\bm{w}}^{\ast}) = \frac{1}{n}\sum_{i=1}^n \|\overline{\bm{w}}^{\ast} - \bm{w}_i^{\ast}\| \leq \frac{D}{4} \leq \frac{\|\bm{w} - \bm{w}^{\ast}\|}{4}.
\end{equation}

Finally, subtracting \cref{eq:16-asmp-eg} from \cref{eq:15-asmp-eg} gives us the desired result.
\end{proof}

\section{Some Useful Lemmas}
\label{app-f}
\begin{lemma}[\textbf{Clipping Bias}]
\label{thm1-oct25}
Suppose $\bm{v}(\bm{\zeta})$ (where $\bm{\zeta}$ denotes the source of randomness) is an unbiased estimator of $\bm{v}$, i.e. $\mathbb{E}_{\bm{\zeta}}[\bm{v}(\bm{\zeta})] = \bm{v}$.
Let $b(\tau)$ denote the clipping bias of $\text{clip}(\bm{v}(\bm{\zeta}), \tau)$, i.e.
\begin{equation*}
    b(\tau) = \Big\|\bm{{v}} - \mathbb{E}_{\bm{\zeta}}\Big[\textup{clip}(\bm{v}(\bm{\zeta}), \tau)\Big]\Big\|.
\end{equation*}
Then for any $p > 1$,
\begin{flalign*}
    b(\tau)
    \leq \Big(\mathbb{E}[\|\bm{v}(\bm{\zeta})\|^p]\Big)^{\frac{1}{p}} \Big(\mathbb{P}(\|\bm{v}(\bm{\zeta})\| \geq \tau)\Big)^{1 - \frac{1}{p}} - \tau \mathbb{P}(\|\bm{v}(\bm{\zeta})\| \geq \tau).
\end{flalign*}
\end{lemma}
\begin{proof}
We shall omit the subscript $\bm{\zeta}$ in expectations henceforth, and it should be inferred from context.
We can bound the clipping bias $b(\tau)$ with a clip norm $\tau$ as:
\begin{flalign}
    b(\tau) &= 
    \Big\|\bm{{v}} - \mathbb{E}\Big[\text{clip}(\bm{v}(\bm{\zeta}), \tau)\Big]\Big\| 
    \\
    & =
    \Big\|\bm{{v}} - \mathbb{E}\Big[\bm{v}(\bm{\zeta}) \min \Big(1, \frac{\tau}{\|\bm{v}(\bm{\zeta})\|} \Big)\Big]\Big\| 
    \\
    \label{eq:jan18-1}
    &= \Big\|\mathbb{E}\Big[\bm{v}(\bm{\zeta}) \Big(1 - \frac{\tau}{\|\bm{v}(\bm{\zeta})\|}\Big) \mathbbm{1}(\|\bm{v}(\bm{\zeta})\| \geq \tau)\Big]\Big\|
    \\
    \label{eq:oct25-1}
    & \leq
    \mathbb{E}\Big[\|\bm{v}(\bm{\zeta})\| {\Big(1 - \frac{\tau}{\|\bm{v}(\bm{\zeta})\|}\Big)} \mathbbm{1}(\|\bm{v}(\bm{\zeta})\| \geq \tau)\Big]
    \\
    \label{eq:oct25-2}
    & = \mathbb{E}\Big[\|\bm{v}(\bm{\zeta})\| \mathbbm{1}(\|\bm{v}(\bm{\zeta})\| \geq \tau)\Big] - \tau \mathbb{E}[\mathbbm{1}(\|\bm{v}(\bm{\zeta})\| \geq \tau)]
    \\
    \label{eq:oct25-3}
    & \leq \Big(\mathbb{E}[\|\bm{v}(\bm{\zeta})\|^p]\Big)^{\frac{1}{p}} \Big(\mathbb{E}\Big[\Big(\mathbbm{1}(\|\bm{v}(\bm{\zeta})\| \geq \tau)\Big)^q\Big]\Big)^{\frac{1}{q}} - \tau \mathbb{P}(\|\bm{v}(\bm{\zeta})\| \geq \tau),
\end{flalign}
for $p,q \in (1,\infty)$ such that $\frac{1}{p} + \frac{1}{q} = 1$; this follows from H\"{o}lder's inequality. Now
\begin{flalign}
    \label{eq:jan18-2}
    \mathbb{E}\Big[\Big(\mathbbm{1}(\|\bm{v}(\bm{\zeta})\| \geq \tau)\Big)^q\Big] = \mathbb{E}[\mathbbm{1}(\|\bm{v}(\bm{\zeta})\| \geq \tau)] = \mathbb{P}(\|\bm{v}(\bm{\zeta})\| \geq \tau).
\end{flalign}
Plugging this back in \cref{eq:oct25-3} and substituting $\frac{1}{q} = 1 - \frac{1}{p}$, we get the desired result for $b(\tau)$.
\end{proof}

\begin{corollary}[\textbf{Clipping Bias}]
\label{cor1-dec22}
In the setting of \Cref{thm1-oct25}, we have the following simpler upper bound for any $p > 1$:
\begin{equation*}
        b(\tau) \leq \frac{\mathbb{E}[\|\bm{v}(\bm{\zeta})\|^p]}{\tau^{p-1}}.
    \end{equation*}
\end{corollary}
\begin{proof}
From \Cref{thm1-oct25}, we have that:
\begin{flalign}
    \label{dec22-eq1}
    b(\tau) \leq \Big(\mathbb{E}[\|\bm{v}(\bm{\zeta})\|^p]\Big)^{\frac{1}{p}} \Big(\mathbb{P}(\|\bm{v}(\bm{\zeta})\| \geq \tau)\Big)^{1 - \frac{1}{p}},
\end{flalign}
for any $p > 1$. Using Markov's inequality, we have:
\begin{equation}
    \mathbb{P}(\|\bm{v}(\bm{\zeta})\| \geq \tau) \leq \frac{\mathbb{E}[\|\bm{v}(\bm{\zeta})\|^p]}{\tau^p}.
\end{equation}
Plugging this in \cref{dec22-eq1}, we get the desired result.
\end{proof}
It is worth mentioning here that a result similar to \Cref{cor1-dec22} has been established in Lemma 10 of \citet{zhang2020adaptive}.

\section{Proof of Theorem~\ref{thm-dp}}
\label{pf-dp}
Using the result of \citet{abadi2016deep}, we know that any $\bm{w}_t$, where $t \in \{0,\ldots,T-1\}$, will be $(\varepsilon,\delta)$-DP if we set $\sigma_n^2 = \nu \frac{T \log(\frac{1}{\delta})}{n^2 \varepsilon^2} \tau^2$ for some absolute constant $\nu$. Thus, $\bm{w}_{\widehat{t}}$ (where $\widehat{t}$ is chosen uniformly at random from $\{0,\ldots,T-1\}$ as defined in \Cref{alg:1}) will also be $(\varepsilon, \delta)$-DP.

\section{Full Version and Proof of Theorem~\ref{thm-small-clip}}
\label{small-clip-sec}
\begin{theorem}[\textbf{Convex Case}]
\label{thm-small-clip-norm-full}
Suppose each $f_i$ is convex and $\mathcal{W}$ is a convex set (which can be $\mathbb{R}^d$). In \Cref{alg:1}, for all $t < T$, set $\eta_t = \eta = \frac{C}{T \tau} \Big({\frac{1}{T} + \varphi^2}\Big)^{-1/2}$ for clip norm $\tau$, where $C > 0$ is a parameter of our choice. Recall $\bm{w}^{*} \in \textup{argmin}_{\bm{w} \in \mathcal{W}} f(\bm{w})$ and $\widehat{t} \sim \textup{Unif}[0, T-1]$. Then, DP-SGD (\Cref{alg:1}) has the following convergence guarantee:
\begin{equation*}
    \frac{1}{n}\sum_{i \in [n]} \mathbb{E}\Big[\min\Big(1,\frac{\tau}{\|\nabla f_i(\bm{w}_{\widehat{t}})\|}\Big) (f_i(\bm{w}_{\widehat{t}}) - f_i(\bm{w}^{*}))\Big] \leq \bigg(\frac{\|\bm{w}_0 - \bm{w}^{*}\|^2}{2 C} + C\bigg) \tau \sqrt{\frac{1}{T} + \varphi^2}.
\end{equation*}
Now suppose Assumption \ref{asmp-lip} holds. Then, DP-SGD has the following upper bound on the optimization risk as a function of the clip norm $\tau \in (0, G_n]$:
\begin{equation*}
    \textup{OR}(T) \leq \frac{1}{\alpha(\tau)}\Bigg(\bigg(\frac{\|\bm{w}_0 - \bm{w}^{*}\|^2}{2 C} + C\bigg) G_n \sqrt{\frac{1}{T} + \varphi^2}\Bigg) + \Bigg(\frac{G_n}{\tau \alpha(\tau)} - 1\Bigg) {\Delta(\bm{w}^{*})},
\end{equation*}
where $\Delta(\bm{w}^{*}) \geq 0$ and $\alpha(\tau) \geq 1$ are as defined in Definitions \ref{het-def} and \ref{may2-asmp-1}, respectively.
\end{theorem}
\noindent \Cref{thm-small-clip} can be obtained from the above theorem by plugging in $T = \frac{1}{3 \varphi^2}$. The proof of \Cref{thm-small-clip-norm-full} is below.

\subsection*{Proof:}
\begin{proof}
Suppose we use a constant clip norm $\tau$ and constant learning rate $\eta$.
For any $\bm{w}^{*} \in \text{argmin}_{\bm{w} \in \mathcal{W}} f(\bm{w})$, $\|\bm{w}_{t+1} - \bm{w}^{*}\| \leq \|\bm{z}_{t+1} - \bm{w}^{*}\|$ as $\bm{w}_{t+1}$ is the projection of $\bm{z}_{t+1}$ onto the convex set $\mathcal{W}$. Thus:
\small
\begin{flalign}
    \mathbb{E}[\|\bm{w}_{t+1} - \bm{w}^{*}\|^2]
    & \leq \mathbb{E}[\|\bm{z}_{t+1} - \bm{w}^{*}\|^2]
    \\
    & = \mathbb{E}[\|\bm{w}_{t} - \bm{w}^{*}\|^2] - 2 \eta \mathbb{E}[\langle \bm{g}_t, \bm{w}_{t} - \bm{w}^{*}\rangle] + \eta^2 \mathbb{E}[\|\bm{g}_t\|^2]
    \\
    & = \mathbb{E}[\|\bm{w}_{t} - \bm{w}^{*}\|^2] - 2 \eta \mathbb{E}\Big[\Big\langle \frac{1}{b}\sum_{i \in \mathcal{S}_t} \text{clip}(\nabla f_i(\bm{w}_t), \tau), \bm{w}_{t} - \bm{w}^{*}\Big\rangle\Big] + \eta^2 \mathbb{E}[\|\bm{g}_t\|^2]
    \\
    \label{eq:apr26-1}
    & = \mathbb{E}[\|\bm{w}_{t} - \bm{w}^{*}\|^2] - \frac{2 \eta}{n}\sum_{i \in [n]} \mathbb{E}\Big[\min\Big(1,\frac{\tau}{\|\nabla f_i(\bm{w}_t)\|}\Big)\langle \nabla f_i(\bm{w}_t), \bm{w}_{t} - \bm{w}^{*}\rangle\Big] + \eta^2 \mathbb{E}[\|\bm{g}_t\|^2]
    \\
    \label{eq:apr26-2}
    & \leq \mathbb{E}[\|\bm{w}_{t} - \bm{w}^{*}\|^2] - \frac{2 \eta}{n}\sum_{i \in [n]} \mathbb{E}\Big[\min\Big(1,\frac{\tau}{\|\nabla f_i(\bm{w}_t)\|}\Big) (f_i(\bm{w}_t) - f_i(\bm{w}^{*}))\Big] + \eta^2 \mathbb{E}[\|\bm{g}_t\|^2].
\end{flalign}
\normalsize
\Cref{eq:apr26-2} follows from the convexity of $f_i$. Next, rearranging the above a bit, followed by summing for $t = 0$ through to $T-1$, and then dividing by $2 \eta T$ throughout, we get:
\small
\begin{flalign}
    \frac{1}{T}\sum_{t=0}^{T-1}\Big\{\frac{1}{n}\sum_{i \in [n]} \mathbb{E}\Big[\min\Big(1,\frac{\tau}{\|\nabla f_i(\bm{w}_t)\|}\Big) (f_i(\bm{w}_t) - f_i(\bm{w}^{*}))\Big]\Big\} & \leq \frac{\|\bm{w}_0 - \bm{w}^{*}\|^2 - \mathbb{E}[\|\bm{w}_{T} - \bm{w}^{*}\|^2]}{2 \eta T} + \frac{\eta}{2 T}\sum_{t=0}^{T-1} \mathbb{E}[\|\bm{g}_t\|^2] 
    \\
    \label{eq:apr26-3}
    & \leq \frac{\|\bm{w}_0 - \bm{w}^{*}\|^2}{2\eta T} + \frac{\eta}{2 T}\sum_{t=0}^{T-1} \mathbb{E}[\|\bm{g}_t\|^2]. 
\end{flalign}
\normalsize
Henceforth, we shall denote $\|\bm{w}_0 - \bm{w}^{*}\|$ by $D_0$ for brevity.
\\
\\
Next:
\begin{flalign}
    \mathbb{E}[\|\bm{g}_t\|^2] & =
    \mathbb{E}\Big[\Big\|\frac{1}{b}\sum_{i \in \mathcal{S}_t} \text{clip}(\nabla f_i(\bm{w}_t), \tau) + \bm{\zeta_t}\Big\|^2\Big]
    \\
    \label{eq:aug1-2022-1}
    & = \underbrace{\mathbb{E}\Big[\Big\|\frac{1}{b}\sum_{i \in \mathcal{S}_t} \text{clip}(\nabla f_i(\bm{w}_t), \tau)\Big\|^2\Big]}_\text{(I)} + d\sigma_n^2.
\end{flalign}
Let $z_{t,i} = 1$ if sample $i \in \mathcal{S}_t$ and 0 otherwise; note that $\mathbb{P}(z_{t,i} = 1) = \frac{b}{n}$. With this, we can rewrite (I) as:
\begin{flalign}
    \text{(I)} &= \mathbb{E}\Big[\Big\|\frac{1}{b}\sum_{i \in [n]} \text{clip}(\nabla f_i(\bm{w}_t), \tau)z_{t,i}\Big\|^2\Big]
    \\
    &= \frac{1}{b^2} \sum_{i \in [n]}  \|\text{clip}(\nabla f_i(\bm{w}_t), \tau)\|^2 \text{E}[z_{t,i}^2] + \frac{1}{b^2} \sum_{i \neq j}  \langle \text{clip}(\nabla f_i(\bm{w}_t), \tau), \text{clip}(\nabla f_j(\bm{w}_t), \tau) \rangle \text{E}[z_{t,i} z_{t,j}].
\end{flalign}
Note that $\text{E}[z_{t,i}^2] = \frac{b}{n}$ $\forall$ $i \in [n]$, $ \text{E}[z_{t,i} z_{t,j}] = \frac{b^2}{n^2}$ $\forall$ $i \neq j$, $\|\text{clip}(\nabla f_i(\bm{w}_t), \tau)\|^2 \leq \tau^2$ $\forall$ $i \in [n]$ and $\langle \text{clip}(\nabla f_i(\bm{w}_t), \tau), \text{clip}(\nabla f_j(\bm{w}_t), \tau) \rangle \leq \|\text{clip}(\nabla f_i(\bm{w}_t), \tau)\| \|\text{clip}(\nabla f_j(\bm{w}_t), \tau)\| \leq \tau^2$ $\forall$ $i \neq j$. Using all this above, we get:
\begin{equation}
    \text{(I)} \leq \tau^2 \Big(1 + \frac{1}{b} - \frac{1}{n}\Big).
\end{equation}
Plugging the above as well as the value of $\sigma_n^2$ back in \cref{eq:aug1-2022-1}, we get:
\begin{flalign}
    \mathbb{E}[\|\bm{g}_t\|^2] & \leq \tau^2 \Big(1 + \frac{1}{b} - \frac{1}{n} + \frac{\nu d T \log(1/\delta)}{n^2 \varepsilon^2}\Big) 
    \\
    \label{eq:aug1-2022-2}
    & \leq \tau^2 \Big(2 + \frac{\nu d T \log(1/\delta)}{n^2 \varepsilon^2}\Big)
    \\
    \label{eq:dec10-3}
    & \leq 2 \tau^2 \Big(1 + \frac{\nu d T \log(1/\delta)}{n^2 \varepsilon^2}\Big).
\end{flalign}
\Cref{eq:aug1-2022-2} follows because $b \geq 1$. 

Plugging \cref{eq:dec10-3} in \cref{eq:apr26-3}, we get:
\begin{equation}
    \frac{1}{T}\sum_{t=0}^{T-1}\Big\{\frac{1}{n}\sum_{i \in [n]} \mathbb{E}\Big[\min\Big(1,\frac{\tau}{\|\nabla f_i(\bm{w}_t)\|}\Big) (f_i(\bm{w}_t) - f_i(\bm{w}^{*}))\Big]\Big\} \leq \frac{D_0^2}{2\eta T} + {\eta T \tau^2}\Big(\frac{1}{T} + \frac{\nu d \log(\frac{1}{\delta})}{n^2 \varepsilon^2} \Big).
\end{equation}
Plugging in $\eta = \frac{C}{T \tau \sqrt{\frac{1}{T} + \frac{\nu d \log({1}/{\delta})}{n^2 \varepsilon^2}}}$
in the above equation, where $C >0$ is a constant of our choice, we get:
\begin{equation}
    \label{eq:apr26-5}
    \frac{1}{T}\sum_{t=0}^{T-1}\Big\{\frac{1}{n}\sum_{i \in [n]} \mathbb{E}\Big[\min\Big(1,\frac{\tau}{\|\nabla f_i(\bm{w}_t)\|}\Big) (f_i(\bm{w}_t) - f_i(\bm{w}^{*}))\Big]\Big\} \leq \Big(\frac{D_0^2}{2 C} + C\Big) \tau \sqrt{\frac{1}{T} + \frac{\nu d \log({1}/{\delta})}{n^2 \varepsilon^2}}.
\end{equation}
Recalling that $\widehat{t} \sim \text{Unif}[0, T-1]$, we can rewrite the above as:
\begin{equation}
    \label{eq:apr26-5-2}
    \frac{1}{n}\sum_{i \in [n]} \mathbb{E}\Big[\min\Big(1,\frac{\tau}{\|\nabla f_i(\bm{w}_{\widehat{t}})\|}\Big) (f_i(\bm{w}_{\widehat{t}}) - f_i(\bm{w}^{*}))\Big] \leq \Big(\frac{D_0^2}{2 C} + C\Big) \tau \sqrt{\frac{1}{T} + \frac{\nu d \log({1}/{\delta})}{n^2 \varepsilon^2}}.
\end{equation}
Next, \Cref{eq:apr26-5-2} can be further rewritten as:
\begin{multline}
    \label{eq:jan23-44}
    \frac{1}{n}\sum_{i \in [n]} \mathbb{E}\Big[\min\Big(1,\frac{\tau}{\|\nabla f_i(\bm{w}_{\widehat{t}})\|}\Big) (f_i(\bm{w}_{\widehat{t}}) - f_i^{*})\Big] \leq \Big(\frac{D_0^2}{2 C} + C\Big) \tau \sqrt{\frac{1}{T} + \frac{\nu d \log({1}/{\delta})}{n^2 \varepsilon^2}} \\ + \frac{1}{n}\sum_{i \in [n]} \mathbb{E}\Big[\min\Big(1,\frac{\tau}{\|\nabla f_i(\bm{w}_{\widehat{t}})\|}\Big)\Big] (f_i(\bm{w}^{*}) - f_i^{*}),
\end{multline}
where $f_i^{*} = \min_{\bm{w} \in \mathcal{W}} f_i(\bm{w})$. 

From \Cref{asmp-lip}, $\|\nabla f_i(\bm{w}_{\widehat{t}})\| \leq G_i$; thus, $\min\Big(1,\frac{\tau}{\|\nabla f_i(\bm{w}_{\widehat{t}})\|}\Big) \geq \min\big(1,\frac{\tau}{G_i}\big)$. Using this and the fact that $f_i(\bm{w}_{\widehat{t}}) - f_i^{*} \geq 0$, we get:
\[\min\Big(1,\frac{\tau}{\|\nabla f_i(\bm{w}_{\widehat{t}})\|}\Big) (f_i(\bm{w}_{\widehat{t}}) - f_i^{*}) \geq \min\Big(1,\frac{\tau}{G_i}\Big) (f_i(\bm{w}_{\widehat{t}}) - f_i^{*}).\]
Using this in \cref{eq:jan23-44} together with the fact that $\min\big(1,\frac{\tau}{\|\nabla f_i(\bm{w}_{\widehat{t}})\|}\big) \leq 1$ and then dividing by $\tau$ throughout, we get:
\begin{equation}
    \label{eq:apr26-6}
    \frac{1}{n}\sum_{i \in [n]} \min\Big(\frac{1}{\tau},\frac{1}{G_i}\Big) \mathbb{E}[f_i(\bm{w}_{\widehat{t}}) - f_i^{*}] \leq \Big(\frac{D_0^2}{2 C} + C\Big) \sqrt{\frac{1}{T} + \frac{\nu d \log({1}/{\delta})}{n^2 \varepsilon^2}} 
    + \frac{1}{n}\sum_{i \in [n]} \frac{(f_i(\bm{w}^{*}) - f_i^{*})}{\tau}.
\end{equation}
Next, from the definition of $\alpha(\tau)$ in \Cref{may2-asmp-1}, we have that:
\begin{flalign}
    \frac{1}{n}\sum_{i \in [n]} \min\Big(\frac{1}{\tau},\frac{1}{G_i}\Big) {\mathbb{E}[f_i(\bm{w}_{\widehat{t}}) - f_i^{*}]} & \geq
    \frac{\alpha(\tau)}{n}\sum_{i \in [n]} \frac{\mathbb{E}[f_i(\bm{w}_{\widehat{t}}) - f_i^{*}]}{G_n}
    \\
    \label{eq:may4-4}
    & = \Big(\frac{\alpha(\tau)}{G_n}\Big) \underbrace{\Big({\mathbb{E}[f(\bm{w}_{\widehat{t}})] - f(\bm{w}^{*})}\Big)}_{=\text{OR}(T)} + \Big(\frac{\alpha(\tau)}{G_n}\Big) \Big(\frac{1}{n}\sum_{i \in [n]} (f_i(\bm{w}^{*}) - f_i^{*})\Big).
\end{flalign}
Next, using \cref{eq:may4-4} in \cref{eq:apr26-6} and the definition of $\text{OR}(T)$, we get (after some rearrangement):
\begin{equation}
    \label{eq:may1-1}
    \text{OR}(T) \leq \frac{1}{\alpha(\tau)}\Bigg(\Big(\frac{D_0^2}{2 C} + C\Big) G_n \sqrt{\frac{1}{T} + \frac{\nu d \log({1}/{\delta})}{n^2 \varepsilon^2}}\Bigg) + \Big(\frac{G_n}{\tau \alpha(\tau)} - 1\Big)\underbrace{\Bigg(\frac{1}{n}\sum_{i \in [n]} (f_i(\bm{w}^{*}) - f_i^{*})\Bigg)}_{\Delta(\bm{w}^{*})}.
\end{equation}
Recalling the definition of $\Delta(\bm{w}^{*})$, we get the desired result.
\end{proof}

\section{Result for the Constrained Convex Case under Assumption~\ref{dec23-asmp1}}
\label{non-lip-cvx}
As mentioned in the main paper, \citet{kamath2021improved} consider stochastic convex optimization (SCO) and derive upper and lower bounds of $\mathcal{O}(\varphi^{1 - \frac{1}{k}})$ and $\Omega(\varphi^{1 - \frac{1}{k}})$ for the generalization error (i.e., $\mathbb{E}_{\bm{x} \sim \mathcal{D}}[\ell(\bm{w}_{\text{priv}}, \bm{x}) - \ell(\bm{w}^{**}, \bm{x})]$, where $\bm{w}^{**} := \text{arg min}_{\bm{w} \in \mathbb{R}^d} \mathbb{E}_{\bm{x} \sim \mathcal{D}}[\ell(\bm{w}, \bm{x})]$) in the constrained convex case under an assumption similar to \Cref{dec23-asmp1}\footnote{Specifically, \citet{kamath2021improved} assume \textit{coordinate-wise} bounded \textit{centered} moments.}. We shall now show that the same bounds hold for the training error (i.e., $f(\bm{w}_{\text{priv}}) - f(\bm{w}^{*})$) in empirical risk minimization (which is what we consider in this work) in the constrained convex case under \Cref{dec23-asmp1}.

\begin{theorem}[\textbf{Constrained Convex Case}]
\label{thm1-dec13}
Suppose \Cref{dec23-asmp1} holds, $f$ is convex and $\mathcal{W}$ is a bounded convex set with diameter $D_{\mathcal{W}} < \infty$. Fix some $\gamma \in (0,1)$. In \Cref{alg:1}, set  $\tau = \frac{G}{\gamma^{1/k}}\Big(\frac{1}{T} + \varphi^2\Big)^{-\frac{1}{2k}}$ and $\eta_t = \eta = \frac{D_\mathcal{W}}{T \tau} \Big(\frac{1}{T} + \varphi^2\Big)^{-\frac{1}{2}}$ for all $t < T$.
Then with a probability of at least $(1 - \gamma)$ which is w.r.t. the random dataset $\mathcal{Z}$ that we obtain, DP-SGD (\Cref{alg:1}) has the following guarantee:
\begin{equation*}
    \textup{OR}(T) \leq
    \frac{5 D_\mathcal{W} G}{2 \gamma^{1/k}} \Big(\frac{1}{T} + \varphi^2\Big)^{\frac{1}{2}(1 - \frac{1}{k})}.
\end{equation*}
So if we set $T = \frac{1}{\varphi^2}$ above, we get the following bound for the risk:
\begin{equation*}
    \textup{OR}(T) \leq \frac{5 D_\mathcal{W} G}{2^{\frac{1}{2}(1 + \frac{1}{k})} \gamma^{1/k}} {\varphi^{(1 - \frac{1}{k})}}.
\end{equation*}
\end{theorem}
\begin{remark}[\textbf{Comparison with Lipschitz Case}]
As per the above theorem, the optimization risk is $\mathcal{O}(\varphi^{1 - {\frac{1}{k}}})$ in the bounded $k^{\text{th}}$ moment \textbf{constrained} convex case. In comparison, the risk is $\mathcal{O}(\varphi)$ in the \textbf{Lipschitz} convex case (equivalent to $k = \infty$); see \citet{bassily2014private}. 
\end{remark}

We also have a matching lower bound in this case.

\begin{theorem}[\textbf{Lower Bound for Constrained Convex Case}]
\label{lower_bound_constrained}
Suppose $\varphi < o(1)$ and $\delta < \exp(-\varepsilon^2)$.
There exists a convex loss function $\ell$ and a bounded convex set $\mathcal{W}$, such that for every $(\varepsilon,\delta)$-DP algorithm $\mathcal{A}$ which tries to solve for $\bm{w}^{*} = \text{arg min}_{\bm{w} \in \mathcal{W}} f(\bm{w})$ where $f$ is the average loss for a dataset $\mathcal{Z}$ of $n$ samples drawn from the data distribution $\mathcal{D}$, there exists a choice of $\mathcal{D}$ such that: 
\begin{itemize}
    \item $f$ satisfies Assumption \ref{dec23-asmp1}.
    \item $\mathbb{E}_{\mathcal{Z} \sim \mathcal{D}^n, \mathcal{A}}\Big[f(\bm{w}_{\mathcal{Z}}^{(\mathcal{A})}) - f(\bm{w}^{*})\Big] \geq \Omega\big(\varphi^{1 - \frac{1}{k}}\big)$, where $\bm{w}_{\mathcal{Z}}^{(\mathcal{A})}$ is the output of algorithm $\mathcal{A}$ on the dataset $\mathcal{Z}$.
\end{itemize}
\end{theorem}

\begin{remark}[\textbf{Tightness of \Cref{thm1-dec13}}]
\label{rem-1}
The $\mathcal{O}(\varphi^{1 - \frac{1}{k}})$ bound on the risk in \Cref{thm1-dec13} is tight (for $\delta < \exp(-\varepsilon^2)$). 
\end{remark}

We now prove Theorem~\ref{thm1-dec13}; the proof of \Cref{lower_bound_constrained} is deferred to \Cref{lb_cons_pf}. 

\subsection*{Proof of Theorem~\ref{thm1-dec13}:}
\label{pf-3}
\begin{proof}
First, using \Cref{bias-heavy-tail}, we have that:
\begin{equation}
    \label{eq:dec13-4}
    \Big\|\frac{1}{n}\sum_{i \in [n]}\text{clip}(\nabla f_i(\bm{w}), \tau) - \nabla f(\bm{w})\Big\| \leq 
    \frac{G_{\mathcal{Z}}^k}{\tau^{k-1}},
\end{equation}
where $G_{\mathcal{Z}}^k = \frac{1}{n}\sum_{i=1}^n (G(\bm{x}_i, y_i))^k$.
\\
\\
For any $\bm{w}^{*} \in \text{argmin}_{\bm{w} \in \mathcal{W}} f(\bm{w})$, $\|\bm{w}_{t+1} - \bm{w}^{*}\| \leq \|\bm{z}_{t+1} - \bm{w}^{*}\|$ as $\bm{w}_{t+1} = \Pi_{\mathcal{W}}(\bm{z}_{t+1})$. So:
\begin{flalign}
    \mathbb{E}[\|\bm{w}_{t+1} - \bm{w}^{*}\|^2] 
    & \leq \mathbb{E}[\|\bm{z}_{t+1} - \bm{w}^{*}\|^2]
    \\
    & = \mathbb{E}[\|\bm{w}_{t} - \bm{w}^{*}\|^2] - 2 \eta \mathbb{E}[\langle \bm{g}_t, \bm{w}_{t} - \bm{w}^{*}\rangle] + \eta^2 \mathbb{E}[\|\bm{g}_t\|^2]
    \\
    & = \mathbb{E}[\|\bm{w}_{t} - \bm{w}^{*}\|^2] - 2 \eta \mathbb{E}\Big[\Big\langle \frac{1}{b}\sum_{i \in \mathcal{S}_t} \text{clip}(\nabla f_i(\bm{w}_t), \tau), \bm{w}_{t} - \bm{w}^{*}\Big\rangle\Big] + \eta^2 \mathbb{E}[\|\bm{g}_t\|^2]
    \\
    \label{eq:dec13-200}
    & =
    \mathbb{E}[\|\bm{w}_{t} - \bm{w}^{*}\|^2] - {2 \eta} \mathbb{E}\Big[\Big\langle \frac{1}{n}\sum_{i \in [n]} \text{clip}(\nabla f_i(\bm{w}_t), \tau), \bm{w}_{t} - \bm{w}^{*}\Big\rangle\Big] + \eta^2 \mathbb{E}[\|\bm{g}_t\|^2]
    \\
    \label{eq:dec31-0}
    & \leq \mathbb{E}[\|\bm{w}_{t} - \bm{w}^{*}\|^2] - 2{\eta} \mathbb{E}[\langle \nabla f(\bm{w}_t), \bm{w}_{t} - \bm{w}^{*}\rangle]
    \\
    \nonumber
    & \text{ } + {2 \eta} \mathbb{E}\Big[\Big\|\frac{1}{n}\sum_{i \in [n]}\text{clip}(\nabla f_i(\bm{w}_t), \tau) - \nabla f(\bm{w}_t)\Big\| \|\bm{w}_{t} - \bm{w}^{*}\|\Big]
    +
    \eta^2 \mathbb{E}[\|\bm{g}_t\|^2]
    \\
    \label{eq:dec13-201}
    & \leq \mathbb{E}[\|\bm{w}_{t} - \bm{w}^{*}\|^2] - 2 \eta \mathbb{E}[f(\bm{w}_t) - f(\bm{w}^{*})]
    \\
    \nonumber
    & \text{ } + {2 \eta D_\mathcal{W}} \mathbb{E}\Big[\Big\|\frac{1}{n}\sum_{i \in [n]}\text{clip}(\nabla f_i(\bm{w}_t), \tau) - \nabla f(\bm{w}_t)\Big\|\Big] + \eta^2 \mathbb{E}[\|\bm{g}_t\|^2].
\end{flalign}
\Cref{eq:dec13-201} follows from the convexity of $f$ together with the fact that $\|\bm{w}_t - \bm{w}^{*}\| \leq D_\mathcal{W}$. 
Next, rearranging the above a bit, followed by summing for $t = 0$ through to $T-1$, and then dividing by $2 \eta T$ throughout, we get:
\begin{flalign}
    \label{eq:jan14-100}
    \frac{1}{T}\sum_{t=0}^{T-1}\Big(\mathbb{E}[f(\bm{w}_t)] - f(\bm{w}^{*})\Big) & \leq \frac{D_\mathcal{W}^2}{2\eta T} + \frac{\eta}{2 T}\sum_{t=0}^{T-1} \mathbb{E}[\|\bm{g}_t\|^2] +
    \frac{D_\mathcal{W}}{T}\sum_{t=0}^{T-1} \mathbb{E}\Big[\Big\|\frac{1}{n}\sum_{i \in [n]}\text{clip}(\nabla f_i(\bm{w}_t), \tau) - \nabla f(\bm{w}_t)\Big\|\Big]
    \\
    \label{eq:dec13-5}
    & \leq
    \frac{D_\mathcal{W}^2}{2\eta T} + {\eta T \tau^2} \Bigg(\frac{\nu d \log({1}/{\delta})}{n^2 \varepsilon^2} + \frac{1}{T}\Bigg)
    + \frac{D_\mathcal{W} G_{\mathcal{Z}}^k}{\tau^{k-1}},
\end{flalign}
where the last step follows by using \cref{eq:dec10-3} and \cref{eq:dec13-4}.

Plugging in $\eta = \frac{D_\mathcal{W}}{T \tau \sqrt{\frac{1}{T} + \frac{\nu d \log({1}/{\delta})}{n^2 \varepsilon^2}}}$ above, we get:
\begin{flalign}
    \label{eq:dec13-6}
    \frac{1}{T}\sum_{t=0}^{T-1}\Big(\mathbb{E}[f(\bm{w}_t)] - f(\bm{w}^{*})\Big) & \leq \frac{3 D_\mathcal{W} \tau}{2} \sqrt{\frac{1}{T} + \frac{\nu d \log({1}/{\delta})}{n^2 \varepsilon^2}}
    + \frac{D_\mathcal{W} G_{\mathcal{Z}}^k}{\tau^{k-1}}.
\end{flalign}
Let us choose $\tau = \frac{G}{\gamma^{1/k}}\Big(\frac{1}{T} + \frac{\nu d \log({1}/{\delta})}{n^2 \varepsilon^2}\Big)^{-\frac{1}{2k}}$ above, where $\gamma \in (0,1)$. With that, we get:
\begin{equation}
    \label{eq:dec13-7}
    \frac{1}{T}\sum_{t=0}^{T-1}\Big(\mathbb{E}[f(\bm{w}_t)] - f(\bm{w}^{*})\Big) \leq 
    D_\mathcal{W} \Bigg(\frac{3 G}{2 \gamma^{\frac{1}{k}}} + \frac{G_{\mathcal{Z}}^k \gamma^{1 - \frac{1}{k}}}{G^{k-1}}\Bigg)\Bigg(\frac{1}{T} + \frac{\nu d \log({1}/{\delta})}{n^2 \varepsilon^2}\Bigg)^{\frac{1}{2}(1 - \frac{1}{k})}.
\end{equation}
Now, using Markov's inequality, $G_{\mathcal{Z}}^k \leq \frac{G^k}{\gamma}$ with a probability of at least $1 - \gamma$ w.r.t. the random dataset $\mathcal{Z}$. Plugging this above, we get:
\begin{equation}
    \label{eq:dec13-7-new}
    \frac{1}{T}\sum_{t=0}^{T-1}\Big(\mathbb{E}[f(\bm{w}_t)] - f(\bm{w}^{*})\Big) \leq \frac{5 D_\mathcal{W} G}{2 \gamma^{1/k}} \Bigg(\frac{1}{T} + \frac{\nu d \log({1}/{\delta})}{n^2 \varepsilon^2}\Bigg)^{\frac{1}{2}(1 - \frac{1}{k})},
\end{equation}
with a probability of at least $1 - \gamma$ w.r.t. the random dataset $\mathcal{Z}$.

Lastly, plugging in $\varphi = \frac{\sqrt{\nu d  \log({1}/{\delta})}}{n\varepsilon}$, noting that $\mathbb{E}[f(\bm{w}_{\widehat{t}})] - f(\bm{w}^{*}) = \frac{1}{T}\sum_{t=0}^{T-1}\Big(\mathbb{E}[f(\bm{w}_t)] - f(\bm{w}^{*})\Big)$ and using the definition of $\text{OR}(T)$, we get the final result.
\end{proof}

\begin{lemma}[\textbf{Clipping Bias under \Cref{asmp-lip-gen}}]
\label{bias-heavy-tail}
Under \Cref{asmp-lip-gen}, we have for any $\bm{w} \in \mathcal{W}$:
\begin{equation*}
    \Big\|\frac{1}{n}\sum_{i \in [n]}\textup{clip}(\nabla f_i(\bm{w}), \tau) - \nabla f(\bm{w})\Big\| \leq
    \frac{G_{\mathcal{Z}}^k}{\tau^{k-1}},
\end{equation*}
where $G_{\mathcal{Z}}^k := \frac{1}{n}\sum_{i=1}^n (G(\bm{x}_i, y_i))^k$.
\end{lemma}

\begin{proof}
Using \Cref{cor1-dec22} and \Cref{asmp-lip-gen}, we have for any $\bm{w} \in \mathcal{W}$:
\begin{flalign}
    \Big\|\frac{1}{n}\sum_{i \in [n]}\text{clip}(\nabla f_i(\bm{w}), \tau) - \nabla f(\bm{w})\Big\|
    & = \|\mathbb{E}_i[\text{clip}(\nabla f_i(\bm{w}), \tau)] - \nabla f(\bm{w})\| 
    \\
    \label{eq:jan17-0}
    & \leq \frac{\mathbb{E}_i[\|\nabla f_i(\bm{w})\|^k]}{\tau^{k-1}}
    \\
    \label{eq:dec13-4-0}
    & 
    \leq \frac{G_{\mathcal{Z}}^k}{\tau^{k-1}},
\end{flalign}
where $G_{\mathcal{Z}}^k = \frac{1}{n}\sum_{i=1}^n (G(\bm{x}_i, y_i))^k$.
\end{proof}

\section{Full Version and Proof of Theorem~\ref{thm1-jan19}}
\label{full-1}
\begin{theorem}[\textbf{Unconstrained Convex Case}]
\label{thm1-jan19-full}
Suppose \Cref{dec23-asmp1} holds, $f$ is convex and $\mathcal{W} = \mathbb{R}^d$. Fix some $\gamma \in (0,1)$ and $C > 0$. In \Cref{alg:1}, set $\tau = \frac{G}{\gamma^{1/k}}\Big(\frac{1}{T} + \varphi^2\Big)^{-\frac{1}{k+1}}$ and $\eta_t = \eta = \frac{C}{T \tau} \Big(\frac{1}{T} + \varphi^2\Big)^{-\frac{1}{2}}$ for all $t < T$. Then with a probability of at least $(1 - \gamma)$ which is w.r.t. the random dataset $\mathcal{Z}$ that we obtain, DP-SGD (\Cref{alg:1}) has the following guarantee:
\begin{equation*}
    \textup{OR}(T) \leq \frac{G}{\gamma^{1/k}}\Bigg\{\frac{1}{2}\Bigg(\frac{\|\bm{w}_0 - \bm{w}^{*}\|^2}{C} + 4 C \Bigg)\Big({\frac{1}{T} + \varphi^2}\Big)^{\frac{1}{2}(1 - \frac{2}{k+1})} + 
    (\|\bm{w}_0 - \bm{w}^{*}\| + C)\Big(\frac{1}{T} + \varphi^2\Big)^{(1 - \frac{2}{k+1})}\Bigg\}.
\end{equation*}
So if we set $T = \frac{1}{\varphi^2}$ above (which is what we do in \Cref{thm1-jan19}), we get the following bound for the risk:
\begin{equation*}
    \textup{OR}(T) \leq \frac{G}{\gamma^{1/k}}\Bigg\{\frac{1}{2^{\frac{1}{2}+\frac{1}{k+1}}}\Bigg(\frac{\|\bm{w}_0 - \bm{w}^{*}\|^2}{C} + 4 C \Bigg)\varphi^{(1 - \frac{2}{k+1})} + 
    {2^{1 - \frac{2}{k+1}}(\|\bm{w}_0 - \bm{w}^{*}\| + C)}\varphi^{2(1 - \frac{2}{k+1})}\Bigg\}.
\end{equation*}
\end{theorem}
We prove this result now.

\subsection*{Proof:}
\label{pf-thm1-jan19}
\begin{proof}
Everything remains the same till \cref{eq:dec31-0} in the proof of \Cref{thm1-dec13}. That is, we have:
\begin{flalign}
    \label{eq:jan18-4-0}
    \mathbb{E}[\|\bm{w}_{t+1} - \bm{w}^{*}\|^2] & \leq \mathbb{E}[\|\bm{w}_{t} - \bm{w}^{*}\|^2] - 2{\eta} \mathbb{E}[\langle \nabla f(\bm{w}_t), \bm{w}_{t} - \bm{w}^{*}\rangle]
    \\
    \nonumber
    & \text{ } + {2 \eta} \mathbb{E}\Big[\Big\|\frac{1}{n}\sum_{i \in [n]}\text{clip}(\nabla f_i(\bm{w}_t), \tau) - \nabla f(\bm{w}_t)\Big\| \|\bm{w}_{t} - \bm{w}^{*}\|\Big]
    +
    \eta^2 \mathbb{E}[\|\bm{g}_t\|^2],
\end{flalign}
where $\bm{w}^{*} = \text{argmin}_{\bm{w} \in \mathbb{R}^d} f(\bm{w})$.

Using \Cref{bias-heavy-tail}, we have:
\begin{equation}
    \label{eq:jan18-5-0}
    \mathbb{E}\Big[\Big\|\frac{1}{n}\sum_{i \in [n]}\text{clip}(\nabla f_i(\bm{w}_t), \tau) - \nabla f(\bm{w}_t)\Big\| \|\bm{w}_{t} - \bm{w}^{*}\|\Big] 
    \leq \frac{G_{\mathcal{Z}}^k}{\tau^{k-1}} \mathbb{E}[\|\bm{w}_{t} - \bm{w}^{*}\|],
\end{equation}
where $G_{\mathcal{Z}}^k = \frac{1}{n}\sum_{i=1}^n (G(\bm{x}_i, y_i))^k$.

Now using \Cref{lem1-jan18}
in \cref{eq:jan18-5-0}, we get:
\begin{equation}
    \mathbb{E}\Big[\Big\|\frac{1}{n}\sum_{i \in [n]}\text{clip}(\nabla f_i(\bm{w}_t), \tau) - \nabla f(\bm{w}_t)\Big\| \|\bm{w}_{t} - \bm{w}^{*}\|\Big] \leq \frac{G_{\mathcal{Z}}^k}{\tau^{k-1}}\Bigg(\|\bm{w}_0 - \bm{w}^{*}\| +
    \eta T \Bigg(G_{\mathcal{Z}} + \tau \sqrt{\frac{\nu d \log(1/\delta)}{n^2 \varepsilon^2}}\Bigg)\Bigg).
\end{equation}
Using the above equation and \cref{eq:dec10-3} in \cref{eq:jan18-4-0} as well as the convexity of $f$, we get:
\begin{flalign}
    \label{eq:jan18-7}
    \mathbb{E}[\|\bm{w}_{t+1} - \bm{w}^{*}\|^2] & \leq \mathbb{E}[\|\bm{w}_{t} - \bm{w}^{*}\|^2] - 2{\eta} \mathbb{E}[f(\bm{w}_t) - f(\bm{w}^{*})]
    \\
    \nonumber
    & \text{ } + \frac{{2 \eta} G_{\mathcal{Z}}^k}{\tau^{k-1}}\Bigg(\|\bm{w}_0 - \bm{w}^{*}\| + {\eta T G_{\mathcal{Z}}} + \eta T \tau \sqrt{\frac{\nu d \log(1/\delta)}{n^2 \varepsilon^2}}\Bigg)
    + 2 \eta^2 \tau^2 \Big(1 + \frac{\nu d T \log({1}/{\delta})}{n^2 \varepsilon^2} \Big).
\end{flalign}
Next, summing the above for $t = 0$ through to $T-1$, rearranging a bit and then dividing by $2 \eta T$ throughout, we get the following:
\begin{multline}
    \frac{1}{T}\sum_{t=0}^{T-1}\Big(\mathbb{E}[f(\bm{w}_t)] - f(\bm{w}^{*})\Big) \leq \frac{\|\bm{w}_0 - \bm{w}^{*}\|^2}{2\eta T} +
    {\eta T \tau^2} \Bigg(\frac{\nu d \log({1}/{\delta})}{n^2 \varepsilon^2} + \frac{1}{T}\Bigg)
    \\
    + \frac{G_{\mathcal{Z}}^k}{\tau^{k-1}}\Big(\|\bm{w}_0 - \bm{w}^{*}\| + {\eta T G_{\mathcal{Z}}} + \eta T \tau \sqrt{\frac{\nu d \log(1/\delta)}{n^2 \varepsilon^2}}\Big). \label{eq:eq-101}
\end{multline}
Let us choose $\eta = \frac{C}{T \tau \sqrt{\frac{1}{T} + \frac{\nu d \log({1}/{\delta})}{n^2 \varepsilon^2}}}$, where $C > 0$ is some constant of our choice. With that, we get after simplifying a bit:
\begin{multline*}
    \frac{1}{T}\sum_{t=0}^{T-1}\Big(\mathbb{E}[f(\bm{w}_t)] - f(\bm{w}^{*})\Big) \leq \Big(\frac{\|\bm{w}_0 - \bm{w}^{*}\|^2}{2 C} + C \Big) {\tau} \sqrt{\frac{1}{T} + \frac{\nu d \log({1}/{\delta})}{n^2 \varepsilon^2}}
    + \frac{(\|\bm{w}_0 - \bm{w}^{*}\| + C)G_{\mathcal{Z}}^k}{\tau^{k-1}} 
    \\
    + \frac{C G_{\mathcal{Z}}^{k+1}}{\tau^k} \frac{1}{\sqrt{\frac{1}{T} + \frac{\nu d \log({1}/{\delta})}{n^2 \varepsilon^2}}}. 
\end{multline*}
Let us choose $\tau = \frac{G}{\gamma^{1/k}}\Big(\frac{1}{T} + \frac{\nu d \log({1}/{\delta})}{n^2 \varepsilon^2}\Big)^{-\frac{1}{k+1}}$
above, where $\gamma \in (0,1)$. With that, we get:
\begin{multline}
    \label{eq:aug1-2022-4}
    \frac{1}{T}\sum_{t=0}^{T-1}\Big(\mathbb{E}[f(\bm{w}_t)] - f(\bm{w}^{*})\Big) \leq \frac{G}{\gamma^{1/k}}\Bigg\{\Bigg(\frac{\|\bm{w}_0 - \bm{w}^{*}\|^2}{2 C} + C +\frac{C G_{\mathcal{Z}}^{k+1}}{\big({G}/{\gamma^{1/k}}\big)^{k+1}} \Bigg)\Bigg({\frac{1}{T} + \frac{\nu d \log({1}/{\delta})}{n^2 \varepsilon^2}}\Bigg)^{\frac{1}{2}(1 - \frac{2}{k+1})} 
    \\
    + (\|\bm{w}_0 - \bm{w}^{*}\| + C) \Bigg(\frac{G_{\mathcal{Z}}^{k}}{\big({G}/{\gamma^{1/k}}\big)^{k}}\Bigg) \Bigg({\frac{1}{T} + \frac{\nu d \log({1}/{\delta})}{n^2 \varepsilon^2}}\Bigg)^{(1 - \frac{2}{k+1})}\Bigg\}.
\end{multline}
Let:
\small
\[\text{(I)} := \Bigg(\frac{C G_{\mathcal{Z}}^{k+1}}{\big({G}/{\gamma^{1/k}}\big)^{k+1}}\Bigg)\Bigg({\frac{1}{T} + \frac{\nu d \log({1}/{\delta})}{n^2 \varepsilon^2}}\Bigg)^{\frac{1}{2}(1 - \frac{2}{k+1})} + (\|\bm{w}_0 - \bm{w}^{*}\| + C) \Bigg(\frac{G_{\mathcal{Z}}^{k}}{\big({G}/{\gamma^{1/k}}\big)^{k}}\Bigg) \Bigg({\frac{1}{T} + \frac{\nu d \log({1}/{\delta})}{n^2 \varepsilon^2}}\Bigg)^{(1 - \frac{2}{k+1})}.\]
\normalsize
Now note that $G_{\mathcal{Z}} \leq \frac{G}{\gamma^{1/k}}$ implies:
\[\text{(I)} \leq \underbrace{C \Bigg({\frac{1}{T} + \frac{\nu d \log({1}/{\delta})}{n^2 \varepsilon^2}}\Bigg)^{\frac{1}{2}(1 - \frac{2}{k+1})} + (\|\bm{w}_0 - \bm{w}^{*}\| + C) \Bigg({\frac{1}{T} + \frac{\nu d \log({1}/{\delta})}{n^2 \varepsilon^2}}\Bigg)^{(1 - \frac{2}{k+1})}}_{:=\text{(II)}}.\]
Thus, $\mathbb{P}_{\mathcal{Z}}\big(\text{(I)} \leq \text{(II)}\big) \geq \mathbb{P}_{\mathcal{Z}}(G_{\mathcal{Z}} \leq {G}/{\gamma^{1/k}})$. But using Markov's inequality, $G_{\mathcal{Z}}^k \leq \frac{G^k}{\gamma}$ with a probability of at least $1 - \gamma$ w.r.t. the random dataset $\mathcal{Z}$. Thus, $\text{(I)} \leq \text{(II)}$ with a probability of at least $1 - \gamma$. Using this in \cref{eq:aug1-2022-4}, we get:
\begin{multline}
    \frac{1}{T}\sum_{t=0}^{T-1}\Big(\mathbb{E}[f(\bm{w}_t)] - f(\bm{w}^{*})\Big) \leq \frac{G}{\gamma^{1/k}}\Bigg\{\Bigg(\frac{\|\bm{w}_0 - \bm{w}^{*}\|^2}{2 C} + 2 C \Bigg)\Bigg({\frac{1}{T} + \frac{\nu d \log({1}/{\delta})}{n^2 \varepsilon^2}}\Bigg)^{\frac{1}{2}(1 - \frac{2}{k+1})} +
    \\
    (\|\bm{w}_0 - \bm{w}^{*}\| + C) \Bigg({\frac{1}{T} + \frac{\nu d \log({1}/{\delta})}{n^2 \varepsilon^2}}\Bigg)^{(1 - \frac{2}{k+1})}\Bigg\},
\end{multline}
with a probability of at least $1 - \gamma$ w.r.t. the random dataset $\mathcal{Z}$.

Lastly, plugging in $\varphi = \frac{\sqrt{\nu d  \log({1}/{\delta})}}{n\varepsilon}$, noting that $\mathbb{E}[f(\bm{w}_{\widehat{t}})] - f(\bm{w}^{*}) = \frac{1}{T}\sum_{t=0}^{T-1}\Big(\mathbb{E}[f(\bm{w}_t)] - f(\bm{w}^{*})\Big)$ and using the definition of $\text{OR}(T)$, we get the final result.
\end{proof}

\begin{lemma}
\label{lem1-jan18}
In the setting of the proof of \Cref{thm1-jan19-full}, for any $0 < t < T$, we have:
\begin{equation}
    \mathbb{E}[\|\bm{w}_t - \bm{w}^{*}\|] \leq 
    \|\bm{w}_0 - \bm{w}^{*}\| +
    \eta T \Bigg(G_{\mathcal{Z}} + \tau \sqrt{\frac{\nu d \log(1/\delta)}{n^2 \varepsilon^2}}\Bigg),
\end{equation}
where $G_{\mathcal{Z}}^k := \frac{1}{n}\sum_{i=1}^n (G(\bm{x}_i, y_i))^k$.
\end{lemma}
\begin{proof}
Let us denote $\frac{1}{b}\sum_{i \in \mathcal{S}_t} \text{clip}(\nabla f_i(\bm{w}_t), \tau)$ by $\bm{u}_t$. Now:
\begin{flalign}
    \mathbb{E}_{\mathcal{S}_t}[\|\bm{u}_t\|] & \leq \mathbb{E}_{\mathcal{S}_t}\Big[\frac{1}{b}\sum_{i \in \mathcal{S}_t} \|\text{clip}(\nabla f_i(\bm{w}_t), \tau)\|\Big]
    \\
    & = \frac{1}{n}\sum_{i \in [n]} \|\text{clip}(\nabla f_i(\bm{w}_t), \tau)\|
    \\
    & \leq \frac{1}{n}\sum_{i \in [n]} \|\nabla f_i(\bm{w}_t)\|
    \\
    & \leq \Big(\frac{1}{n}\sum_{i \in [n]} \|\nabla f_i(\bm{w}_t)\|^k\Big)^{1/k} \hspace{4 cm} \text{(using Jensen's inequality)}
    \\
    \label{eq:aug1-2022-3}
    & \leq G_{\mathcal{Z}},
\end{flalign}
where $G_{\mathcal{Z}}^k := \frac{1}{n}\sum_{i=1}^n (G(\bm{x}_i, y_i))^k$. 

Now for any $t>0$:
\begin{flalign}
    \mathbb{E}[\|\bm{w}_t - \bm{w}^{*}\|] & \leq \mathbb{E}[\|(\bm{w}_t - \bm{w}_0) + (\bm{w}_0 - \bm{w}^{*})\|]
    \\
    & \leq 
    \mathbb{E}[\|\bm{w}_t - \bm{w}_0\|] + \|\bm{w}_0 - \bm{w}^{*}\|
    \\
    & \leq \eta \mathbb{E}\Big[\Big\|\sum_{t'=0}^{t-1}(\bm{u}_{t'} + \bm{\zeta}_{t'})\Big\|\Big] + \|\bm{w}_0 - \bm{w}^{*}\|
    \\
    & \leq \eta  \mathbb{E}\Big[\Big\|\sum_{t'=0}^{t-1}\bm{u}_{t'}\Big\|\Big] + \eta \mathbb{E}\Big[\Big\|\sum_{t'=0}^{t-1} \bm{\zeta}_{t'}\Big\|\Big] +  \|\bm{w}_0 - \bm{w}^{*}\|
    \\
    \label{eq:jan18-4-100}
    & \leq \eta \sum_{t'=0}^{t-1} \mathbb{E}[\|\bm{u}_{t'}\|] + \eta \sqrt{\mathbb{E}\Big[\Big\|\sum_{t'=0}^{t-1} \bm{\zeta}_{t'}\Big\|^2\Big]} +  \|\bm{w}_0 - \bm{w}^{*}\|
    \\
    \label{eq:jan18-4}
    & \leq {\eta t G_{\mathcal{Z}}} + \eta \sqrt{t \sigma_n^2 d} + \|\bm{w}_0 - \bm{w}^{*}\|,
\end{flalign}
where \cref{eq:jan18-4} follows
by using \cref{eq:aug1-2022-3} and because $\sum_{t'=0}^{t-1} \bm{\zeta}_{t'}$ is $\mathcal{N}(\vec{0}, t \sigma_n^2 \text{I}_d)$. Plugging in the value of $\sigma_n^2$ and using the fact that $t < T$, we get the desired result.
\end{proof}

{\section{Full Version and Proof of Theorem~\ref{thm1-may4}}
\label{sec-improve-may4}
\begin{theorem}[\textbf{Unconstrained Convex Case Under \Cref{asmp1-may4}}]
\label{thm1-may4-full}
Suppose Assumptions \ref{dec23-asmp1} and \ref{asmp1-may4} hold, $f$ is convex and $\mathcal{W} = \mathbb{R}^d$. 
Fix some $C > 0$. In \Cref{alg:1}, set $T \geq \frac{1}{\varphi^2}$, $\tau = {G}\Big(\frac{1}{T} + \varphi^2\Big)^{-\frac{1}{2k}}$ and $\eta_t = \eta = \frac{C}{T \tau} \Big(\frac{1}{T} + \varphi^2\Big)^{-\frac{1}{2}}$ for all $t < T$. Then with a probability of at least ${3}/{4}$ which is w.r.t. the random dataset $\mathcal{Z}$ that we obtain, DP-SGD (\Cref{alg:1}) has the following \textbf{improved} guarantee:
\begin{equation*}
    \textup{OR}(T) \leq G \Bigg\{\Big(\frac{\|\bm{w}_0 - \bm{w}^{*}\|^2}{C} + 2 C \Big) \Big({\frac{1}{T} + \varphi^2}\Big)^{\frac{1}{2}(1 - \frac{1}{k})}
    + {16} \varphi^{(1 - \frac{1}{k})} D\Bigg\}.
\end{equation*}
So if we set $T = \frac{1}{\varphi^2}$ above (which is what we do in \Cref{thm1-may4}), we get the following bound for the risk:
\begin{equation*}
    \textup{OR}(T) \leq {G}\Bigg\{2^{\frac{1}{2}(1 - \frac{1}{k})}\Big(\frac{\|\bm{w}_0 - \bm{w}^{*}\|^2}{C} + 2 C \Big) + 
    {16} D\Bigg\}\varphi^{(1 - \frac{1}{k})}.
\end{equation*}
\end{theorem}
We now prove this result.

\subsection*{Proof:}
\begin{proof}
Similar to \cref{eq:jan18-4-0} in the proof of \Cref{thm1-jan19-full}, taking expectation only w.r.t. the randomness in the current iteration $t$, we have:
\begin{flalign}
    \label{eq:jan18-4-0-new}
    \mathbb{E}_t[\|\bm{w}_{t+1} - \bm{w}^{*}\|^2] & \leq \|\bm{w}_{t} - \bm{w}^{*}\|^2 - 2{\eta} \langle \nabla f(\bm{w}_t), \bm{w}_{t} - \bm{w}^{*}\rangle
    \\
    \nonumber
    & \text{ } + {2 \eta} \Big\|\frac{1}{n}\sum_{i \in [n]}\text{clip}(\nabla f_i(\bm{w}_t), \tau) - \nabla f(\bm{w}_t)\Big\| \|\bm{w}_{t} - \bm{w}^{*}\|
    +
    \eta^2 \mathbb{E}_t[\|\bm{g}_t\|^2],
\end{flalign}
where $\bm{w}^{*} = \text{argmin}_{\bm{w} \in \mathbb{R}^d} f(\bm{w})$.

\Cref{eq:jan18-5-0} in the proof of \Cref{thm1-jan19-full} also holds here, i.e., we have:
\begin{equation}
    \label{eq:may4-8}
    \Big\|\frac{1}{n}\sum_{i \in [n]}\text{clip}(\nabla f_i(\bm{w}_t), \tau) - \nabla f(\bm{w}_t)\Big\| \|\bm{w}_{t} - \bm{w}^{*}\|
    \leq \frac{G_{\mathcal{Z}}^k}{\tau^{k-1}} \|\bm{w}_{t} - \bm{w}^{*}\|,
\end{equation}
where $G_{\mathcal{Z}}^k = \frac{1}{n}\sum_{i=1}^n (G(\bm{x}_i, y_i))^k$.

Plugging in \cref{eq:may4-8} into \cref{eq:jan18-4-0-new}, we get:
\begin{flalign}
    \label{eq:apr15-1}
    \mathbb{E}_t[\|\bm{w}_{t+1} - \bm{w}^{*}\|^2] & \leq \|\bm{w}_{t} - \bm{w}^{*}\|^2 - 2{\eta}\langle \nabla f(\bm{w}_t), \bm{w}_{t} - \bm{w}^{*}\rangle
    + {2 \eta} \Big(\frac{G_{\mathcal{Z}}^k}{\tau^{k-1}} \|\bm{w}_{t} - \bm{w}^{*}\|\Big) + \eta^2 \mathbb{E}_t[\|\bm{g}_t\|^2].
\end{flalign}
Further, using \cref{eq:dec10-3} to bound $\mathbb{E}_t[\|\bm{g}_t\|^2]$ as well as the convexity of $f$ above, we get:
\begin{equation}
    \label{eq:apr15-2}
    \mathbb{E}_t[\|\bm{w}_{t+1} - \bm{w}^{*}\|^2] \leq \|\bm{w}_{t} - \bm{w}^{*}\|^2 \underbrace{- 2{\eta}(f(\bm{w}_t) - f(\bm{w}^{*})) + {2 \eta} \Big(\frac{G_{\mathcal{Z}}^k}{\tau^{k-1}} \|\bm{w}_{t} - \bm{w}^{*}\|\Big)}_{\text{(I)}} + 2 \eta^2 \tau^2 \Big(1 + \frac{\nu d T \log({1}/{\delta})}{n^2 \varepsilon^2} \Big).
\end{equation}
Now, plugging in our choice of $\tau = {G}\Big(\frac{1}{T} + \frac{\nu d \log({1}/{\delta})}{n^2 \varepsilon^2}\Big)^{-\frac{1}{2k}}$ in $\text{(I)}$ and using the fact that $T \geq \frac{n^2 \varepsilon^2}{\nu d \log({1}/{\delta})}$, we get:
\begin{flalign}
    \text{(I)} &\leq - 2 {\eta}(f(\bm{w}_t) - f(\bm{w}^{*})) + 4 \eta \Big(\frac{G_{\mathcal{Z}}^k}{G^{k-1}}\Big) \Big(\frac{\nu d \log({1}/{\delta})}{n^2 \varepsilon^2}\Big)^{\frac{1}{2}(1 - \frac{1}{k})} \|\bm{w}_{t} - \bm{w}^{*}\|.
\end{flalign}
Now note that $G_{\mathcal{Z}}^k \|\bm{w}_{t} - \bm{w}^{*}\|$ depends on the random dataset $\mathcal{Z}$. But $G_{\mathcal{Z}}^k \leq 4 {G^k}$ implies $G_{\mathcal{Z}}^k \|\bm{w}_{t} - \bm{w}^{*}\| \leq 4 {G^k} \|\bm{w}_{t} - \bm{w}^{*}\|$ for all $t$. Thus, $\mathbb{P}_{\mathcal{Z}}\Big(G_{\mathcal{Z}}^k \|\bm{w}_{t} - \bm{w}^{*}\| \leq 4 {G^k} \|\bm{w}_{t} - \bm{w}^{*}\|, \text{ } \forall t\Big) \geq \mathcal{P}_{\mathcal{Z}}\Big(G_{\mathcal{Z}}^k \leq 4 {G^k}\Big) \geq \frac{3}{4}$, where the last step follows from Markov's inequality. Thus, for all $t$, we have:
\begin{flalign}
    \text{(I)} &\leq - 2 {\eta}(f(\bm{w}_t) - f(\bm{w}^{*})) + {16 \eta G} \Big(\frac{\nu d \log({1}/{\delta})}{n^2 \varepsilon^2}\Big)^{\frac{1}{2}(1 - \frac{1}{k})} \|\bm{w}_{t} - \bm{w}^{*}\|,
\end{flalign}
with a probability of at least $\frac{3}{4}$ w.r.t. the random dataset $\mathcal{Z}$. Henceforth, we will not mention this and it should be inferred directly.

\textbf{Case 1: $\|\bm{w}_{t} - \bm{w}^{*}\| \leq D$}.
\\
In this case, we simply have:
\begin{flalign}
    \label{eq:apr15-3}
    \text{(I)} &\leq - 2 {\eta}(f(\bm{w}_t) - f(\bm{w}^{*})) + {16 \eta G} \Big(\frac{\nu d \log({1}/{\delta})}{n^2 \varepsilon^2}\Big)^{\frac{1}{2}(1 - \frac{1}{k})} D.
\end{flalign}
\textbf{Case 2: $\|\bm{w}_{t} - \bm{w}^{*}\| > D$}.
\\
In this case, we have:
\begin{flalign}
    \label{eq:apr15-4}
    \text{(I)} &\leq - {\eta}(f(\bm{w}_t) - f(\bm{w}^{*})) - {\eta}\underbrace{\Big\{(f(\bm{w}_t) - f(\bm{w}^{*})) - {16 G} \Big(\frac{\nu d \log({1}/{\delta})}{n^2 \varepsilon^2}\Big)^{\frac{1}{2}(1 - \frac{1}{k})} \|\bm{w}_{t} - \bm{w}^{*}\|\Big\}}_{\geq 0 \text{ using \Cref{asmp1-may4}}} 
    \\
    & \leq - {\eta}(f(\bm{w}_t) - f(\bm{w}^{*})).
\end{flalign}
Combining equations (\ref{eq:apr15-3}) and (\ref{eq:apr15-4}), we have:
\begin{flalign}
    \label{eq:apr15-5}
    \text{(I)} &\leq - {\eta}(f(\bm{w}_t) - f(\bm{w}^{*})) + {16 \eta G} \Big(\frac{\nu d \log({1}/{\delta})}{n^2 \varepsilon^2}\Big)^{\frac{1}{2}(1 - \frac{1}{k})} D.
\end{flalign}
Now plugging \cref{eq:apr15-5} in \cref{eq:apr15-2}, we get:
\begin{equation}
    \label{eq:apr15-6}
    \mathbb{E}_t[\|\bm{w}_{t+1} - \bm{w}^{*}\|^2] \leq \|\bm{w}_{t} - \bm{w}^{*}\|^2 - {\eta}(f(\bm{w}_t) - f(\bm{w}^{*})) + {16 \eta G} \Big(\frac{\nu d \log({1}/{\delta})}{n^2 \varepsilon^2}\Big)^{\frac{1}{2}(1 - \frac{1}{k})} D  + 2 \eta^2 \tau^2 \Big(1 + \frac{\nu d T \log({1}/{\delta})}{n^2 \varepsilon^2} \Big).
\end{equation}
Next, summing the above for $t = 0$ through to $T-1$ after taking expectation throughout, rearranging a bit and then dividing by $\eta T$ throughout, we get the following:
\begin{equation}
    \frac{1}{T}\sum_{t=0}^{T-1}\Big(\mathbb{E}[f(\bm{w}_t)] - f(\bm{w}^{*})\Big) \leq \frac{\|\bm{w}_0 - \bm{w}^{*}\|^2}{\eta T} + {2 \eta T \tau^2} \Bigg(\frac{\nu d \log({1}/{\delta})}{n^2 \varepsilon^2} + \frac{1}{T}\Bigg) + {16 G} \Big(\frac{\nu d \log({1}/{\delta})}{n^2 \varepsilon^2}\Big)^{\frac{1}{2}(1 - \frac{1}{k})} D.
\end{equation}
Plugging in our choice of $\eta = \frac{C}{T \tau \sqrt{\frac{1}{T} + \frac{\nu d \log({1}/{\delta})}{n^2 \varepsilon^2}}}$, where $C > 0$ is some constant of our choice, and $\tau = {G}\Big(\frac{1}{T} + \frac{\nu d \log({1}/{\delta})}{n^2 \varepsilon^2}\Big)^{-\frac{1}{2k}}$, we get:
\begin{multline}
    \frac{1}{T}\sum_{t=0}^{T-1}\Big(\mathbb{E}[f(\bm{w}_t)] - f(\bm{w}^{*})\Big) \leq \Big(\frac{\|\bm{w}_0 - \bm{w}^{*}\|^2}{C} + 2 C \Big) G \Big({\frac{1}{T} + \frac{\nu d \log({1}/{\delta})}{n^2 \varepsilon^2}}\Big)^{\frac{1}{2}(1 - \frac{1}{k})}
    \\
    + {16 G} \Big(\frac{\nu d \log({1}/{\delta})}{n^2 \varepsilon^2}\Big)^{\frac{1}{2}(1 - \frac{1}{k})} D.
\end{multline}
Lastly, plugging in $\varphi = \frac{\sqrt{\nu d  \log({1}/{\delta})}}{n\varepsilon}$, noting that $\mathbb{E}[f(\bm{w}_{\widehat{t}})] - f(\bm{w}^{*}) = \frac{1}{T}\sum_{t=0}^{T-1}\Big(\mathbb{E}[f(\bm{w}_t)] - f(\bm{w}^{*})\Big)$ and using the definition of $\text{OR}(T)$, we get the final result.
\end{proof}}

\section{Full Version and Proof of Theorem~\ref{lower_bound}}
\label{app-l}
\begin{theorem}[\textbf{Lower Bound for Unconstrained Convex Case Under \Cref{asmp1-may4}}]
\label{lower_bound_full}
Suppose $\varphi < o(1)$ {and $\delta < \exp(-\varepsilon^2)$}. There exists a convex loss function $\ell$, such that for every $(\varepsilon,\delta)$-DP algorithm $\mathcal{A}$ which tries to solve for $\bm{w}^{*} = \text{arg min}_{\bm{w} \in \mathbb{R}^d} f(\bm{w})$ where $f$ is the average loss for a dataset $\mathcal{Z}$ of $n$ samples drawn from the data distribution $\mathcal{D}$, there exists a choice of $\mathcal{D}$ such that: 
\begin{itemize}
    \item $f$ satisfies Assumptions \ref{dec23-asmp1} and \ref{asmp1-may4} (the latter up to constant terms and with a probability of at least $1 - \exp\big(-\mathcal{O}\big(\nicefrac{\sqrt{d \log(1/\delta)}}{\varepsilon}\big)\big)$ w.r.t. $\mathcal{Z}$).
    \item $\mathbb{E}_{\mathcal{Z} \sim \mathcal{D}^n, \mathcal{A}}\Big[f(\bm{w}_{\mathcal{Z}}^{(\mathcal{A})}) - f(\bm{w}^{*})\Big] \geq \Omega\big(\varphi^{1 - \frac{1}{k}}\big)$, where $\bm{w}_{\mathcal{Z}}^{(\mathcal{A})}$ is the output of algorithm $\mathcal{A}$ on the dataset $\mathcal{Z}$.
\end{itemize}
\end{theorem}

As mentioned in the main text, even though we follow the proof outline of Theorem 6.4 of \citet{kamath2021improved}, our proof is more involved. First, we had to use a \textit{non-obvious} loss function $\ell$ (to obtain the lower bound on) for the \textit{unconstrained} case. Second, since we are in the ERM setting, we have to lower bound the expected training error which is harder than lower bounding the expected generalization error in the SCO setting of \citet{kamath2021improved}; see the footnote in the line after \cref{eq:footnote} (i.e., footnote \ref{ft_note}) for details. {Finally, note that \citet{kamath2021improved} derive a bound for $(0,\rho)$-zCDP while our bound is for $(\varepsilon,\delta)$-DP. To our knowledge, $(\varepsilon,\delta)$-DP does not imply $(0,\tilde{\rho})$-zCDP for some $\tilde{\rho}$ due to which we had to re-derive an important tool in their analysis (specifically, Theorem 1.4 of \citet{kamath2021improved}); see \Cref{zcd-fano}. That is also the reason we had to impose the constraint of $\delta < \exp(-\varepsilon^2)$.}

\begin{proof}
We shall borrow some ideas from \citet{acharya2021differentially} and \citet{kamath2021improved}. 

Let $\mathcal{V}$ be a set of $d$-dimensional points satisfying:
\begin{itemize}
    \item For all $\bm{v} \in \mathcal{V}$, $\bm{v} \in \{0, 1\}^d$ and the number of 1's in $\bm{v}$ is $\frac{d}{2}$.
    \item For all $\bm{v}, \bm{v}' \in \mathcal{V}$, $d_{\text{Ham}}(\bm{v}, \bm{v}') \geq \frac{d}{8}$.
\end{itemize}
By Lemma 6 of \citet{acharya2021differentially}, there must exist such a $\mathcal{V}$ with cardinality at least $2^{\frac{7d}{128}}$. Note that $\|\bm{v}\| = \sqrt{d/2}$ for all $\bm{v} \in \mathcal{V}$.

Next, similar to \citet{kamath2021improved}, let $Q_{\bm{v}}$ be a distribution whose support includes $\vec{0}_d$ and $p^{-1/k} \bm{v}$ for {$p = \frac{2 \sqrt{d \log(1/\delta)}}{n \varepsilon} < \frac{1}{2}$}, such that $\mathbb{P}_{\bm{x} \sim Q_{\bm{v}}}(\bm{x} = \vec{0}_d) = 1-p$ and $\mathbb{P}_{\bm{x} \sim Q_{\bm{v}}}(\bm{x} = p^{-1/k} \bm{v}) = p$. Note that:
\begin{equation}
    \label{eq:86}
    \mathbb{E}_{\bm{x} \sim Q_{\bm{v}}}[\bm{x}] = p^{1 - \frac{1}{k}}\bm{v} \text{ and } \mathbb{E}_{\bm{x} \sim Q_{\bm{v}}}[\|\bm{x}\|^k] = \|\bm{v}\|^k.
\end{equation}
Now, let $\ell:\mathbb{R}^{d} \times \mathbb{R}^{d} \xrightarrow{} \mathbb{R}$ be defined as:
\begin{equation}
    \ell(\bm{w}, \bm{x}) = - \langle \bm{w}, \bm{x} \rangle + 2 \|\bm{x}\| \max\big(\|\bm{w}\| - 1, 0\big).
\end{equation}
We have a dataset $\mathcal{Z}_{\bm{v}}$ of $n$ i.i.d. samples $\{\bm{x}_i\}_{i=1}^n$ drawn from $Q_{\bm{v}}$. Then, we denote the $i^{\text{th}}$ sample's loss in $\mathcal{Z}_{\bm{v}}$ by:
\begin{equation}
    f_{\bm{v},i}(\bm{w}) = \ell(\bm{w}, \bm{x}_i) = - \langle \bm{w}, \bm{x}_i \rangle + 2 \|\bm{x}_i\| \max\big(\|\bm{w}\| - 1, 0\big).
\end{equation}
Here, the subscript $\bm{v}$ denotes that $\mathcal{Z}_{\bm{v}}$ is drawn from $Q_{\bm{v}}$. Next, the average loss is:
\begin{equation}
    \label{eq:88}
    f_{\bm{v}}(\bm{w}) = \frac{1}{n}\sum_{i=1}^n f_{\bm{v},i}(\bm{w}) = - \langle \bm{w}, \overline{\bm{x}} \rangle + 2 \|\overline{\bm{x}}\| \max\big(\|\bm{w}\| - 1, 0\big),
\end{equation}
where $\overline{\bm{x}} = \frac{1}{n} \sum_{i=1}^n \bm{x}_i$. \Cref{eq:88} follows because each $\bm{x}_i = c_i \bm{v}$, where $c_i = 0$ with a probability of $1-p$ and $c_i = p^{-1/k}$ otherwise, due to which $\frac{1}{n}\sum_{i=1}^n \|\bm{x}_i\| = \Big\|\frac{1}{n}\sum_{i=1}^n \bm{x}_i\Big\| = \|\overline{\bm{x}}\|$.

Note that $q(\bm{w}) = \max\big(\|\bm{w}\| - 1, 0\big)$ is convex. This is because both $q_1(\bm{w}) = \|\bm{w}\| - 1$ and $q_2(\bm{w}) = 0$ are convex and the point-wise maximum of  convex functions is also convex. Thus, $f_{\bm{v}}(\bm{w})$ is convex. 

Now, we shall show that \Cref{dec23-asmp1} holds. It can be checked that $\|\nabla \ell(\bm{w}, \bm{x})\| \leq 3 \|\bm{x}\|$ for all $\bm{w} \in \mathbb{R}^d$. Thus,
\begin{equation}
    \label{eq:90}
    \mathbb{E}_{\bm{x} \sim Q_{\bm{v}}}[\|\nabla \ell(\bm{w}, \bm{x})\|^k] \leq 3^k \mathbb{E}_{\bm{x} \sim Q_{\bm{v}}}[\|\bm{x}\|^k] = (3 \|\bm{v}\|)^k,
\end{equation}
where the last step follows from \cref{eq:86}. Thus, \Cref{dec23-asmp1} holds here with $G = 3 \|\bm{v}\| = 3\sqrt{\frac{d}{2}}$.

Let us first obtain $\text{arg min}_{\bm{w} \in \mathbb{R}^d} f_{\bm{v}}(\bm{w})$. Note that if $\|\overline{\bm{x}}\| = 0$ (which happens when all the $\bm{x}_i$'s are $\vec{0}_d$), $f_{\bm{v}}$ is identically 0 and so any point is a minimizer. So, let us focus on the case of $\|\overline{\bm{x}}\| > 0$; we claim that $\text{arg min}_{\bm{w} \in \mathbb{R}^d} f_{\bm{v}}(\bm{w}) = \hat{\bm{x}} := {\overline{\bm{x}}}/{\|\overline{\bm{x}}\|}$. Clearly, $\hat{\bm{x}} = \text{arg min}_{\bm{w}:\|\bm{w}\| \leq 1} f_{\bm{v}}(\bm{w})$. Also, for any $\bm{w}$ such that $\|\bm{w}\| > 1$, we have:
\begin{flalign}
    \label{eq:91}
   f_{\bm{v}}(\bm{w}) - f_{\bm{v}}(\hat{\bm{x}}) &= - \langle \bm{w}, \overline{\bm{x}} \rangle + 2 \|\overline{\bm{x}}\| \big(\|\bm{w}\| - 1\big) + \|\overline{\bm{x}}\|
   \\
   &= \Big(-\langle \bm{w}, \overline{\bm{x}} \rangle + \|\overline{\bm{x}}\| \|\bm{w}\|\Big) + \|\overline{\bm{x}}\| \|\bm{w}\| - \|\overline{\bm{x}}\|
   \\
   \label{eq:93-old}
   & \geq \|\overline{\bm{x}}\| (\|\bm{w}\| - 1)
   \\
   \label{eq:94}
   & > 0.
\end{flalign}
\Cref{eq:93} follows from the Cauchy-Schwarz inequality while \cref{eq:94} follows because $\|\bm{w}\| > 1$. Thus, $f_{\bm{v}}(\bm{w}) > f_{\bm{v}}(\hat{\bm{x}})$ for all $\bm{w}$ such that $\|\bm{w}\| > 1$. Hence, we must have $\text{arg min}_{\bm{w} \in \mathbb{R}^d} f_{\bm{v}}(\bm{w}) = \hat{\bm{x}}$. Also, note that since each $\bm{x}_i = c_i \bm{v}$ where $c_i \in \{0,p^{-1/k}\}$, $\overline{\bm{x}} = \Big(\frac{1}{n}\sum_{i=1}^n c_i \Big) \bm{v}$ with $\frac{1}{n}\sum_{i=1}^n c_i > 0$ and so,
\begin{equation}
    \label{eq:95-new2}
    \hat{\bm{x}} = \bm{v}/\|\bm{v}\| := \bm{w}_{\bm{v}}^{*},
\end{equation}
which is independent of the dataset $\mathcal{Z}_{\bm{v}}$ that we receive; in fact, $\bm{w}_{\bm{v}}^{*} = \text{arg min}_{\bm{w} \in \mathbb{R}^d} \mathbb{E}_{\bm{x} \sim Q_{\bm{v}}}[\ell(\bm{w}, \bm{x})]$. We shall be using these facts later.

Rewriting \cref{eq:93-old} by replacing $\hat{\bm{x}}$ with $\bm{w}_{\bm{v}}^{*}$ in the LHS, we get:
\begin{equation}
    \label{eq:93}
    f_{\bm{v}}(\bm{w}) - f_{\bm{v}}(\bm{w}_{\bm{v}}^{*}) \geq \|\overline{\bm{x}}\| (\|\bm{w}\| - 1) \text{ } \forall \text{ } \bm{w} \in \mathbb{R}^d.
\end{equation}
Also note that \cref{eq:93} holds trivially for the case of $\|\overline{\bm{x}}\| = 0$. 

Finally, let us show that \Cref{asmp1-may4} also holds up to constant factors {with high-probability} over the dataset $\mathcal{Z}_{\bm{v}}$. Specifically, we shall restrict our attention to $\|\overline{\bm{x}}\| \geq \frac{p^{1 - \frac{1}{k}}}{2} \|\bm{v}\|$. Since $\mathbb{E}[\|\overline{\bm{x}}\|] = p^{1 - \frac{1}{k}} \|\bm{v}\|$, using the Chernoff bound for Binomial random variables\footnote{Specifically, for $n$ i.i.d. binomial ($\in \{0,1\}$) random variables $\{Z_i\}_{i=1}^n$ with $\mathbb{P}(Z_i = 1) = \mu$ $\forall$ $i \in [n]$, we use $\mathbb{P}\Big(\frac{1}{n}\sum_{i=1}^n Z_i \leq {\mu}(1-\beta)\Big) \leq \exp(-\frac{n \mu \beta^2}{3})$ for $\beta \in (0,1)$.}, we have $\|\overline{\bm{x}}\| \geq \frac{p^{1 - \frac{1}{k}}}{2} \|\bm{v}\|$ with a probability of at least
\begin{equation}
    \label{eq:98}
    1 - \exp(-\mathcal{O}(n p)) = 1 - \exp\Big(-\mathcal{O}\Big(\frac{\sqrt{d \log(1/\delta)}}{\varepsilon}\Big)\Big), 
\end{equation}
where the last step follows by plugging in $p = \frac{2 \sqrt{d \log(1/\delta)}}{n \varepsilon}$.

Now, let us consider all $\bm{w}$ such that $\|\bm{w} - \bm{w}_{\bm{v}}^{*}\| \geq 4$ -- in this case, since $\|\bm{w}_{\bm{v}}^{*}\| = 1$, we must have $\|\bm{w}\| \geq 3$. Further 
\begin{flalign}
    \label{eq:95}
    \|\bm{w} - \bm{w}_{\bm{v}}^{*}\| & \leq \|\bm{w}\| + 1 
    \\
    \label{eq:96}
    & \leq 2 (\|\bm{w}\| - 1),
\end{flalign}
where the last step holds for $\|\bm{w}\| \geq 3$. \Cref{eq:95} follows by using the triangle inequality and the fact that $\|\bm{w}_{\bm{v}}^{*}\| = 1$. Using \cref{eq:96} in \cref{eq:93}, we get for $\|\bm{w} - \bm{w}_{\bm{v}}^{*}\| \geq 4$:
\begin{equation}
    \label{eq:100-asmp}
    f_{\bm{v}}(\bm{w}) - f_{\bm{v}}(\bm{w}_{\bm{v}}^{*}) \geq \Omega\Big(\|\overline{\bm{x}}\| \|\bm{w} - \bm{w}_{\bm{v}}^{*}\|\Big).
\end{equation}
Recall that we are restricting our attention to $\|\overline{\bm{x}}\| \geq \frac{p^{1 - \frac{1}{k}}}{2} \|\bm{v}\|$ which occurs with probability $\geq 1 - \exp\Big(-\mathcal{O}\Big(\frac{\sqrt{d \log(1/\delta)}}{\varepsilon}\Big)\Big)$ (w.r.t. the random dataset $\mathcal{Z}_{\bm{v}}$). 
Using this as well as the fact that $\|\bm{w} - \bm{w}_{\bm{v}}^{*}\|$ is independent of $\mathcal{Z}_{\bm{v}}$ in \cref{eq:100-asmp}, and then plugging in the value of $p = \frac{2 \sqrt{d \log(1/\delta)}}{n \varepsilon}$ and $G = 3 \|v\|$ (as defined after \cref{eq:90}), we get:
\begin{equation}
    f_{\bm{v}}(\bm{w}) - f_{\bm{v}}(\bm{w}_{\bm{v}}^{*}) \geq \Omega\Big( p^{1 - \frac{1}{k}} {\|\bm{v}\|} \|\bm{w} - \bm{w}_{\bm{v}}^{*}\|\Big) = \Omega\Big(\big(\nicefrac{\sqrt{d \log(1/\delta)}}{n \varepsilon}\big)^{1 - \frac{1}{k}} G \|\bm{w} - \bm{w}_{\bm{v}}^{*}\|\Big),
\end{equation}
with probability $\geq 1 - \exp\Big(-\mathcal{O}\Big(\frac{\sqrt{d \log(1/\delta)}}{\varepsilon}\Big)\Big)$. Hence, \Cref{asmp1-may4} holds up to constant factors {with high-probability} over the dataset $\mathcal{Z}_{\bm{v}}$.

Now that we have shown that $f_{\bm{v}}$ is convex and satisfies Assumptions \ref{dec23-asmp1} and \ref{asmp1-may4}, let us move onto the lower bound. Let us denote the output of the $(\varepsilon, \delta)$-DP algorithm $\mathcal{A}$ (on the dataset $\mathcal{Z}_{\bm{v}}$) by $\bm{w}_{\bm{v}}^{(\mathcal{A})}$\footnote{A better choice of notation would have been to put $\mathcal{Z}_{\bm{v}}$ instead of just $\bm{v}$ in the subscript; however, we do not do so to avoid overloading notation going further.}. We shall consider two cases.

\textbf{Case 1, $\|\bm{w}_{\bm{v}}^{(\mathcal{A})}\| \geq 3$:} Using \cref{eq:93}, we get
\begin{flalign}
    \label{eq:101-new}
    f_{\bm{v}}(\bm{w}_{\bm{v}}^{(\mathcal{A})}) - f_{\bm{v}}(\bm{w}_{\bm{v}}^{*}) & \geq 2 \|\overline{\bm{x}}\|.
\end{flalign}

\textbf{Case 2, $\|\bm{w}_{\bm{v}}^{(\mathcal{A})}\| < 3$:} 
First, let us consider $\|\bm{w}_{\bm{v}}^{(\mathcal{A})}\| \in (1,3)$. Rewriting \cref{eq:91} for $\|\overline{\bm{x}}\| > 0$ by replacing $\hat{\bm{x}}$ with $\bm{w}_{\bm{v}}^{*}$ in the LHS, we get:
\begin{equation}
    \label{eq:91-new}
    f_{\bm{v}}(\bm{w}) - f_{\bm{v}}(\bm{w}_{\bm{v}}^{*}) \geq - \langle \bm{w}, \overline{\bm{x}} \rangle + 2 \|\overline{\bm{x}}\| \big(\|\bm{w}\| - 1\big) + \|\overline{\bm{x}}\| \text{ } \forall \text{ } \bm{w} \in \mathbb{R}^d.
\end{equation}
Also note that \cref{eq:91-new} holds trivially for the case of $\|\overline{\bm{x}}\| = 0$. Using \cref{eq:91-new}, we get
\begin{flalign}
    f_{\bm{v}}(\bm{w}_{\bm{v}}^{(\mathcal{A})}) - f_{\bm{v}}(\bm{w}_{\bm{v}}^{*}) & \geq - \langle \bm{w}_{\bm{v}}^{(\mathcal{A})}, \overline{\bm{x}} \rangle + 2 \|\overline{\bm{x}}\| \|\bm{w}_{\bm{v}}^{(\mathcal{A})}\| - \|\overline{\bm{x}}\|
    \\
    & = \|\overline{\bm{x}}\| \Bigg(- \Big\langle \bm{w}_{\bm{v}}^{(\mathcal{A})}, \underbrace{\frac{\overline{\bm{x}}}{\|\overline{\bm{x}}\|}}_{\bm{w}_{\bm{v}}^{*}}\Big\rangle + 2\|\bm{w}_{\bm{v}}^{(\mathcal{A})}\| - 1\Bigg)
    \\
    & = \frac{\|\overline{\bm{x}}\|}{2} \Big({1} -2 \langle \bm{w}_{\bm{v}}^{(\mathcal{A})}, \bm{w}_{\bm{v}}^{*}\rangle + {4\|\bm{w}_{\bm{v}}^{(\mathcal{A})}\| - 3}\Big).
\end{flalign}
Next, using the fact that $\|\bm{w}_{\bm{v}}^{*}\|^2 = 1$ and $4\|\bm{w}_{\bm{v}}^{(\mathcal{A})}\| - 3 \geq \|\bm{w}_{\bm{v}}^{(\mathcal{A})}\|^2$ for $\|\bm{w}_{\bm{v}}^{(\mathcal{A})}\| \in (1,3)$ above, we get:
\begin{flalign}
   f_{\bm{v}}(\bm{w}_{\bm{v}}^{(\mathcal{A})}) - f_{\bm{v}}(\bm{w}_{\bm{v}}^{*}) \geq \frac{\|\overline{\bm{x}}\|}{2} \Big(\|\bm{w}_{\bm{v}}^{*}\|^2 - 2 \langle \bm{w}_{\bm{v}}^{(\mathcal{A})}, \bm{w}_{\bm{v}}^{*}\rangle + \|\bm{w}_{\bm{v}}^{(\mathcal{A})}\|^2\Big) = \frac{\|\overline{\bm{x}}\|}{2} \|\bm{w}_{\bm{v}}^{(\mathcal{A})} - \bm{w}_{\bm{v}}^{*}\|^2.
\end{flalign}
Let us now consider $\|\bm{w}_{\bm{v}}^{(\mathcal{A})}\| \leq 1$. In this case:
\begin{flalign}
    f_{\bm{v}}(\bm{w}_{\bm{v}}^{(\mathcal{A})}) - f_{\bm{v}}(\bm{w}_{\bm{v}}^{*}) & \geq - \langle \bm{w}_{\bm{v}}^{(\mathcal{A})}, \overline{\bm{x}} \rangle + \|\overline{\bm{x}}\|
    \\
    & = \|\overline{\bm{x}}\| \Bigg(- \Big\langle \bm{w}_{\bm{v}}^{(\mathcal{A})}, \underbrace{\frac{\overline{\bm{x}}}{\|\overline{\bm{x}}\|}}_{\bm{w}_{\bm{v}}^{*}}\Big\rangle + 1\Bigg)
    \\
    & = \frac{\|\overline{\bm{x}}\|}{2} \Big(1 + 1 - 2 \langle \bm{w}_{\bm{v}}^{(\mathcal{A})}, \bm{w}_{\bm{v}}^{*}\rangle\Big)
    \\
    \label{eq:108}
    & \geq \frac{\|\overline{\bm{x}}\|}{2} \Big(\|\bm{w}_{\bm{v}}^{(\mathcal{A})}\|^2 + \|\bm{w}_{\bm{v}}^{*}\|^2 - 2 \langle \bm{w}_{\bm{v}}^{(\mathcal{A})}, \bm{w}_{\bm{v}}^{*}\rangle\Big)
    \\
    & = \frac{\|\overline{\bm{x}}\|}{2} \|\bm{w}_{\bm{v}}^{(\mathcal{A})} - \bm{w}_{\bm{v}}^{*}\|^2.
\end{flalign}
\Cref{eq:108} follows because $\|\bm{w}_{\bm{v}}^{*}\|=1$ and $\|\bm{w}_{\bm{v}}^{(\mathcal{A})}\| \leq 1$. 

So for both $\|\bm{w}_{\bm{v}}^{(\mathcal{A})}\| \in (1,3)$ and $\|\bm{w}_{\bm{v}}^{(\mathcal{A})}\| \leq 1$, we have:
\begin{equation}
    \label{eq:111-new}
    f_{\bm{v}}(\bm{w}_{\bm{v}}^{(\mathcal{A})}) - f_{\bm{v}}(\bm{w}_{\bm{v}}^{*}) \geq \frac{\|\overline{\bm{x}}\|}{2} \|\bm{w}_{\bm{v}}^{(\mathcal{A})} - \bm{w}_{\bm{v}}^{*}\|^2.
\end{equation}
Combining the results of both cases, i.e., \cref{eq:101-new} and \cref{eq:111-new}, we get:
\begin{equation}
    \label{eq:footnote}
    f_{\bm{v}}(\bm{w}_{\bm{v}}^{(\mathcal{A})}) - f_{\bm{v}}(\bm{w}_{\bm{v}}^{*}) \geq
    2 \|\overline{\bm{x}}\| \min\Big(1, \frac{1}{4}\|\bm{w}_{\bm{v}}^{(\mathcal{A})} - \bm{w}_{\bm{v}}^{*}\|^2\Big).
\end{equation}
Now taking expectation w.r.t. ${Q}_{\bm{v}}$ and $\mathcal{A}$ (henceforth, we shall omit these in the subscript of expectation for brevity), we get\footnote{\label{ft_note}Lower bounding $\mathbb{E}\big[f_{\bm{v}}(\bm{w}_{\bm{v}}^{(\mathcal{A})}) - f_{\bm{v}}(\bm{w}_{\bm{v}}^{*})\big]$, i.e. the training error, is harder than lower bounding $\mathbb{E}\big[\ell(\bm{w}_{\bm{v}}^{(\mathcal{A})}) - \ell(\bm{w}_{\bm{v}}^{*})\big]$, i.e. the generalization error. This is because the lower bound for the training error includes $\|\overline{\bm{x}}\|$ and $\bm{w}_{\bm{v}}^{(\mathcal{A})}$, both of which depend on the dataset $\mathcal{Z}$, making the expectation of the lower bound challenging to compute. In contrast, in the lower bound for the generalization error, $\|\overline{\bm{x}}\|$ is replaced by $\|\mathbb{E}_{\bm{x} \sim Q_{\bm{v}}}[\bm{x}]\|$ and so, $\bm{w}_{\bm{v}}^{(\mathcal{A})}$ is the only quantity depending on $\mathcal{Z}$, making the expectation of the corresponding lower bound much simpler to compute.}:
\begin{flalign}
    \mathbb{E}\Big[f_{\bm{v}}(\bm{w}_{\bm{v}}^{(\mathcal{A})}) - f_{\bm{v}}(\bm{w}_{\bm{v}}^{*})\Big] & \geq
    2 \mathbb{E}\Big[\|\overline{\bm{x}}\| \min\Big(1, \frac{1}{4}\|\bm{w}_{\bm{v}}^{(\mathcal{A})} - \bm{w}_{\bm{v}}^{*}\|^2\Big)\Big]
    \\
    & \geq 2 \mathbb{E}\Bigg[\|\overline{\bm{x}}\| \min\Big(1, \frac{1}{4}\|\bm{w}_{\bm{v}}^{(\mathcal{A})} - \bm{w}_{\bm{v}}^{*}\|^2\Big) \Big| \|\overline{\bm{x}}\| \geq \frac{p^{1 - \frac{1}{k}}}{2} \|\bm{v}\| \Bigg] \underbrace{\mathbb{P}\Big(\|\overline{\bm{x}}\| \geq \frac{p^{1 - \frac{1}{k}}}{2} \|\bm{v}\| \Big)}_{= \Theta(1)}
    \\
    \label{eq:115-new}
    & \geq \Omega\big({p^{1 - \frac{1}{k}}} \|\bm{v}\|\big) \mathbb{E}\Bigg[\min\Big(1, \frac{1}{4}\|\bm{w}_{\bm{v}}^{(\mathcal{A})} - \bm{w}_{\bm{v}}^{*}\|^2\Big) \Big| \|\overline{\bm{x}}\| \geq \frac{p^{1 - \frac{1}{k}}}{2} \|\bm{v}\|\Bigg].
\end{flalign}
To obtain \cref{eq:115-new}, we have used the fact that $\mathbb{P}\Big(\|\overline{\bm{x}}\| \geq \frac{p^{1 - \frac{1}{k}}}{2} \|\bm{v}\| \Big) = \Theta(1)$ from \cref{eq:98}. Next, letting $M_{\bm{v}, \mathcal{A}} = \min\Big(1, \frac{1}{4}\|\bm{w}_{\bm{v}}^{(\mathcal{A})} - \bm{w}_{\bm{v}}^{*}\|^2\Big)$, we have:
\begin{flalign}
   \mathbb{E}[M_{\bm{v}, \mathcal{A}}] = \mathbb{E}\Bigg[M_{\bm{v}, \mathcal{A}} \Big| \|\overline{\bm{x}}\| \geq \frac{p^{1 - \frac{1}{k}}}{2} \|\bm{v}\|\Bigg] \mathbb{P}\Big(\|\overline{\bm{x}}\| \geq \frac{p^{1 - \frac{1}{k}}}{2} \|\bm{v}\|\Big) + \mathbb{E}\Bigg[M_{\bm{v}, \mathcal{A}} \Big| \|\overline{\bm{x}}\| < \frac{p^{1 - \frac{1}{k}}}{2} \|\bm{v}\|\Bigg] {\mathbb{P}\Big(\|\overline{\bm{x}}\| < \frac{p^{1 - \frac{1}{k}}}{2} \|\bm{v}\|\Big)}.
\end{flalign}
Now note that $\mathbb{E}\Bigg[M_{\bm{v}, \mathcal{A}} \Big| \|\overline{\bm{x}}\| < \frac{p^{1 - \frac{1}{k}}}{2} \|\bm{v}\|\Bigg] \leq 1$ (as $M_{\bm{v}, \mathcal{A}} \leq 1$ by definition) and from \cref{eq:98}, $\mathbb{P}\Big(\|\overline{\bm{x}}\| < \frac{p^{1 - \frac{1}{k}}}{2} \|\bm{v}\|\Big) \leq \exp\Big(-\mathcal{O}\Big(\frac{\sqrt{d \log(1/\delta)}}{\varepsilon}\Big)\Big)$. Also, from \Cref{lem-lb}, there exists a $\bm{v}$, say $\hat{\bm{v}}$, for which $\mathbb{E}[M_{\hat{\bm{v}}, \mathcal{A}}] = \Omega(1)$ for $\delta < \exp(-\varepsilon^2)$. Thus, for $\bm{v} = \hat{\bm{v}}$:
\begin{equation}
    \mathbb{E}\Bigg[M_{\hat{\bm{v}}, \mathcal{A}} \Big| \|\overline{\bm{x}}\| \geq \frac{p^{1 - \frac{1}{k}}}{2} \|\hat{\bm{v}}\|\Bigg] = \Theta \Big(\mathbb{E}[M_{\hat{\bm{v}}, \mathcal{A}}]\Big) = \Omega(1).
\end{equation}
Putting this back in \cref{eq:115-new}, we get for $\bm{v} = \hat{\bm{v}}$:
\begin{equation}
    \mathbb{E}\Big[f_{\hat{\bm{v}}}(\bm{w}_{\hat{\bm{v}}}^{(\mathcal{A})}) - f_{\hat{\bm{v}}}(\bm{w}_{\hat{\bm{v}}}^{*})\Big] \geq \Omega\Big({p^{1 - \frac{1}{k}}} \|\hat{\bm{v}}\|\Big) = \Omega\Big(\Big(\frac{\sqrt{d \log(1/\delta)}}{n \varepsilon}\Big)^{1 - \frac{1}{k}} G\Big),
\end{equation}
where the last step follows by plugging in the values of $p$ and $G$.

This finishes the proof.
\end{proof}

\begin{lemma}
\label{lem-lb}
In the setting of the proof of Theorem~\ref{lower_bound_full} and for $\delta < \exp(-\varepsilon^2)$, there exists some $\bm{\hat{v}} \in \mathcal{V}$ such that:
\begin{equation}
    \mathbb{E}_{{Q}_{\hat{\bm{v}}}, \mathcal{A}}\Bigg[\min\Big(1, \frac{1}{4}\|\bm{w}_{\hat{\bm{v}}}^{(\mathcal{A})} - \bm{w}_{\hat{\bm{v}}}^{*}\|^2\Big)\Bigg] \geq \Omega(1). 
\end{equation}
\end{lemma}

\begin{proof}
From Lemma 22 of \citet{bun2016concentrated}, note that $(\varepsilon, \delta)$-DP implies $\big(5 (\frac{\varepsilon^2}{\log({1}/{\delta})})^{1/4} - \frac{\varepsilon^2}{4 \log({1}/{\delta})}, \frac{\varepsilon^2}{4 \log({1}/{\delta})})$-zCDP; this follows by setting $\hat{\xi} = 0$ \& $\hat{\rho} = \frac{\varepsilon^2}{\log({1}/{\delta})}$ in that lemma. Since $\hat{\rho} < 1$ in that lemma, we impose the constraint of $\delta < \exp(-\varepsilon^2)$.

Now, we shall use \Cref{zcd-fano} (stated and proved after this proof). In \Cref{zcd-fano}, let us use the loss function $\widetilde{l}(\bm{w}, \bm{w}') = \min\Big(1, \frac{1}{4}\|\bm{w} - \bm{w}'\|^2\Big)$. 
As mentioned after \cref{eq:95-new2}, $\bm{w}_{\bm{v}}^{*} = \frac{\bm{v}}{\|\bm{v}\|}$ is the minimizer of the expected loss $\ell$ over the distribution $Q_{\bm{v}}$. Now:
\begin{equation}
    \widetilde{l}(\bm{w}_{{\bm{v}}}^{*}, \bm{w}_{{\bm{v}'}}^{*}) = \min\Big(1, \frac{1}{4}\|\bm{w}_{{\bm{v}}}^{*} - \bm{w}_{{\bm{v}'}}^{*}\|^2\Big) = \min\Bigg(1, \frac{1}{4}\Bigg\|\frac{\bm{v}}{\|\bm{v}\|} - \frac{\bm{v}'}{\|\bm{v}'\|}\Bigg\|^2\Bigg). 
\end{equation}
Using the fact that $\|\bm{v}\| = \|\bm{v}'\| = \sqrt{\frac{d}{2}}$ and $d_{\text{Ham}}(\bm{v}, \bm{v}') \geq \frac{d}{8}$ for all $\bm{v} \neq \bm{v}'$, we have that $\widetilde{l}(\bm{w}_{{\bm{v}}}^{*}, \bm{w}_{{\bm{v}'}}^{*}) \geq \Omega(1)$ for all $\bm{v} \neq \bm{v}'$. 
Also, $d_{TV}(Q_{\bm{v}}, Q_{\bm{v}'}) = p = \frac{2 \sqrt{d \log(1/\delta)}}{n \varepsilon}$ by definition and $\log |\mathcal{V}| = \mathcal{O}(d)$. Using all of this in \Cref{zcd-fano}, we get:
\begin{equation}
    \frac{1}{|\mathcal{V}|} \sum_{\bm{v} \in \mathcal{V}} \mathbb{E}\Big[\widetilde{l}(\bm{w}_{{\bm{v}}}^{(\mathcal{A})}, \bm{w}_{{\bm{v}}}^{*})\Big] \geq \Omega(1).
\end{equation}
Since the average over all $\bm{v}$ is $\Omega(1)$, there must exist some $\bm{v}$, say $\hat{\bm{v}}$, for which:
\begin{equation}
    \mathbb{E}\Big[\widetilde{l}(\bm{w}_{\hat{\bm{v}}}^{(\mathcal{A})}, \bm{w}_{\hat{\bm{v}}}^{*})\Big] \geq \Omega(1).
\end{equation}
This completes the proof.
\end{proof}

\begin{lemma}[\textbf{$(\xi,\rho)$-zCDP  Fano’s inequality}: based on Thm. 1.4 of \citet{kamath2021improved}]
\label{zcd-fano}
Let $\{p_1,\ldots,p_M\} \subseteq \mathcal{P}$ be a set of probability distributions, $\bm{\theta}: \mathcal{P} \xrightarrow{} \mathbb{R}^d$ be a parameter of interest, and $\tilde{l}: \mathbb{R}^d \times \mathbb{R}^d \xrightarrow{} \mathbb{R}$ be a loss function. Suppose for all $i \neq j$, it satisfies (a) $\tilde{l}(\bm{\theta}(p_i), \bm{\theta}(p_j)) \geq r$, (b) $d_{TV}(p_i, p_j) \leq \alpha$, (c) $d_{KL}(p_i, p_j) \leq \beta$. Then for any $(\xi,\rho)$-zCDP estimator $\bm{\hat{\theta}}$:
\begin{multline}
    \frac{1}{M}\sum_{i \in M} \mathbb{E}_{\bm{X} \sim p_i^n} \Big[\tilde{l}\big(\bm{\hat{\theta}}(X), \bm{\theta}(p_i)\big)\Big] \geq 
    \\
    \frac{r}{2} \textup{max}\Big\{1 - \frac{\beta+\log 2}{\log M}, 1 - \frac{\rho(n^2 \alpha^2 + n \alpha (1-\alpha)) + {\color{red} \xi(u n \alpha \log(e u n \alpha) + n \log (e n) \exp(-\frac{n \alpha (u-1)^2}{3}))} + \log 2}{\log M} \Big\}, 
    \\
    \text{ for any } u > 1.
\end{multline}
\end{lemma}
Compared to Theorem 1.4 of \citet{kamath2021improved} (who only consider the case of $\xi=0$), note the extra term in red.
\begin{proof}
The proof is largely based on the proof of Theorem 1.4 of \citet{kamath2021improved}; we mention the changes required (while following their notation).

The first change is that instead of $d_{\text{KL}}(\hat{p}(x), \hat{p}(x')) \leq \rho d_{\text{Ham}}(x, x')^2$ (in \citet{kamath2021improved}), here we have $d_{\text{KL}}(\hat{p}(x), \hat{p}(x')) \leq {\color{blue}\xi d_{\text{Ham}}(x, x') \log(e . d_{\text{Ham}}(x, x'))} + \rho d_{\text{Ham}}(x, x')^2$; this follows using Proposition 27 of \citet{bun2016concentrated} and the fact that $\sum_{i=1}^k \frac{1}{i} \leq \log(e.k)$. Secondly, due to the extra term in blue, we need to bound $\mathbb{E}_{z \sim \text{Bin}(n, \alpha)}[z \log(e. z)]$ (this is because $d_{\text{Ham}}(x, x') \sim \text{Bin}(n,\alpha)$ as discussed in the proof of Theorem 1.4 of \citet{kamath2021improved}). 
For any $u > 1$, we have: 
\begin{flalign}
    \mathbb{E}[z \log(e. z)] & \leq \mathbb{E}[z \log(e. z)| z < u n \alpha] + \mathbb{E}[z \log(e. z)| z \geq u n \alpha] \mathbb{P}(z \geq u n \alpha)
    \\
    & \leq u n \alpha \log(e u n \alpha) + n \log (e n) \exp\Big(-\frac{n \alpha (u-1)^2}{3}\Big),
\end{flalign}
where the last step uses the Chernoff bound. Combining all of this and following the rest of the steps in the proof of Theorem 1.4 of \citet{kamath2021improved}, we get the desired result.
\end{proof}

\section{Proof of Theorem~\ref{lower_bound_constrained}}
\label{lb_cons_pf}
The proof of Theorem~\ref{lower_bound_constrained} follows from that of \Cref{lower_bound_full} by just restricting ourselves to $\|\bm{w}\| \leq 1$ (as well as $\|\bm{w}_{\bm{v}}^{(\mathcal{A})}\| \leq 1$).

\section{Full Version and Proof of Theorem~\ref{thm1-dec15}}
\label{full-2}
\begin{theorem}[\textbf{Unconstrained Nonconvex Case}]
\label{thm1-dec15-full}
Suppose \Cref{dec23-asmp1} holds, $f$ is $L$-smooth and $\mathcal{W} = \mathbb{R}^d$. Fix some $\gamma \in (0,1)$ and $C > 0$. 
In \Cref{alg:1}, set $\tau = G \Big(\frac{G}{\gamma^2 C \sqrt{L}}\Big)^{\frac{1}{2k-1}}\Big(\frac{1}{T} + \varphi^2\Big)^{-\frac{1}{2(2k-1)}}$ and $\eta_t = \eta = \frac{C}{T \tau \sqrt{L}}\Big(\frac{1}{T} + \varphi^2\Big)^{-\frac{1}{2}}$ for all $t < T$.
Then with a probability of at least $(1 - \gamma)$ which is w.r.t. the random dataset $\mathcal{Z}$ that we obtain, DP-SGD (\Cref{alg:1}) has the following guarantee:
\begin{equation*}
    \textup{OR}(T) \leq \frac{(\sqrt{L})^{1 - \frac{1}{2k-1}}
    {G^{1 + \frac{1}{2k-1}}}}{\gamma^{\frac{2}{2k-1}} C^{\frac{1}{2k-1}}} \Bigg(3 C + \frac{2 (f(\bm{w}_0) - \min_{\bm{w} \in \mathbb{R}^d} f(\bm{w}))}{C}\Bigg)
    \Big(\frac{1}{T} + \varphi^2 \Big)^{\frac{1}{2}(1-\frac{1}{2k-1})}.
\end{equation*}
So if we set $T = \frac{1}{\varphi^2}$ above (which is what we do in \Cref{thm1-dec15}), we get the following bound for the risk:
\begin{equation*}
    \textup{OR}(T) \leq
    \frac{(\sqrt{2 L})^{1 - \frac{1}{2k-1}}
    {G^{1 + \frac{1}{2k-1}}}}{\gamma^{\frac{2}{2k-1}} C^{\frac{1}{2k-1}}} \Bigg(3 C + \frac{2 (f(\bm{w}_0) - \min_{\bm{w} \in \mathbb{R}^d} f(\bm{w}))}{C}\Bigg) \varphi^{(1-\frac{1}{2k-1})}.
\end{equation*}
\end{theorem}
We prove this below.

\subsection*{Proof:}
\label{pf-4}
\begin{proof}
From \Cref{bias-heavy-tail}, recall that
\begin{equation}
    \label{eq:dec13-4-cp}
    \Big\|\frac{1}{n}\sum_{i \in [n]}\text{clip}(\nabla f_i(\bm{w}), \tau) - \nabla f(\bm{w})\Big\| \leq \frac{G_{\mathcal{Z}}^k}{\tau^{k-1}},
\end{equation}
where $G_{\mathcal{Z}}^k = \frac{1}{n}\sum_{i=1}^n (G(\bm{x}_i, y_i))^k$.

Using the $L$-smoothness of $f$ and taking expectation only with respect to the randomness in the current iteration, we have:
\begin{flalign}
   \mathbb{E}[f(\bm{w}_{t+1})] 
   & \leq f(\bm{w}_{t}) - \eta \mathbb{E}[\langle \nabla f(\bm{w}_{t}), \bm{g}_t \rangle] + \frac{\eta^2 L}{2} \mathbb{E}[\|\bm{g}_t\|^2]
   \\
   \label{eq:may4-0}
   & = f(\bm{w}_{t}) - \eta \Big[\Big\langle \nabla f(\bm{w}_{t}), \frac{1}{n}\sum_{i = 1}^n \text{clip}(\nabla f_i(\bm{w}_t), \tau) \Big \rangle \Big] + {\eta^2 L \tau^2} \Big(1 + \frac{\nu d T \log(\frac{1}{\delta})}{n^2 \varepsilon^2} \Big)
   \\
   \label{eq:dec15-1}
   & =  f(\bm{w}_{t}) - \frac{\eta}{2}\Bigg\{\Big\|\frac{1}{n}\sum_{i = 1}^n \text{clip}(\nabla f_i(\bm{w}_t), \tau)\Big\|^2 + \|\nabla f(\bm{w}_t)\|^2 
   - \Big\| \frac{1}{n}\sum_{i = 1}^n \text{clip}(\nabla f_i(\bm{w}_t), \tau) - \nabla f(\bm{w}_t)\Big\|^2\Bigg\} 
   \\
   \nonumber
   & \text{ } \text{ } + {\eta^2 L \tau^2} \Big(1 + \frac{\nu d T \log(\frac{1}{\delta})}{n^2 \varepsilon^2} \Big)
   \\
    \label{eq:dec15-3}
    & \leq f(\bm{w}_{t}) - \frac{\eta}{2} \|\nabla f(\bm{w}_{t})\|^2 + \frac{\eta}{2} \Big(\frac{G_{\mathcal{Z}}^{2k}}{\tau^{2(k-1)}}\Big) + {\eta^2 L \tau^2} \Big(1 + \frac{\nu d T \log(\frac{1}{\delta})}{n^2 \varepsilon^2} \Big).
\end{flalign}
In \cref{eq:may4-0}, we have used \cref{eq:dec10-3}.
\Cref{eq:dec15-1} follows by using the fact for any two vectors $\bm{a}$ and $\bm{b}$, $\langle \bm{a}, \bm{b} \rangle = \frac{1}{2}\Big(\|\bm{a}\|^2 + \|\bm{b}\|^2 - \|\bm{a} - \bm{b}\|^2 \Big)$. \Cref{eq:dec15-3} is obtained by using \cref{eq:dec13-4-cp}. 

Next, summing up the above for $t=0$ through to $T-1$, taking expectation throughout and then after rearranging a bit and using the fact that $\mathbb{E}[f(\bm{w}_T)] \geq f^{*} = \min_{\bm{w} \in \mathbb{R}^d} f(\bm{w})$, we get:
\begin{flalign}
   \frac{1}{T}\sum_{t=0}^{T-1}\mathbb{E}[\|\nabla f(\bm{w}_t)\|^2] & \leq \frac{2(f(\bm{w}_0) - f^{*})}{\eta T} + 2 {\eta T L \tau^2} \Big(\frac{1}{T} + \frac{\nu d \log({1}/{\delta})}{n^2 \varepsilon^2} \Big) + \frac{G_{\mathcal{Z}}^{2k}}{\tau^{2(k-1)}}.
\end{flalign}
Let us plug in $\eta = \frac{C}{T \tau \sqrt{L} \sqrt{\frac{1}{T} + \frac{\nu d \log({1}/{\delta})}{n^2 \varepsilon^2}}}$ above, where $C > 0$ is a constant of our choice. With that, we get:
\begin{equation}
    \label{eq:apr20-1}
    \frac{1}{T}\sum_{t=0}^{T-1}\mathbb{E}[\|\nabla f(\bm{w}_t)\|^2] \leq \underbrace{\Big(\frac{2(f(\bm{w}_0) - f^{*})}{C} + 2 C\Big)}_{:=C'} \sqrt{L} \tau \sqrt{\frac{1}{T} + \frac{\nu d \log({1}/{\delta})}{n^2 \varepsilon^2}}
    + \frac{G_{\mathcal{Z}}^{2k}}{\tau^{2(k-1)}}.
\end{equation}
Let us now choose $\tau = G \Big(\frac{G}{\gamma^2 C \sqrt{L}}\Big)^{\frac{1}{2k-1}}\Big(\frac{1}{T} + \frac{\nu d \log({1}/{\delta})}{n^2 \varepsilon^2}\Big)^{-\frac{1}{2(2k-1)}}$, where $\gamma \in (0,1)$. That gives us:
\begin{equation}
    \frac{1}{T}\sum_{t=0}^{T-1}\mathbb{E}[\|\nabla f(\bm{w}_t)\|^2] \leq \frac{(\sqrt{L})^{1 - \frac{1}{2k-1}}
    {G^{1 + \frac{1}{2k-1}}}}{\gamma^{\frac{2}{2k-1}}} \Bigg(\frac{C'}{C^{\frac{1}{2k-1}}} + C^{1 - \frac{1}{2k-1}} \Big(\frac{G_{\mathcal{Z}}^{k}}{{G^{k}}/{\gamma}}\Big)^2\Bigg)
    \Big(\frac{1}{T} + \frac{\nu d \log({1}/{\delta})}{n^2 \varepsilon^2}\Big)^{\frac{1}{2}(1-\frac{1}{2k-1})}.
\end{equation}
Now, using Markov's inequality, $G_{\mathcal{Z}}^k \leq \frac{G^k}{\gamma}$ with a probability of at least $1 - \gamma$ w.r.t. the random dataset $\mathcal{Z}$. Plugging this above, we get:
\begin{equation}
    \frac{1}{T}\sum_{t=0}^{T-1}\mathbb{E}[\|\nabla f(\bm{w}_t)\|^2] \leq \frac{(\sqrt{L})^{1 - \frac{1}{2k-1}}
    {G^{1 + \frac{1}{2k-1}}}}{\gamma^{\frac{2}{2k-1}}} \Bigg(\frac{C'}{C^{\frac{1}{2k-1}}} + C^{1 - \frac{1}{2k-1}}\Bigg)
    \Big(\frac{1}{T} + \frac{\nu d \log({1}/{\delta})}{n^2 \varepsilon^2}\Big)^{\frac{1}{2}(1-\frac{1}{2k-1})},
\end{equation}
with a probability of at least $1 - \gamma$ w.r.t. the random dataset $\mathcal{Z}$.

Lastly, plugging in $\varphi = \frac{\sqrt{\nu d  \log({1}/{\delta})}}{n\varepsilon}$ and the value of $C'$, noting that $\mathbb{E}[\|\nabla f(\bm{w}_{\widehat{t}})\|^2] = \frac{1}{T}\sum_{t=0}^{T-1}\mathbb{E}[\|\nabla f(\bm{w}_t)\|^2]$ and using the definition of $\text{OR}(T)$, we get the final result.
\end{proof}

\end{document}